\documentclass[twoside,11pt]{article}

\usepackage{jmlr2e}

\RequirePackage{amsmath,amsmath,amsfonts,amssymb}
\RequirePackage{natbib}

\usepackage{graphicx,color}

\usepackage[boxed]{algorithm2e}
\SetKwRepeat{Do}{do}{while}%

\newtheorem{fact}{Fact}[section]
\newtheorem{assumption}{Assumption}[section]
\newtheorem{question}{Question}[section]

\newcommand{\mat}{}
\newcommand{\vct}{}

\newcommand{\ud}{\mathrm d}
\newcommand{\nml}{\mathcal{N}}

\newcommand{\argmin}{\mathrm{argmin}}
\newcommand{\argmax}{\mathrm{argmax}}

\newcommand{\diag}{\mathrm{diag}}

\newcommand{\ols}{\mathrm{ols}}
\newcommand{\tr}{\mathrm{tr}}

\newcommand{\unif}{\mathrm{unif}}

\newcommand{\Kernel}{\mathrm{ker}}

\newcommand{\vol}{\mathrm{vol}}

\newcommand{\Range}{\mathrm{Range}}

\newcommand{\mP}{\mathbb P}

\renewcommand{\hat}{\widehat}
\renewcommand{\tilde}{\widetilde}

\jmlrheading{18}{2017}{1-40}{3/17; Revised 9/17}{12/17}{17-175}{Yining Wang, Adams Wei Yu and Aarti Singh}

\ShortHeadings{Computationally Tractable Experiment Selection}{Wang, Yu and Singh}
\firstpageno{1}

\begin{document}

\title{On Computationally Tractable Selection of Experiments in Measurement-Constrained Regression Models}

\author{\name Yining Wang \email yiningwa@cs.cmu.edu \\
       \name Adams Wei Yu \email weiyu@cs.cmu.edu \\
       \name Aarti Singh \email aartisingh@cmu.edu\\
       \addr Machine Learning Department, School of Computer Science\\
       Carnegie Mellon University\\
       Pittsburgh, PA 15213, USA}

\editor{Michael Mahoney}

\maketitle

\begin{abstract}
We derive computationally tractable methods to select a small subset of experiment settings from a large pool of given design points.
The primary focus is on linear regression models, while the technique extends to generalized linear models and Delta's method (estimating functions of linear regression models)
as well.
The algorithms are based on a continuous relaxation of an otherwise intractable combinatorial optimization problem,
with sampling or greedy procedures as post-processing steps.
Formal approximation guarantees are established for both algorithms, 
and numerical results on both synthetic and real-world data confirm the effectiveness of the proposed methods.
\end{abstract}


\section{Introduction}
{Despite the availability of large datasets, in many applications, collecting 
labels for all data points is not possible due to measurement constraints. We consider the 
problem of measurement-constrained regression where we are given a large pool of $n$ 
data points but can only observe a small set of $k\ll n$ labels. 
Classical {\em experimental design} approaches in statistical literature \cite{optimal-design-book}
have investigated this problem, 
but the proposed solutions tend to be often combinatorial. 
In this work, we investigate computationally tractable methods for selecting data points to 
label from a given pool of design points in measurement-constrained regression.}

Despite the simplicity and wide applicability of OLS, 
in practice it may not be possible to obtain the full $n$-dimensional response vector $\vct y$
due to measurement constraints.
It is then a common practice to select a small subset of rows (e.g., $k\ll n$ rows) in $\mat X$ so that the statistical efficiency of regression on the selected subset of design points is maximized.
Compared to the classical \emph{experimental design} problem \citep{optimal-design-book} where $X$ can be freely designed,
in this work we consider the setting where the selected design points must come from an existing (finite) design pool $\mat X$.

Below we list three example applications where such measurement constraints are relevant:
\paragraph{Example 1 (Material synthesis)}
Material synthesis experiments are time-consuming and expensive, whose results are sensitive to experimental setting features such as temperature, duration and reactant ratio.
Given hundreds, or even thousands of possible experimental settings, it is important to select a handful of representative settings
such that a model can be built with maximized statistical efficiency to predict quality of the outcome material from experimental features.
In this paper we consider such an application of low-temperature microwave-assisted thin film crystallization \citep{reeja2012microwave,nakamura2017design}
 and demonstrate the effectiveness of our proposed algorithms.

\paragraph{Example 2 (CPU benchmarking)}
Central processing units (CPU) are vital to the performance of a computer system.
It is an interesting statistical question to understand how known manufacturing parameters (clock period, cache size, etc.) of a CPU
affect its execution time (performance) on benchmark computing tasks \citep{cpu-relative-performance}.
As the evaluation of real benchmark execution time is time-consuming and costly,
it is desirable to select a subset of CPUs available in the market with diverse range of manufacturing parameters
so that the statistical efficiency is maximized by benchmarking for the selected CPUs.

\paragraph{Example 3 (Wind speed prediction)} 
In \cite{minnesota-wind-speed} a data set is created to record wind speed across a year at established measurement locations on highways in Minnesota, US.
Due to instrumentation constraints, wind speed can only be measured at intersections of high ways,
and a small subset of such intersections is selected for wind speed measurement in order to reduce data gathering costs.

\vspace{-0.1in}
\paragraph{}

In this work, we primarily focus on the linear regression model (though extensions to generalized 
linear models and functions of linear models are also considered later)
$$
\vct y = \mat X\vct\beta_0+\vct\varepsilon,
$$
where $\mat X\in\mathbb R^{n\times p}$ is the design matrix,
$\vct y\in\mathbb R^n$ is the response and $\vct\varepsilon\sim\nml_n(\vct 0,\sigma^2\mat I_n)$ are homoscedastic Gaussian noise with variance $\sigma^2$.
$\vct\beta_0$ is a $p$-dimensional regression model that one wishes to estimate.
We consider the ``large-scale, low-dimensional'' setting where both $n,p\to\infty$ and $p<n$, and $\mat X$ has full column rank.
A common estimator is the \emph{Ordinary Least Squares} (OLS) estimator:
$$
\hat{\vct\beta}^\ols = \argmin_{\vct\beta\in\mathbb R^p} \|\vct y-\mat X\vct\beta\|_2^2 = (\mat X^\top\mat X)^{-1}\mat X^\top\vct y.
$$
The mean square error of the estimated regression coefficients $\mathbb{E}\|\hat{\vct\beta}^\ols - \vct\beta_0\|^2_2 = \sigma^2 \tr\left[\left(\mat X^\top\mat X\right)^{-1}\right]$.
Under measurement constraints, it is well-known that the statistically optimal subset $S^*$ 
for estimating the regression coefficients 
is given by the \emph{A-optimality} criterion \citep{optimal-design-book}:
\begin{equation}
S^* = \argmin_{S\subseteq[n], |S|\leq k}\tr\left[\left(\mat X_S^\top\mat X_S\right)^{-1}\right].
\label{eq_css}
\end{equation}
Despite the statistical optimality of Eq.~(\ref{eq_css}), the optimization problem is combinatorial in nature and the optimal subset $S^*$ is difficult to compute.
A brute-force search over all possible subsets of size $k$ requires $O(n^kk^3)$ operations, which is infeasible for even moderately sized designs $\mat X$.

In this work, we focus on \textbf{computationally tractable} methods for experiment selection that achieve near-optimal statistical efficiency in linear regression models.
We consider two experiment selection models: 
the \emph{with replacement} model where each design point (row of $\mat X$) can be selected more than once with independent noise involved for each selection,
and the \emph{without replacement} model where distinct row subsets are required.
We propose two computationally tractable algorithms: one sampling based algorithm that achieves $O(1)$ approximation 
of the statistically optimal solution for the with replacement model and, when $\Sigma^*$ is well-conditioned, the algorithm also works for the without replacement model.
In the ``soft budget'' setting $|S|=O_\mP(k)$, the approximation ratio can be further improved to $1+O(\epsilon)$.
We also propose a greedy method that achieves $O(p^2/k)$ approximation for the without replacement model regardless of the conditioning of the design pool $X$.

\section{Problem formulation and backgrounds}\label{sec:formulation}

We first give a formal definition of the experiment selection problem in linear regression models:
\begin{definition}[experiment selection problem in linear regression models]
Let $\mat X$ be a known $n\times p$ design matrix with full column rank and $k$ be the subset budget, with $p\leq k\leq n$.
An experiment selection problem aims to find a subset $S\subseteq[n]$ of size $k$, either deterministically or randomly, then observes $y_S=X_S\beta_0+\tilde\varepsilon$,
where each coordinate of $\tilde\varepsilon$ is i.i.d.~Gaussian random variable with zero mean and equal covariance,
and then generates an estimate $\widehat \beta$ of the regression coefficients based on $(X_S,y_S)$.
Two types of experiment selection algorithms are considered:
\begin{enumerate}
\item \emph{With replacement}: $S$ is a multi-set which allows duplicates of indices.
Note that fresh (independent) noise is imposed after, and independent of, experiment selection,
and hence duplicate design points have independent noise.
We use $\mathcal A_1(k)$ to denote the class of all with replacement experiment selection algorithms.

\item \emph{Without replacement}: $S$ is a standard set that may not have duplicate indices.
{We use $\mathcal A_2(k)$ to denote the class of all without replacement experiment selection algorithms.}
\end{enumerate}
\label{defn:subsampling}
\end{definition}

\begin{table}[t]
\caption{\small Summary of approximation results. $\Sigma_*=X^\top\diag(\pi^*)X$, where $\pi^*$ is the optimal solution of Eq.~(\ref{eq_c_opt}). Sampling based algorithms can be with or without replacement and are randomized, which succeed with at least 0.8 probability.}
\label{tab:summary}
\vspace{0.1in}
\scalebox{0.75}{
\centering
\begin{tabular}{c|ccccc}
\hline
Algorithm& Model& Constraint&  $C(n,p,k)$&  Assumptions \\
\hline
\cite{levscore-regression}& with rep.& additive\protect\footnotemark[1]& -\protect\footnotemark[2]& asymptotic\\
 \cite{faster-css}& with rep.& additive& $O(n/k)$& $k=\Omega(p)$\\
sampling&  with rep.& $|S|=O_\mP(k)$& $1+O(\epsilon)$& $p\log k/k=O(\epsilon^2)$\\
sampling& without rep.& $|S|=O_\mP(k)$& $1+O(\epsilon)$& $\|\Sigma_*^{-1}\|_2\kappa(\Sigma_*)\|X\|_{\infty}\log p=O(\epsilon^2)$\\
sampling&  with rep.& $|S|\leq k$& $O(1)$& $p\log k/k=O(1)$\\
sampling& without rep.& $|S|\leq k$& $O(1)$& $\|\Sigma_*^{-1}\|_2\kappa(\Sigma_*)\|X\|_{\infty}\log p=O(1)$\\
greedy&  without rep.& $|S|\leq k$&\begin{tabular}{@{}c@{}}\textbf{Rigorous}: $1+\frac{p(p+1)}{2(k-p+1)}$\\ \textbf{Conjecture}: $1+O(\frac{p}{k-p})$\end{tabular}&
\shortstack{$k>p$, or $k=\Omega(p^2/\epsilon)$\\ if $(1+\epsilon)$-approximation desired}\\
\hline
\end{tabular}
}
\label{tab:results}
\end{table}

\footnotetext[1]{``Additive'' means that the statistical error of the resulting estimator cannot be bounded by a multiplicative factor of the minimax optimal error.}
\footnotetext[2]{The leverage score sampling method in \cite{levscore-regression} does not have rigorous approximation guarantees in terms of $\|\hat\beta-\beta_0\|_2^2$
or $\|X\hat\beta-X\beta_0\|_2^2$.
However, the bounds in that paper establish that leverage score sampling can be worse or better than uniform sampling under different settings. } 

As evaluation criterion, we consider the mean square error $\mathbb E\|\hat\beta-\beta_0\|_2^2$, 
where $\hat\beta$ is an estimator of $\beta_0$ with $X_S$ and $y_S$ as inputs.
We study \emph{computationally tractable} algorithms that approximately achieves the \emph{minimax} rate of convergence 
over $\mathcal A_1(k)$ or $\mathcal A_2(k)$:
\begin{equation}
 \inf_{A'\in\mathcal A_b(k)}\sup_{\beta_0\in\mathbb R^p}\mathbb E\left[\|\hat\beta_{A'}-\beta_0\|_2^2\right].
\label{eq:minimax}
\end{equation}
Formally, we give the following definition:
\begin{definition}[$C(n,p,k)$-approximate algorithm]
Fix $n\geq k\geq p$ and $b\in\{1,2\}$.
We say an algorithm (either deterministic or randomized) $A\in\mathcal A_b(k)$ is a $C(n,p,k)$-approximate algorithm if
for any $X\in\mathbb R^{n\times p}$ with full column rank,
$A$ produces a subset $S\subseteq[n]$ with size $k$ and an estimate $\hat\beta_A$ in polynomial time such that, with probability at least 0.8, 
\begin{equation}
\sup_{\beta_0\in\mathbb R^p}\mathbb E\left[\|\hat\beta_A-\beta_0\|_2^2\Big|X_S\right] \leq C(n,p,k)\cdot \inf_{A'\in\mathcal A_b(k)}\sup_{\beta_0\in\mathbb R^p}\mathbb E\left[\|\hat\beta_{A'}-\beta_0\|_2^2\right].
\end{equation}
Here both expectations are taken over the randomness in the noise variables $\tilde\varepsilon$ and the inherent randomness in $A$.
\label{defn:approxalg}
\end{definition}
Table \ref{tab:results} gives an overview of approximation ratio $C(n,p,k)$ for algorithms proposed in this paper.
{We remark that the combinatorial A-optimality solution of Eq.~(\ref{eq_css}) 
upper bounds the minimax risk (since minimaxity is defined over deterministic algorithms 
as well), hence the approximation guarantees also hold with respect to the combinatorial A-optimality objective.}

\subsection{Related work}

There has been an increasing amount of work on fast solvers for the general least-square problem $\min_{\beta}\|y-X\beta\|_2^2$.
Most of existing work along this direction \citep{sketching-book,subsample-ols,faster-ols,raskutti2015statistical}
focuses solely on the computational aspects and do not consider statistical constraints such as limited measurements of $y$.
A convex optimization formulation was proposed in \cite{constrained-adaptive-sensing} for a constrained adaptive sensing problem, which is a special case of our setting, but
without finite sample guarantees with respect to the combinatorial problem.
In \cite{horel2014budget} a computationally tractable approximation algorithm was proposed for the D-optimality criterion of the experimental design problem.
However, the core idea in \cite{horel2014budget} of pipage rounding an SDP solution \citep{ageev2004pipage} is not applicable 
in our problem because the objective function we consider in Eq.~(\ref{eq_c_opt}) is not submodular.

Popular subsampling techniques such as leverage score sampling \citep{cur_relative} were studied in least square and linear regression problems \citep{optimal-subsampling,levscore-regression,minnesota-wind-speed}.
While computation remains the primary subject, measurement constraints and statistical properties were also analyzed within the linear regression model \citep{optimal-subsampling}.
However, none of the above-mentioned (computationally efficient) methods achieve near minimax optimal statistical efficiency in terms of estimating the underlying linear model $\beta_0$,
since the methods can be worse than uniform sampling which has a fairly large approximation constant for general $X$. 
{One exception is \cite{faster-css}, which proposed a greedy algorithm that achieves an error bound of estimation of $\beta_0$ that is within a multiplicative factor
of the minimax optimal statistical efficiency. The multiplicative factor is however large and depends on the size of the full sample pool $X$, as we remark in Table \ref{tab:results} and also
in Sec.~\ref{subsec:greedy}.}

Another related area is \emph{active learning} \citep{active-mle,hard-margin-active,stratification},
 which is a stronger setting where feedback from prior measurements can be used to guide subsequent data selection.
 \cite{active-mle} analyzes an SDP relaxation in the context of active maximum likelihood estimation.
However, the analysis in \cite{active-mle} only works for the with replacement model
and the two-stage feedback-driven strategy proposed in \cite{active-mle} is not available under the experiment selection model defined in Definition \ref{defn:subsampling} where no feedback is assumed.

\subsection{Notations}

For a matrix $A\in\mathbb R^{n\times m}$, we use $\|A\|_p=\sup_{x\neq 0}\frac{\|Ax\|_p}{\|x\|_p}$ to denote the induced $p$-norm of $A$.
In particular, $\|A\|_1=\max_{1\leq j\leq m}\sum_{i=1}^n{|A_{ij}|}$ and $\|A\|_{\infty} = \max_{1\leq i\leq n}\sum_{j=1}^m{|A_{ij}|}$.
$\|A\|_F=\sqrt{\sum_{i,j}{A_{ij}^2}}$ denotes the Frobenius norm of $A$.
Let $\sigma_1(A)\geq\sigma_2(A)\geq\cdots\geq\sigma_{\min(n,m)}(A)\geq 0$ be the singular values of $A$, sorted in descending order.
The condition number $\kappa_2(A)$ is defined as $\kappa(A)=\sigma_1(A)/\sigma_{\min(n,m)}(A)$.
For sequences of random variables $X_n$ and $Y_n$, we use $X_n\overset{p}{\to} Y_n$ to denote $X_n$ converges in probability to $Y_n$.
We say $a_n\lesssim b_n$ if $\lim_{n\to\infty}|a_n/b_n|\leq 1$ and $a_n\gtrsim b_n$ if $\lim_{n\to\infty}|b_n/a_n|\leq 1$.
For two $d$-dimensional symmetric matrices $A$ and $B$, we write $A\preceq B$ if $u^\top (A-B)u\leq 0$ for all $u\in\mathbb R^d$,
and $A\succeq B$ if $u^\top(A-B)u \geq 0$ for all $u\in\mathbb R^d$.

\section{Methods and main results}

We describe two computationally feasible algorithms for the experiment selection problem,
both based on a continuous convex optimization problem.
Statistical efficiency bounds are presented for both algorithms, with detailed proofs given in Sec.~\ref{sec:proof}.

\subsection{Continuous optimization and minimax lower bounds}\label{subsec:convex}
We consider the following continuous optimization problem, which is a convex relaxation of the combinatorial A-optimality criterion of Eq.~(\ref{eq_css}): 
\begin{align}
\pi^* &= \argmin_{\pi\in\mathbb R^n}f(\pi;X) = \argmin_{\pi\in\mathbb R^p}\tr\left[\left(X^\top\diag(\pi) X\right)^{-1}\right],\label{eq_c_opt}\\
&s.t. \;\;\;\pi\geq 0, \;\;\|\pi\|_1\leq k,\nonumber\\
& \;\;\;\;\;\; \|\pi\|_{\infty} \leq 1 \;\;\;\text{(only for the without replacement model).}\nonumber
\end{align}
Note that the $\|\pi\|_{\infty}\leq 1$ constraint is only relevant for the without replacement model
and for the with replacement model we drop this constraint in the optimization problem.
It is easy to verify that both the objective $f(\pi;X)$ and the feasible set in Eq.~(\ref{eq_c_opt}) are convex,
and hence the global optimal solution $\pi^*$ of Eq.~(\ref{eq_c_opt}) can be obtained using computationally tractable algorithms.
In particular, we describe an SDP formulation and a practical projected gradient descent algorithm in Appendix \ref{appsec:optimization} and \ref{suppsec:projection},
both provably converge to the global optimal solution of Eq.~(\ref{eq_c_opt}) with running time scaling polynomially in $n,p$ and $k$.

We first present two facts, which are proved in Sec.~\ref{sec:proof}.
\begin{fact}
Let $\pi$ and $\pi'$ be feasible solutions of Eq.~(\ref{eq_c_opt}) such that $\pi_i\leq \pi_i'$ for all $i=1,\cdots,n$.
Then $f(\pi;X)\geq f(\pi';X)$, with equality if and only if $\pi=\pi'$.
\label{fact:2}
\end{fact}
\begin{fact}
$\|\pi^*\|_1=k$.
\label{fact:3}
\end{fact}

We remark that the inverse monotonicity of $f$ in $\pi$ implies second fact, which can potentially be used to understand sparsity of $\pi^*$ in later sections.

The following theorem shows that the optimal solution of Eq.~(\ref{eq_c_opt}) lower bounds the minimax risk defined in Eq.~(\ref{eq:minimax}).
Its proof is placed in Sec.~\ref{sec:proof}.
\begin{theorem}
Let $f_1^*(k;X)$ and $f_2^*(k;X)$ be the optimal objective values of Eq.~(\ref{eq_c_opt}) for with replacement and without replacement, respectively.
Then for $b\in\{1,2\}$, 
\begin{equation}
\inf_{A\in\mathcal A_b(k)}\sup_{\beta_0\in\mathbb R^p}\mathbb E\left[\|\hat\beta_A-\beta_0\|_2^2\right] \geq \sigma^2\cdot f_b^*(k;X).
\label{eq:minimax_lowrbound}
\end{equation}
\label{thm:minimax}
\end{theorem}


{
Despite the fact that Eq.~(\ref{eq_c_opt}) is computationally tractable, its solution $\pi^*$ is \emph{not} a valid experiment selection algorithm under a measurement budget of $k$ because
there can be much more than $k$ components in $\pi^*$ that are not zero.
In the following sections, we discuss strategies for choosing a subset $\hat S$ of rows with $|\hat S|=O_\mP(k)$ (i.e., soft constraint) or $|\hat S|\leq k$ (i.e., hard constraint), using the solution $\pi^*$ to the above.
}



\subsection{Sampling based experiment selection: soft size constraint}\label{subsec:sampling_expected}

We first consider a weaker setting where \emph{soft} constraint is imposed on the size of the selected subset $\hat S$;
in particular, it is allowed that $|\hat S|=O(k)$ with high probability, meaning that a constant fraction of over-selection is allowed.
The more restrictive setting of hard constraint is treated in the next section.

A natural idea of obtaining a valid subset $S$ of size $k$ is by \emph{sampling} from a weighted row distribution specified by $\pi^*$.
Let $\Sigma_*=X^\top\diag(\pi^*)X$ and
define distributions $p_j^{(1)}$ and $p_j^{(2)}$ for $j=1,\cdots,n$ as
\begin{align*}
&P^{(1)}:\;\; p_j^{(1)} = \pi_j^*x_j^\top\Sigma_*^{-1} x_j/p, &\text{with replacement};\\
&P^{(2)}:\;\; p_j^{(2)} = \pi_j^*/k,&\text{without replacement}.
\end{align*}
Note that both $\{p_j^{(1)}\}_{j=1}^n$ and $\{p_j^{(2)}\}_{j=1}^n$ sum to one because
$\sum_{j=1}^n{\pi_j^*}=k$ and $\sum_{j=1}^n{\pi_j^*x_j^\top\Sigma_*^{-1}x_j} = \tr((\sum_{j=1}^n{\pi_j^*x_jx_j^\top})\Sigma_*^{-1})
= \tr(\Sigma_*\Sigma_*^{-1}) = p$.

\begin{algorithm}[h]
\SetAlgorithmName{Figure}{}
\SetAlgoLined
\DontPrintSemicolon
\SetKwInOut{Input}{input}\SetKwInOut{Output}{output}
\Input{$X\in\mathbb R^{n\times p}$, optimal solution $\pi^*$, target subset size $k$.}
\Output{$\hat S\subseteq[n]$, a selected subset of size at most $O_\mP(k)$.}
Initialization: $t=0$, $S_0=\emptyset$.\;
\underline{\emph{With replacement}}: for $t=1,\cdots,k$ do:\;
\quad - sample $i_t\sim P^{(1)}$ and set $w_t=\lceil\pi_{i_t}^*/(kp_{i_t}^{(1)})\rceil$;\;
\quad - update: $S_{t+1}=S_t\cup\{\text{$w_t$ repetitions of $x_{i_t}$}\}$.\;
\underline{\emph{Without replacement}}: for $i=1,\cdots,n$ do:\;
\quad - sample $w_i\sim \mathrm{Bernoulli}(kp_j^{(2)})$;\;
\quad - update: $S_{i+1} = S_i\cup\{\text{$w_i$ repetitions of $x_i$}\}$.\;
Finally, output $\hat S=S_k$ for with replacement and $\hat S=S_n$ for without replacement.
\caption{Sampling based experiment selection (expected size constraint).}
\label{alg:subsampling_expected}
\end{algorithm}

Under the without replacement setting the distribution $P^{(2)}$ is straightforward: $p_j^{(2)}$ is proportional to the optimal continuous weights $\pi_j^*$;
under the with replacement setting, the sampling distribution takes into account leverage scores (effective resistance) of each data point in the conditioned covariance as well.
Later analysis (Theorem \ref{thm:sampling_expected}) shows that it helps with the finite-sample condition on $k$.
Figure \ref{alg:subsampling_expected} gives details of the sampling based algorithms for both with and without replacement settings.

The following proposition bounds the size of $\hat S$ in high probability:
\begin{proposition}
For any $\delta\in(0,1/2)$ with probability at least $1-\delta$ it holds that
$|\hat S|\leq 2(1+1/\delta) k$.
That is, $|\hat S|\leq O_\mP(k)$.
\label{prop:support_upperbound}
\end{proposition}
\begin{proof}
Apply Markov's inequality and note that an additional $k$ samples need to be added due to the ceiling operator in with replacement sampling.
\end{proof}

The sampling procedure is easy to understand in an asymptotic sense:
it is easy to verify that $\mathbb EX_{i_t}^\top X_{i_t}=X^\top\diag(\pi^*/k) X$ and $\mathbb E[w_{i_t}]=1$,
for both with and without replacement settings.
Note that $\|p_{i_t}^{(2)}\|_{\infty}\leq 1/k$ by feasibility constraints and hence $\mathrm{Bernoulli}(kp_{i_t}^{(2)})$ is a valid distribution for all $i_t\in[n]$.
For the with replacement setting,
by weak law of large numbers, $X_{\hat S}^\top X_{\hat S}\overset{p}{\to} X^\top\diag(\pi^*) X$ as $k\to\infty$
and hence $\tr[(X_S^\top X_S)^{-1}]\overset{p}{\to} f(\pi^*;X)$ by continuous mapping theorem.
A more refined analysis is presented in Theorem \ref{thm:sampling_expected} to provide explicit conditions under which the asymptotic approximations are valid and
on the statistical efficiency of $\hat S$ as well as analysis under the more restrictive without replacement regime.

\begin{theorem}
Fix $\epsilon>0$ as an arbitrarily small accuracy parameter.
Suppose the following conditions hold:
\begin{align*}
&\text{With replacement}:\;\; p\log k/k=O(\epsilon^2);\\
&\text{Without replacement}: \;\; \|\Sigma_*^{-1}\|_2\kappa(\Sigma_*)\|X\|_{\infty}^2\log p = O(\epsilon^2).
\end{align*}
Here $\Sigma_*=X^\top\diag(\pi^*) X$ and $\kappa(\Sigma_*)$ denotes the conditional number of $\Sigma_*$.
Then with probability at least $0.9$ the subset OLS estimator $\hat\beta=(X_{\hat S}^\top X_{\hat S})^{-1}X_{\hat S}^\top y_{\hat S}$ satisfies
$$
\mathbb E\left[\|\hat\beta-\beta_0\|_2^2\Big| X_{\hat S}\right] = \sigma^2\tr\left[(X_{\hat S}^\top X_{\hat S})^{-1}\right] \leq \left(1+O(\epsilon)\right)\cdot \sigma^2f_b^*(X;k), \;\;\;\;\;b\in\{1,2\}.
$$
\label{thm:sampling_expected}
\end{theorem}

We adapt the proof technique of Spielman and Srivastava in their seminal work on spectral sparsification of graphs \citep{graph-sparsification}.
More specifically, we prove the following stronger ``two-sided'' result which shows that $X_{\hat S}^\top X_{\hat S}$ is a \emph{spectral approximation}
of $\Sigma_*$ with high probability, under suitable conditions.

\begin{lemma}
Under the same conditions in Theorem \ref{thm:sampling_expected}, it holds that with probability at least $0.9$ that
$$
(1-\epsilon)z^\top\Sigma_* z\leq z^\top\hat\Sigma_{\hat S}z \leq (1+\epsilon)z^\top\Sigma_* z, \;\;\;\;\forall z\in\mathbb R^p,
$$
where $\Sigma_*=X^\top\diag(\pi^*)X$ and $\Sigma_{\hat S}=X_{\hat S}^\top X_{\hat S}$.
\label{lem:spectral_expected}
\end{lemma}

Lemma \ref{lem:spectral_expected} implies that with high probability $\sigma_j(\hat\Sigma_{\hat S})\geq (1-\epsilon)\sigma_j(\Sigma_*)$ for all $j=1,\cdots,p$.
Recall that $\tr[(X_S^\top X_S)^{-1}]=\sum_{j=1}^p{\sigma_j(\hat\Sigma_{\hat S})^{-1}}$ and $f(\pi^*;X)=\sum_{j=1}^p{\sigma_j(\Sigma_*)}$.
Subsequently, for $\epsilon\in(0,1-c]$ for some constant $c>0$, $\tr[(X_{\hat S}^\top X_{\hat S})^{-1}] \leq (1+O(\epsilon))f(\pi^*;X)$.
Theorem \ref{thm:sampling_expected} is thus proved.

\subsection{Sampling based experiment selection: hard size constraint}\label{subsec:sampling_deterministic}

In some applications it is mandatory to respect a hard subset size constraint;
that is, a randomized algorithm is expected to output $\hat S$ that satisfies $|\hat S|\leq k$ almost surely,
and no over-sampling 
is allowed.
To handle such hard constraints, we revise the algorithm in Figure \ref{alg:subsampling_expected} as follows:

\begin{algorithm}[h]
\SetAlgorithmName{Figure}{}
\SetAlgoLined
\DontPrintSemicolon
\SetKwInOut{Input}{input}\SetKwInOut{Output}{output}
\Input{$X\in\mathbb R^{n\times p}$, optimal solution $\pi^*$, target subset size $k$.}
\Output{$\hat S\subseteq[n]$, a selected subset of size at most $k$.}
Initialization: $t=0$, $S_0=\emptyset$, $R_0=\emptyset$.\;
1. \emph{With replacement}: sample $i_t\sim P^{(1)}$; set $w_t=\lceil\pi_{i_t}^*/(kp_{i_t}^{(1)})\rceil$;\;
\quad   \emph{Without replacement}: pick random $i_t\notin R_{t-1}$; sample $w_t\sim\mathrm{Bernoulli}(kp_{i_t}^{(2)})$.\;
2. \emph{Update}: $S_t= S_{t-1} \cup \{\text{$w_t$ repetitions of $i_t$}\}$, $R_t=R_{t-1}\cup\{i_t\}$.\;
3. Repeat steps 1 and 2 until at some $t = T$, $|S_{T+1}|>k$ or $R_{T+1}=[n]$. Output $\hat S=S_{T}$.
\caption{Sampling based experiment selection (deterministic size constraint).}
\label{alg:subsampling_deterministic}
\end{algorithm}

We have the following theorem, which mimics Theorem \ref{thm:sampling_expected} but with weaker approximation bounds:
\begin{theorem}
Suppose the following conditions hold:
\begin{align*}
&\text{with replacement}: \;\;p\log k/k=O(1);\\
&\text{without replacement}:\;\; \|\Sigma_*^{-1}\|_2\kappa(\Sigma_*)\|X\|_{\infty}^2\log p = O(1).
\end{align*}
Then with probability at least 0.8 the subset estimator $\hat\beta=(X_{\hat S}^\top X_{\hat S})^{-1} X_S^\top y_S$ satisfies 
$$
\mathbb E\left[\|\hat\beta-\beta_0\|_2^2\Big| X_{\hat S}\right] = \sigma^2\tr\left[(X_{\hat S}^\top X_{\hat S})^{-1}\right] \leq O(1)\cdot \sigma^2f_b^*(X;k), \;\;\;\;\;b\in\{1,2\}.
$$
\label{thm:sampling_deterministic}
\end{theorem}


The following lemma is key to the proof of Theorem \ref{thm:sampling_deterministic}.
Unlike Lemma \ref{lem:spectral_expected}, in Lemma \ref{lem:spectral_deterministic} we only prove one side of the spectral approximation relation,
which suffices for our purposes.
To handle without replacement, we cite matrix Bernstein for combinatorial matrix sums in \citep{mackey2014matrix}.
\begin{lemma}
Define $\hat\Sigma_{\hat S}=X^\top_{\hat S}X_{\hat S}$.
Suppose the following conditions hold:
\begin{align*}
&\text{With replacement}:\;\; p\log T/T=O(1);\\
&\text{Without replacement}: \;\; \|\Sigma_*^{-1}\|_2\kappa(\Sigma_*)\|X\|_{\infty}^2\log p = O(T/n).
\end{align*}
Then with probability at least $0.9$ the following holds:
\begin{equation}
z^\top\hat\Sigma_{\hat S} z\geq K_Tz^\top\Sigma_* z, \;\;\;\;\;\;\forall z\in\mathbb R^p,
\label{eq:spectral_lower_bound}
\end{equation}
where $K_T=\Omega(T/k)$ for with replacement and $K_T=\Omega(T/n)$ for without replacement.
\label{lem:spectral_deterministic}
\end{lemma}


Finally, we need to relate conditions on $T$ in Lemma \ref{lem:spectral_deterministic} to interpretable conditions on subset budget $k$:
\begin{lemma}
Let $\delta>0$ be an arbitrarily small fixed failure probability. The with probability at least $1-\delta$ we have that
$T\geq \delta k$ for with replacement and $T\geq \delta n$ for without replacement.
\label{lem:T}
\end{lemma}
\begin{proof}
For with replacement we have $\mathbb E[\sum_{t=1}^T{w_t}] = T$ and for without replacement we have
$\mathbb E[\sum_{t=1}^T{w_t}] = Tk/n$.
Applying Markov's inequality on $\Pr[\sum_{t=1}^T{w_t}>k]$ for $T=\delta k$ and/or $T=\delta n$ we complete the proof of Lemma \ref{lem:T}.
\end{proof}
Combining Lemmas \ref{lem:spectral_deterministic} and \ref{lem:T} with $\delta=0.1$ and note that $T\leq k$ almost surely (because $w_t\geq 1$), we prove Theorem \ref{thm:sampling_deterministic}.


\subsection{Greedy experiment selection}\label{subsec:greedy}

\begin{algorithm}[h]
\SetAlgorithmName{Figure}{}
\SetAlgoLined
\DontPrintSemicolon
\SetKwInOut{Input}{input}\SetKwInOut{Output}{output}
\Input{$X\in\mathbb R^{n\times p}$, Initial subset $S_0\subseteq[n]$, target size $k\leq |S_0|$.}
\Output{$\hat S\subseteq[n]$, a selected subset of size $k$.}
Initialization: $t=0$.\;
1. Find $j^*\in S_t$ such that $\tr[(X_{S_t\backslash\{j^*\}}^\top X_{S_t\backslash\{j^*\}})^{-1}]$ is minimized.\;
2. Remove $j^*$ from $S_t$: $S_{t+1}=S_t\backslash\{j^*\}$.\;
3. Repeat steps 1 and 2 until $|S_t|=k$. Output $\hat S=S_t$.
\caption{Greedy experiment selection.}
\label{alg:greedy}
\end{algorithm}

\cite{faster-css} proposed an interesting greedy removal algorithm (outlined in Figure \ref{alg:greedy}) and established the following result:
\begin{lemma}
Suppose $\hat S\subseteq[n]$ of size $k$ is obtained by running algorithm in Figure \ref{alg:greedy} with an initial subset $S_0\subseteq[n]$, $|S_0|\geq k$.
Both $\hat S$ and $S_0$ are standard sets (i.e., without replacement).
Then
$$
\tr\left[(X_{\hat S}^\top X_{\hat S})^{-1}\right]\leq \frac{|S_0|-p+1}{k-p+1}\tr\left[(X_{S_0}^\top X_{S_0})^{-1}\right].
$$
\end{lemma}

In \cite{faster-css} the greedy removal procedure in Figure \ref{alg:greedy} is applied to the entire design set $S_0=[n]$,
which gives approximation guarantee $\tr[(X_{\hat S}^\top X_{\hat S})^{-1}]\leq \frac{n-p+1}{k-p+1}\tr[(X^\top X)^{-1}]$.
This results in an approximation ratio of $C(n,p,k)=\frac{n-p+1}{k-p+1}$ as defined in Eq.~(\ref{eq:minimax}),
by applying the trivial bound $\tr[(X^\top X)^{-1}]\leq f_2^*(k;X)$, which is tight for a design that has exactly $k$ non-zero rows.

To further improve the approximation ratio, we consider applying the greedy removal procedure
with $S_0$ equal to the support of $\pi^*$; that is, $S_0=\{j\in[n]: \pi^*_j > 0\}$.
Because $\|\pi^*\|_{\infty}\leq 1$ under the without replacement setting, we have the following corollary:
\begin{corollary}
Let $S_0$ be the support of $\pi^*$ and suppose $\|\pi^*\|_{\infty}\leq 1$. Then
$$
\tr[(X_{\hat S}^\top X_{\hat S})^{-1}]\leq \frac{\|\pi^*\|_0-p+1}{k-p+1}f(\pi^*;X) = \frac{\|\pi^*\|_0-p+1}{k-p+1}f_2^*(k;X).
$$
\label{cor:s0}
\end{corollary} 

It is thus important to upper bound the support size $\|\pi^*\|_0$.
With the trivial bound of $\|\pi^*\|_0\leq n$ we recover the $\frac{n-p+1}{k-p+1}$ approximation ratio by applying Figure \ref{alg:greedy}
to $S_0=[n]$.
In order to bound $\|\pi^*\|_0$ away from $n$, we consider the following assumption imposed on $X$:
\begin{assumption}
Define mapping $\phi:\mathbb R^p\to\mathbb R^{\frac{p(p+1)}{2}}$ as
$\phi(x)=(\xi_{ij}x(i)x(j))_{1\leq i\leq j\leq p}$,
where $x(i)$ denotes the $i$th coordinate of a $p$-dimensional vector $x$ and $\xi_{ij}=1$ if $i=j$ and $\xi_{ij}=2$ otherwise.
Denote $\tilde\phi(x)=(\phi(x),1)\in\mathbb R^{\frac{p(p+1)}{2}+1}$ as the affine version of $\phi(x)$.
For any $\frac{p(p+1)}{2}+1$ distinct rows of $X$, their mappings under $\tilde\phi$ are linear independent.
\label{asmp:general-position}
\end{assumption}

Assumption \ref{asmp:general-position} is essentially a general-position assumption, which 
assumes that no $\frac{p(p+1)}{2}+1$ design points in $X$ lie on a degenerate affine subspace after a specific quadratic mapping.
Like other similar assumptions in the literature \citep{lasso-unique},
Assumption \ref{asmp:general-position} is very mild and almost always satisfied in practice,
for example, if each row of $X$ is independently sampled from absolutely continuous distributions.

We are now ready to state the main lemma bounding the support size of $\pi^*$.
\begin{lemma}
$\|\pi^*\|_0\leq k+\frac{p(p+1)}{2}$ if Assumption \ref{asmp:general-position} holds.
\label{lem:support-bound}
\end{lemma}

Lemma \ref{lem:support-bound} is established by an interesting observation into the properties of Karush-Kuhn-Tucker (KKT) conditions
of the optimization problem Eq.~(\ref{eq_c_opt}), which involves a linear system with $\frac{p(p+1)}{2}+1$ variables.
The complete proof of Lemma \ref{lem:support-bound} is given in Sec.~\ref{subsec:proof-support}.
{ To contrast the results in Lemma \ref{lem:support-bound} with 
classical rank/support bounds in SDP and/or linear programming (e.g. the Pataki's bound \citep{pataki1998rank}),
note that the number of constraints in the SDP formulation of Eq.~(\ref{eq_c_opt}) (see also Appendix \ref{appsec:optimization})
is linear in $n$, and hence analysis similar to \citep{pataki1998rank} would result in an upper bound of $\|\pi^*\|_0$ that scales with $n$,
which is less useful for our analytical purpose.
}

Combining results from both Lemma \ref{lem:support-bound} and Corollary \ref{cor:s0} we arrive at the following theorem,
which upper bounds the approximation ratio of the greedy removal procedure in Figure \ref{alg:greedy} initialized by the support of $\pi^*$.
\begin{theorem}
Let $\pi^*$ be the optimal solution of the without replacement version of Eq.~(\ref{eq_c_opt}) and $\hat S$ be the output of 
the greedy removal procedure in Figure \ref{alg:greedy} initialized with $S_0=\{i\in[n]: \pi_i^*>0\}$.
If $k>p$ and Assumption \ref{asmp:general-position} holds then the subset OLS estimator $\hat\beta=(X_{\hat S}^\top X_{\hat S})^{-1}X_{\hat S}^\top y_{\hat S}$ satisfies
$$
\mathbb E\left[\|\hat\beta-\beta_0\|\big| X_{\hat S}\right] = \sigma^2\tr\left[(X_{\hat S}^\top X_{\hat S})^{-1}\right] \leq \left(1+\frac{p(p+1)}{2(k-p+1)}\right)f_2^*(k;X).
$$
\label{thm:greedy}
\end{theorem}
Under a slightly stronger condition that $k> 2p$, the approximation ratio $C(n,k,p)=1+\frac{p(p+1)}{2(k-p+1)}$ can be simplified to
$C(n,k,p)=1+O(p^2/k)$.
In addition, $C(n,k,p)=1+o(1)$ if $p^2/k\to 0$,
meaning that near-optimal experiment selection is achievable with computationally tractable methods
if $O(p^2)$ design points are allowed in the selected subset.

\subsection{Interpretable subsampling example: anisotropic Gaussian design}

Even though we consider a fixed pool of design points so far, here we use an anisotropic Gaussian design example to demonstrate that a non-uniform sampling can outperform uniform sampling even 
under random designs, and to interpret the conditions required in previous analysis. Let
 $x_1,\cdots,x_n$ be i.i.d.~distributed according to an isotropic Gaussian distribution $\mathcal N_p(0,\Sigma_0)$.

We first show that non-uniform weights $\pi_i\neq k/n$ could improve the objective $f(\pi;X)$.
Let $\pi^{\unif}$ be the uniformly weighted solution of $\pi^{\unif}_i=k/n$, corresponding to selecting each row of $X$ uniformly at random.
We then have 
$$
f(\pi^{\unif};X) = \frac{1}{k}\tr\left[\left(\frac{1}{n}X^\top X\right)^{-1}\right] \overset{p}{\to} \frac{1}{k}\tr(\Sigma_0^{-1}).
$$
On the other hand, let $B^2=\gamma\tr(\Sigma_0)$ for some universal constant $\gamma>1$ and $\mathbb B=\{x\in\mathbb R^p: \|x\|_2^2\leq B^2\}$,
$\mathbb X=\{x_1,\cdots,x_n\}$.
By Markov inequality, $|\mathbb X\cap\mathbb B|\gtrsim \frac{\gamma n}{1-\gamma}$.
Define weighted solution $\pi^w$ as $\pi_i^w \propto 1/p(x_i|\mathbb B)\cdot I[x_i\in\mathbb B]$ normalized such that $\|\pi_i^w\|_1=k$.
\footnote{Note that $\pi^*$ may \emph{not} be the optimal solution of Eq.~(\ref{eq_c_opt}); however, it suffices for the purpose of the demonstration of improvement resulting from non-uniform weights.}
Then
\begin{multline*}
f(\pi^w;X) = \frac{1}{k}\frac{\tr\left[\left(\frac{1}{n}\sum_{x_i\in\mathbb X\cap\mathbb B}{\frac{x_ix_i^\top}{p(x_i|\mathbb B)}}\right)^{-1}\right]}{\frac{1}{n}\sum_{x_i\in\mathbb X\cap\mathbb B}{1/p(x_i|\mathbb B)}}\\
\overset{p}{\to} \frac{1}{k}\frac{n}{|\mathbb X\cap\mathbb B|}\frac{\tr\left[\left(\int_{\mathbb B}xx^\top\ud x\right)^{-1}\right]}{\int_{\mathbb B}1\ud x}
\lesssim \frac{p^2}{(1-\gamma)\tr(\Sigma_0)}.
\end{multline*}
Here in the last inequality we apply Lemma \ref{lem:integration}.
Because $\frac{\tr(\Sigma_0^{-1})}{p} = \frac{1}{p}\sum_{i=1}^p{\frac{1}{\sigma_i(\Sigma_0)}} \geq \left(\frac{1}{p}\sum_{i=1}^p{\sigma_i(\Sigma_0)}\right)^{-1}=\frac{p}{\tr(\Sigma_0)}$
by Jensen's inequality, we conclude that in general $f(\pi^w;X) < f(\pi^{\unif};X)$, and the gap is larger for
ill-conditioned covariance $\Sigma_0$.
This example shows that uneven weights in $\pi$ helps reducing the trace of inverse of the weighted covariance $X^\top\diag(\pi)X$. 

Under this model, we also simplify the conditions for the without replacement model in theorem \ref{thm:sampling_expected} and \ref{thm:sampling_deterministic}.
Because $x_1,\cdots,x_n\overset{i.i.d.}{\sim}\mathcal N_p(0,\Sigma_0)$,
it holds that $\|X\|_{\infty}^2 \leq O_\mP(\|\Sigma_0\|_2^2p\log n)$.
In addition, by simple algebra $\|\Sigma_*^{-1}\|_2 \leq  p^{-1}\kappa(\Sigma_*)\tr(\Sigma_*^{-1})
\leq p^{-1}\kappa(\Sigma_*)f_1^*(k;X)$.
Using a very conservative upper bound of $f_1^*(k;X)$ by sampling rows in $X$ uniformly at random and apply weak law of large numbers
and the continuous mapping theorem, we have that
$f_1^*(k;X)\lesssim \frac{1}{k}\tr(\Sigma_0^{-1})$.
In addition, $\tr(\Sigma_0^{-1})\|\Sigma_0\|_2\leq p\|\Sigma_0^{-1}\|_2\|\Sigma_0\|_2 = p\kappa(\Sigma_0)$.
Subsequently, the condition $\|\Sigma_*^{-1}\|_2\kappa(\Sigma_*)\|X\|_{\infty}^2\log p = O(\epsilon^2)$ is implied by
\begin{equation}
\frac{p\kappa(\Sigma_*)^2\kappa(\Sigma_0)\log p\log n}{k} = O(\epsilon^2).
\label{eq:gaussian-condition}
\end{equation}
Essentially, the condition is reduced to $k\gtrsim \kappa(\Sigma_0)\kappa(\Sigma_*)\cdot p\log n\log p$.
The linear dependency on $p$ is necessary, as we consider the low-dimensional linear regression problem and $k<p$ would imply an infinite 
mean-square error in estimation of $\beta_0$.
We also remark that the condition is scale-invariant, as $X'=\xi X$ and $X$ share the same quantity $\|\Sigma_*^{-1}\|_2\kappa(\Sigma_*)\|X\|_{\infty}^2\log k$.

\subsection{Extensions}

We discuss possible extension of our results beyond estimation of $\beta_0$ in the linear regression model.

\subsubsection{Generalized linear models}\label{subsec:glm}
In a generalized linear model $\mu(x)=\mathbb E[Y|x]$ satisfies $g(\mu(x))=\eta=x^\top\beta_0$
for some known link function $g:\mathbb R\to\mathbb R$.
Under regularity conditions \citep{van2000asymptotic}, the maximum-likelihood estimator $\hat\beta_n\in\argmax_{\beta}\{\sum_{i=1}^n{\log p(y_i|x_i;\beta)}\}$
satisfies $\mathbb E\|\hat\beta_n-\beta_0\|_2^2 = (1+o(1))\tr(I(X,\beta_0)^{-1})$, where
$I(X,\beta_0)$ is the Fisher's information matrix:
\begin{equation}
I(X,\beta_0) = -\sum_{i=1}^n{\mathbb E\frac{\partial^2\log p(y_i|x_i;\beta_0)}{\partial\beta\partial\beta^\top}}
= -\sum_{i=1}^n{\left(\mathbb E\frac{\partial^2\log p(y_i;\eta_i)}{\partial\eta_i^2}\right)x_ix_i^\top}.
\label{eq:glm_fisher}
\end{equation}
Here both expectations are taken over $y$ conditioned on $X$ and the last equality is due to the sufficiency of $\eta_i=x_i^\top\beta_0$.
The experiment selection problem is then formulated to select a subset $S\subseteq[n]$ of size $k$, either with or without duplicates,
that minimizes $\tr(I(X_S,\beta_0)^{-1})$.

It is clear from Eq.~(\ref{eq:glm_fisher}) that the optimal subset $S^*$ depends on the unknown parameter $\beta_0$, which itself is to be estimated.
This issue is known as the \emph{design dependence} problem for generalized linear models \citep{khuri2006design}.
One approach is to consider \emph{locally optimal designs} \citep{khuri2006design,chernoff1953locally}, where a consistent estimate $\check\beta$ of $\beta_0$
is first obtained on an initial design subset \footnote{Notice that a consistent estimate can be obtained using much fewer points than an estimate with finite approximation guarantee.}
and then $\check\eta_i=x_i^\top\check\beta$ is supplied to compute a more refined design subset
to get the final estimate $\hat\beta$.
With the initial estimate $\check\beta$ available, one may apply transform $x_i\mapsto \tilde x_i$ defined as
$$
\tilde x_i=\sqrt{-\mathbb E\frac{\partial^2\log p(y_i;\check\eta_i)}{\partial\eta^2}}x_i.
$$
Note that under regularity conditions $-\mathbb E\frac{\partial^2\log p(y_i;\check\eta_i)}{\partial\eta_i^2}=\mathbb E\left(\frac{\partial\log(y_i;\check x_i)}{\partial\eta_i}\right)^2$ is non-negative and hence the square-root is well-defined.
All results in Theorems \ref{thm:sampling_expected}, \ref{thm:sampling_deterministic} and \ref{thm:greedy} are valid with $X=[x_1,\cdots,x_n]^\top$ replaced by $\widetilde X=[\tilde x_1,\cdots,\tilde x_n]^\top$ for generalized linear models.
Below we consider two generalized linear model examples and derive explicit forms of $\widetilde X$.

\paragraph{Example 1: Logistic regression}
In a logistic regression model responses $y_i\in\{0,1\}$ are binary and the likelihood model is
$$
p(y_i;\eta_i) = \psi(\eta_i)^{y_i}(1-\psi(\eta_i))^{1-y_i}, \;\;\;\;\;\;\text{where}\;\;\;\;
\psi(\eta_i) = \frac{e^{\eta_i}}{1+e^{\eta_i}}.
$$
Simple algebra yields
$$
\widetilde x_i = \sqrt{\frac{e^{\check\eta_i}}{(1+e^{\check\eta_i})^2}}x_i,
$$
where $\check\eta_i=x_i^\top\check\beta$.

\paragraph{Example 2: Poisson count model} In a Poisson count model the response variable $y_i$ takes values of non-negative integers
and follows a Poisson distribution with parameter $\lambda=e^{\eta_i}=e^{x_i^\top\beta_0}$.
The likelihood model is formally defined as
$$
p(y_i=r;\eta_i) = \frac{e^{\eta_i r}e^{-e^{\eta_i}}}{r!}, \;\;\;\;\;\; r=0,1,2,\cdots.
$$
Simple algebra yields
$$
\widetilde x_i = \sqrt{e^{\check\eta_i}}x_i,
$$
where $\check\eta_i=x_i^\top\check\beta$.

\subsubsection{Delta's method}
Suppose $g(\beta_0)$ is the quantity of interest, where $\beta_0\in\mathbb R^p$ is the parameter in a linear regression model
and $g:\mathbb R^{p}\to\mathbb R^m$ is some known function.
Let $\hat\beta_n=(X^\top X)^{-1}X^\top y$ be the OLS estimate of $\beta_0$.
If $\nabla g$ is continuously differentiable and $\hat\beta_n$ is consistent, then by the classical delta's method \citep{van2000asymptotic}
$\mathbb E\|g(\hat\beta_n)-g(\beta_0)\|_2^2 = (1+o(1))\sigma^2\tr(\nabla g(\beta_0)(X^\top X)^{-1}\nabla g(\beta_0)^\top)=(1+o(1))\sigma^2\tr(G_0(X^\top X)^{-1})$,
where $G_0=\nabla g(\beta_0)^\top\nabla g(\beta_0)$.
If $G_0$ depends on the unknown parameter $\beta_0$ then the design dependence problem again exists, 
and a locally optimal solution can be obtained by replacing $G_0$ in the objective function with $\check G=\nabla g(\check\beta)^\top\nabla g(\check\beta)$
for some initial estimate $\check\beta$ of $\beta_0$.

If $\check G$ is invertible, then there exists invertible $p\times p$ matrix $\check P$ such that $\check G=\check P\check P^\top$ because 
$\check G$ is positive definite.
Applying the linear transform
$$
x_i \mapsto \tilde x_i = \check P^{-1}x_i
$$
we have that $\tr[G_0(X^\top X)^{-1}]=\tr[(\widetilde X^\top\widetilde X)^{-1}]$, where $\widetilde X=[\tilde x_1,\cdots,\tilde x_n]^\top$.
Our results in Theorems \ref{thm:sampling_expected}, \ref{thm:sampling_deterministic} and \ref{thm:greedy} remain valid by operating on the transformed matrix 
$\widetilde X=X\check P^{-\top}$.

{
\paragraph{Example: prediction error.}
In some application scenarios the \emph{prediction error} $\|Z\hat\beta-Z\beta_0\|_2^2$
rather than the estimation error $\|\hat\beta-\beta_0\|_2^2$ is of interesting, either because
the linear model is used mostly for prediction or component of the underlying model $\beta_0$ lack physical interpretations.
Another interesting application is the \emph{transfer learning} \citep{pan2010survey}, in which the training and testing data have different designs (e.g., $Z$ instead of $X$)
but share the same conditional distribution of labels, parameterized by the linear model $\beta_0$.

Suppose $Z\in\mathbb R^{m\times p}$ is a known full-rank data matrix upon which predictions are seeked,
and define $\hat\Sigma_Z=\frac{1}{m}Z^\top Z\succ 0$ to be the sample covariance of $Z$.
Our algorithmic framework as well as its corresponding analysis remain valid for such prediction problems
with transform $x_i\mapsto \hat\Sigma_Z^{-1/2}x_i$.
In particular, the guarantees for the greedy algorithm and the with replacement sampling algorithm remain unchanged, 
and the guarantee for the without replacement sampling algorithm is valid as well, except that the 
$\|\Sigma_*^{-1}\|_2$ and $\kappa(\Sigma_*)$ terms have to be replaced by the (relaxed) optimal sample covariance after the linear transform $x_i\mapsto\hat\Sigma_Z^{-1/2}x_i$.
}

\section{Numerical results on synthetic data}\label{subsec:synthetic}


\begin{figure}[p]
\centering
\includegraphics[width=5cm]{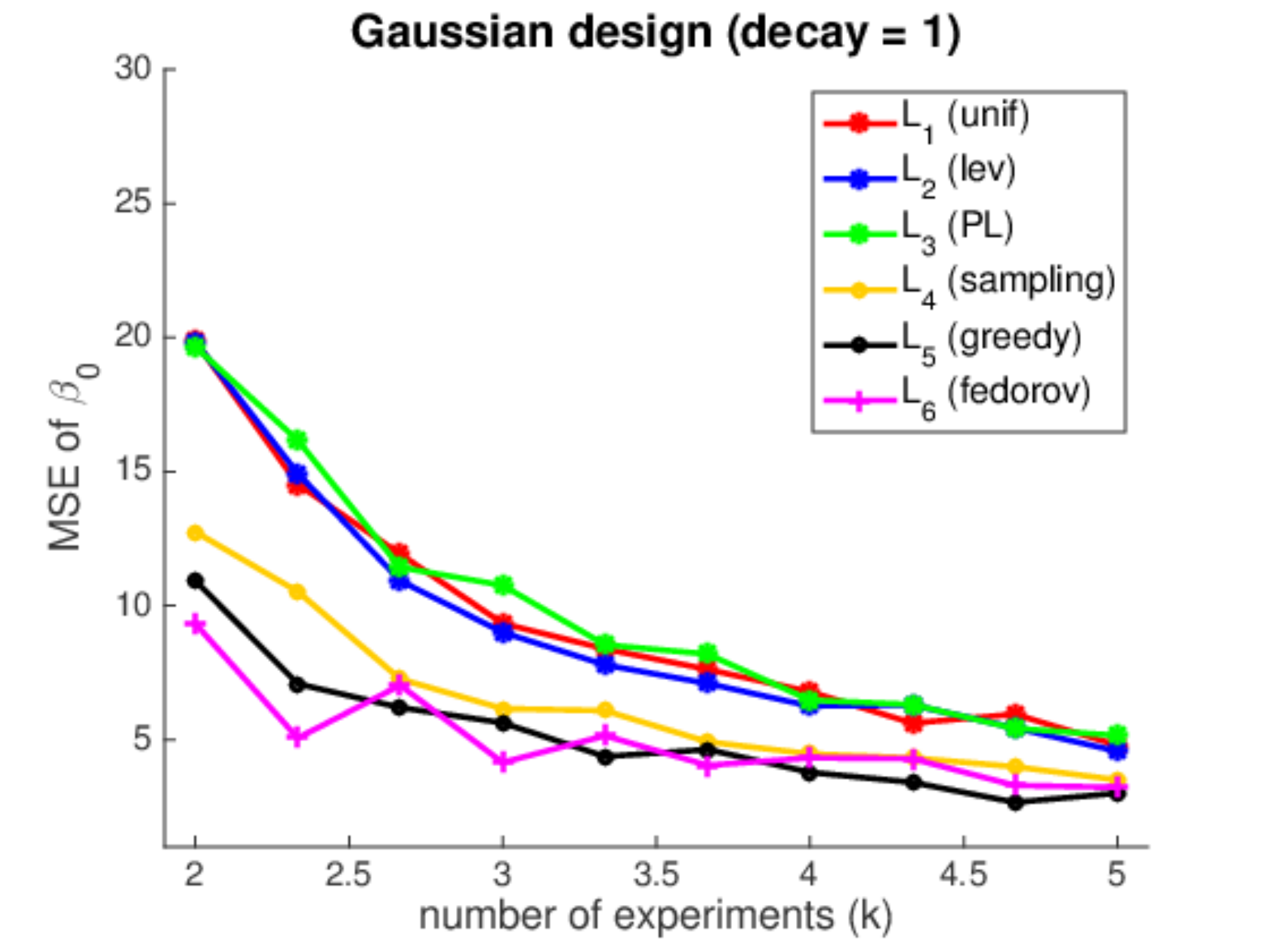}
\includegraphics[width=5cm]{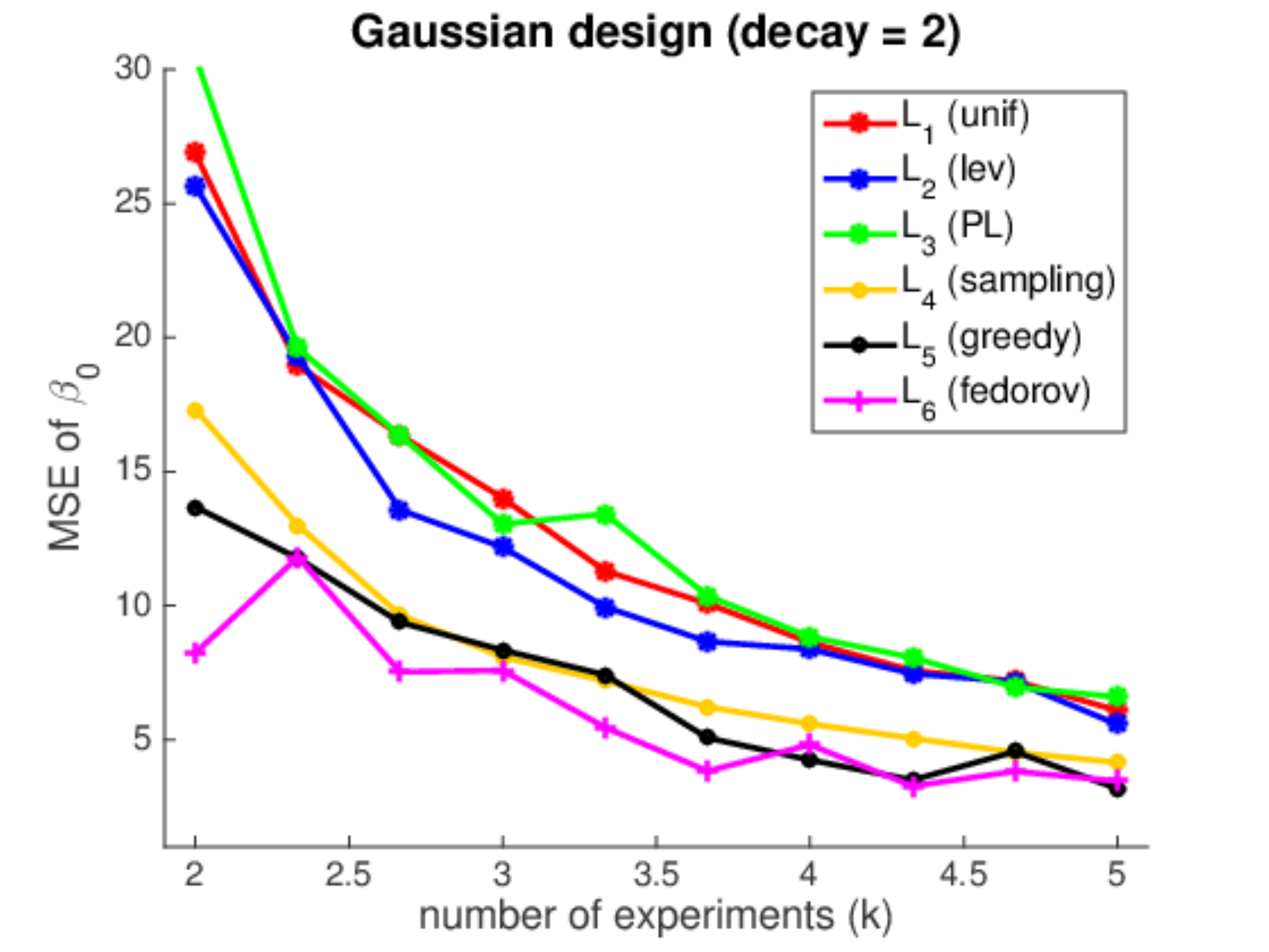}
\includegraphics[width=5cm]{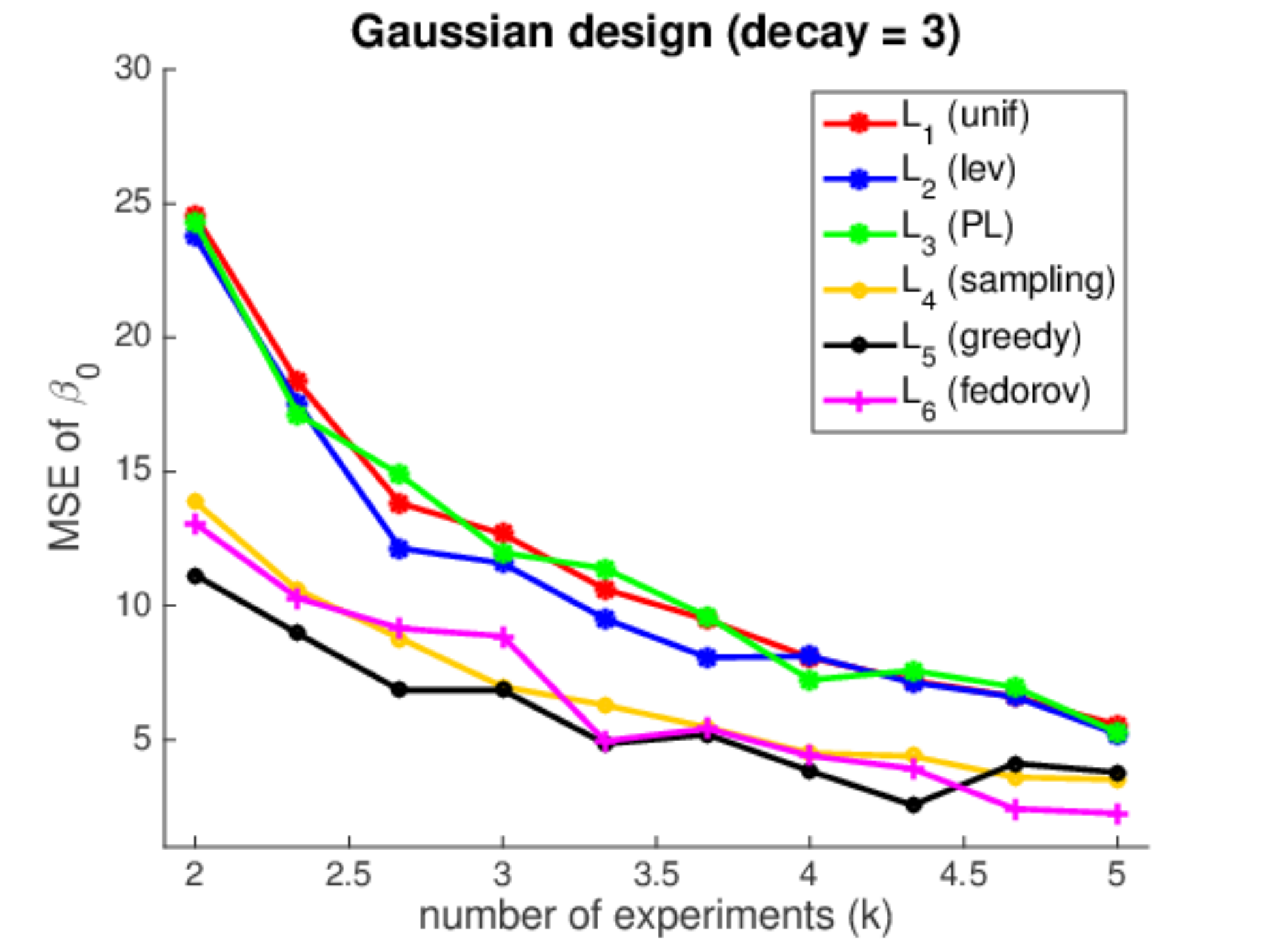}
\includegraphics[width=5cm]{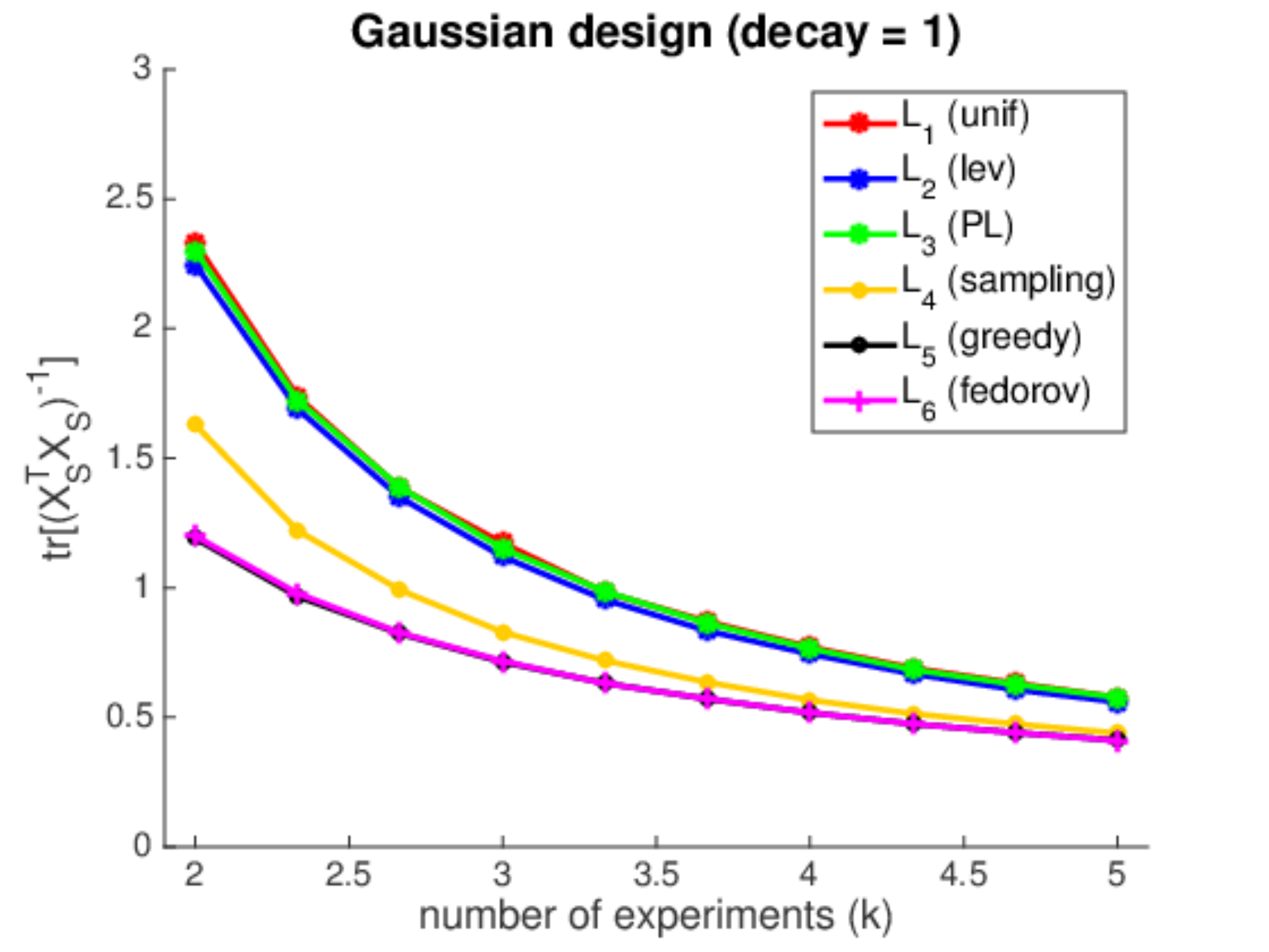}
\includegraphics[width=5cm]{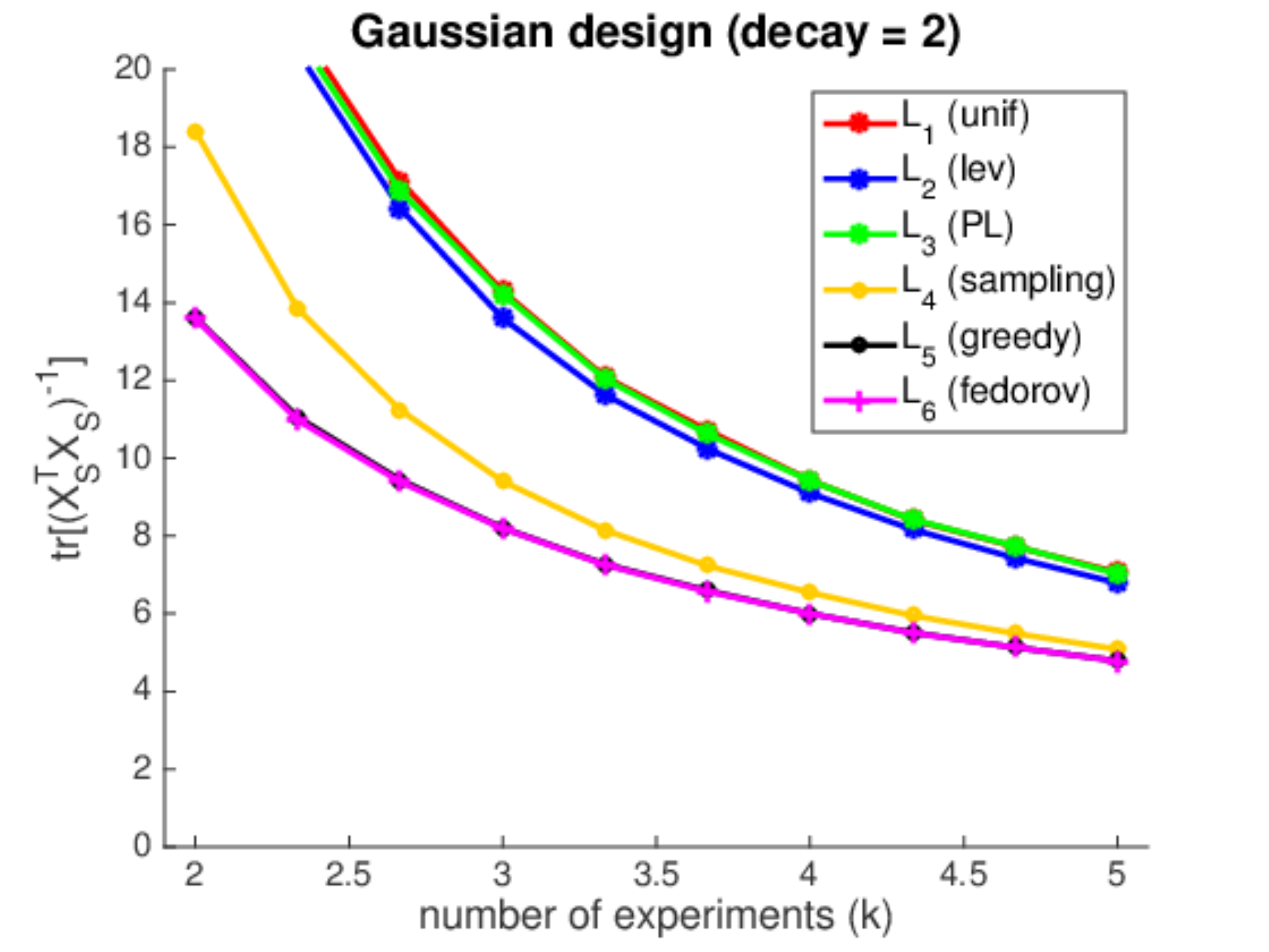}
\includegraphics[width=5cm]{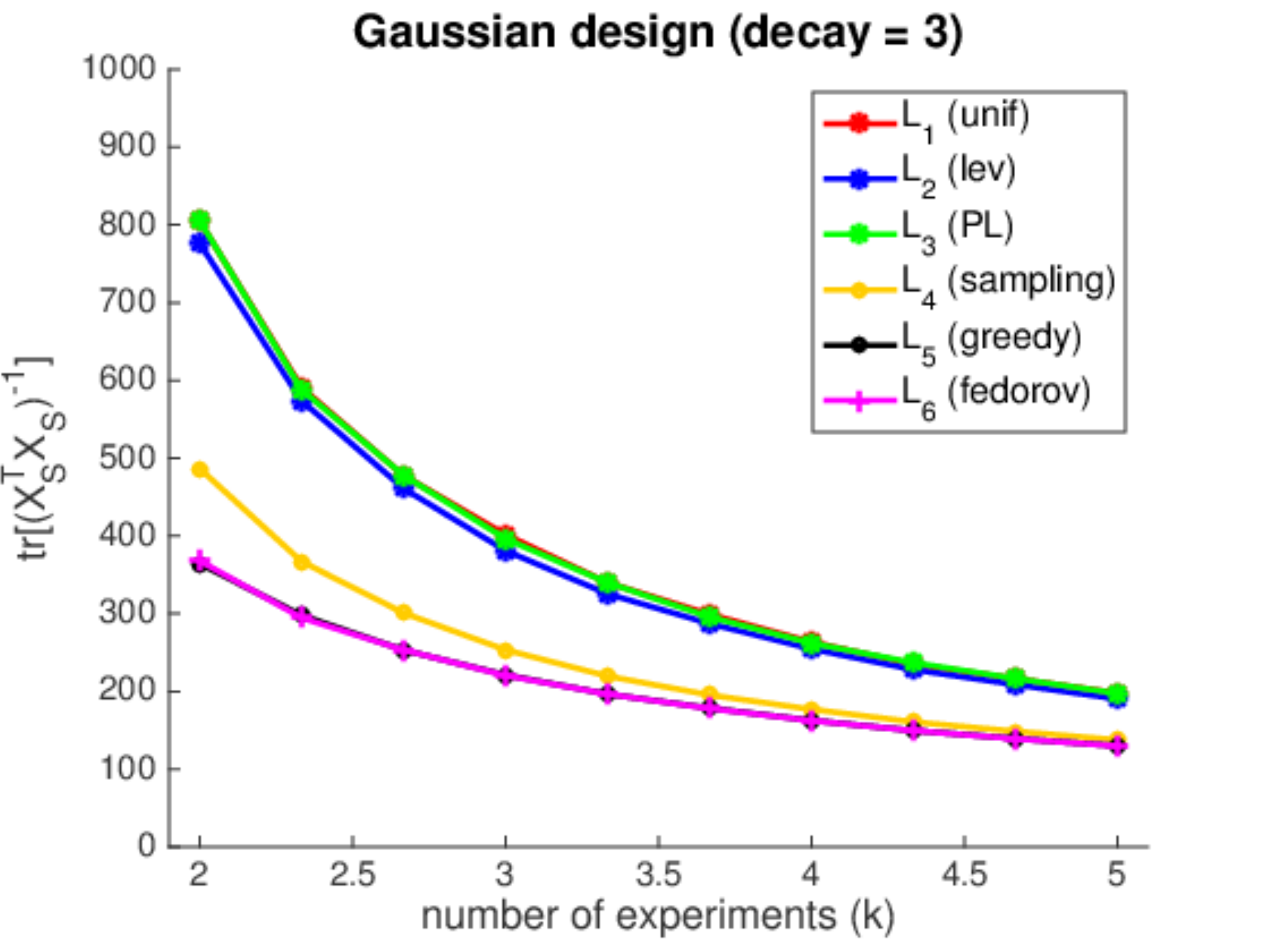}
\vskip 0.5in
\includegraphics[width=5cm]{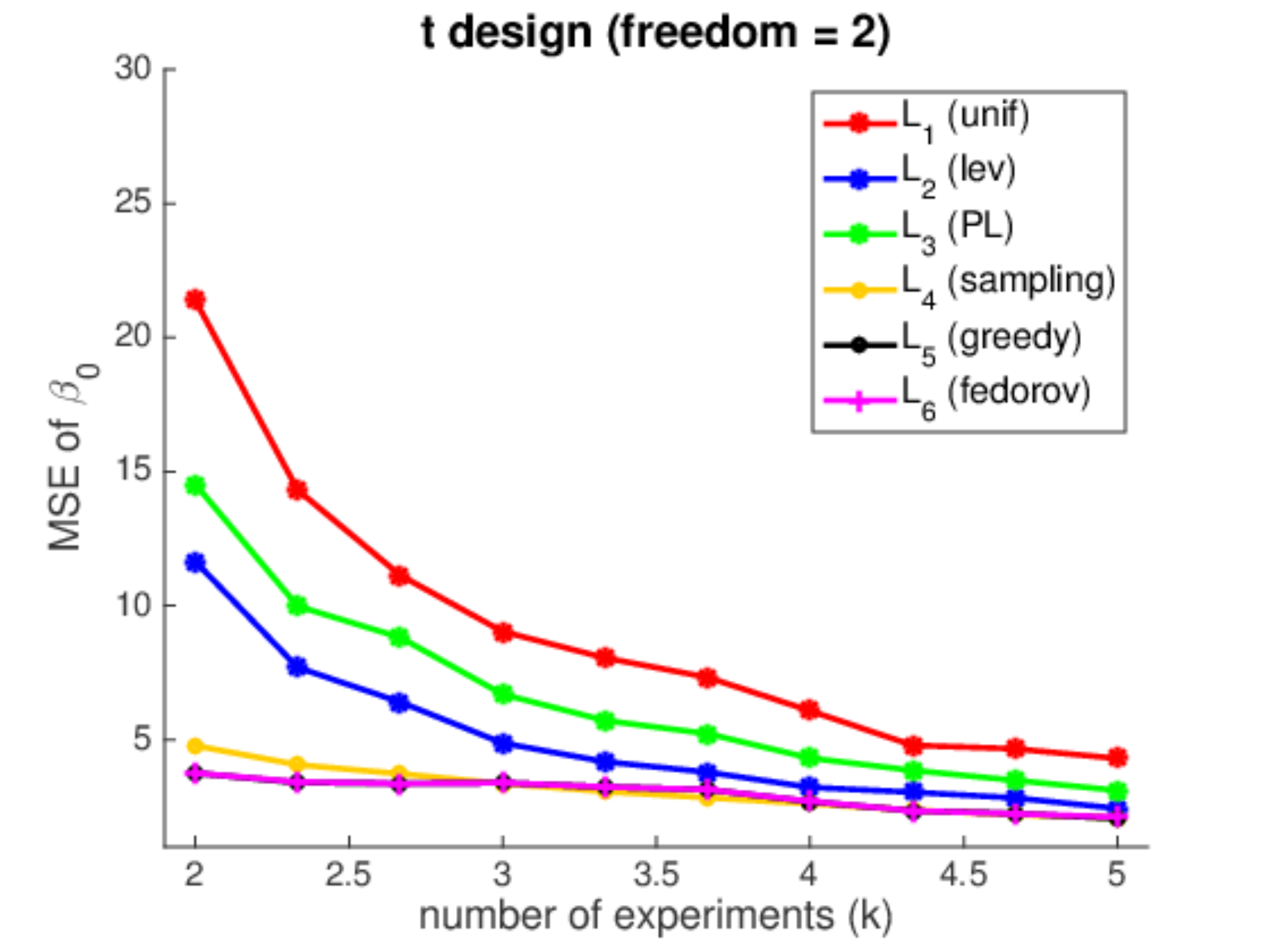}
\includegraphics[width=5cm]{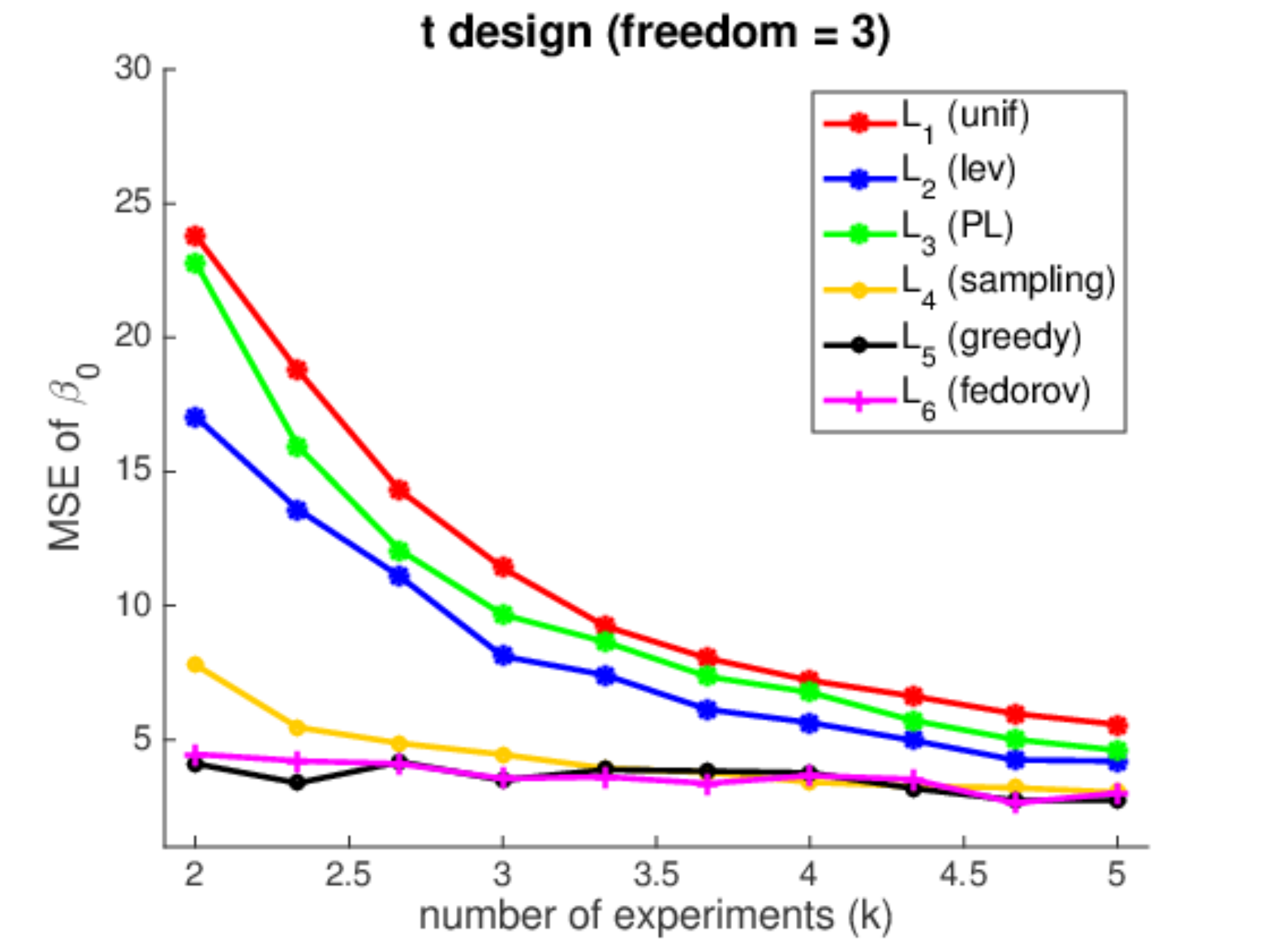}
\includegraphics[width=5cm]{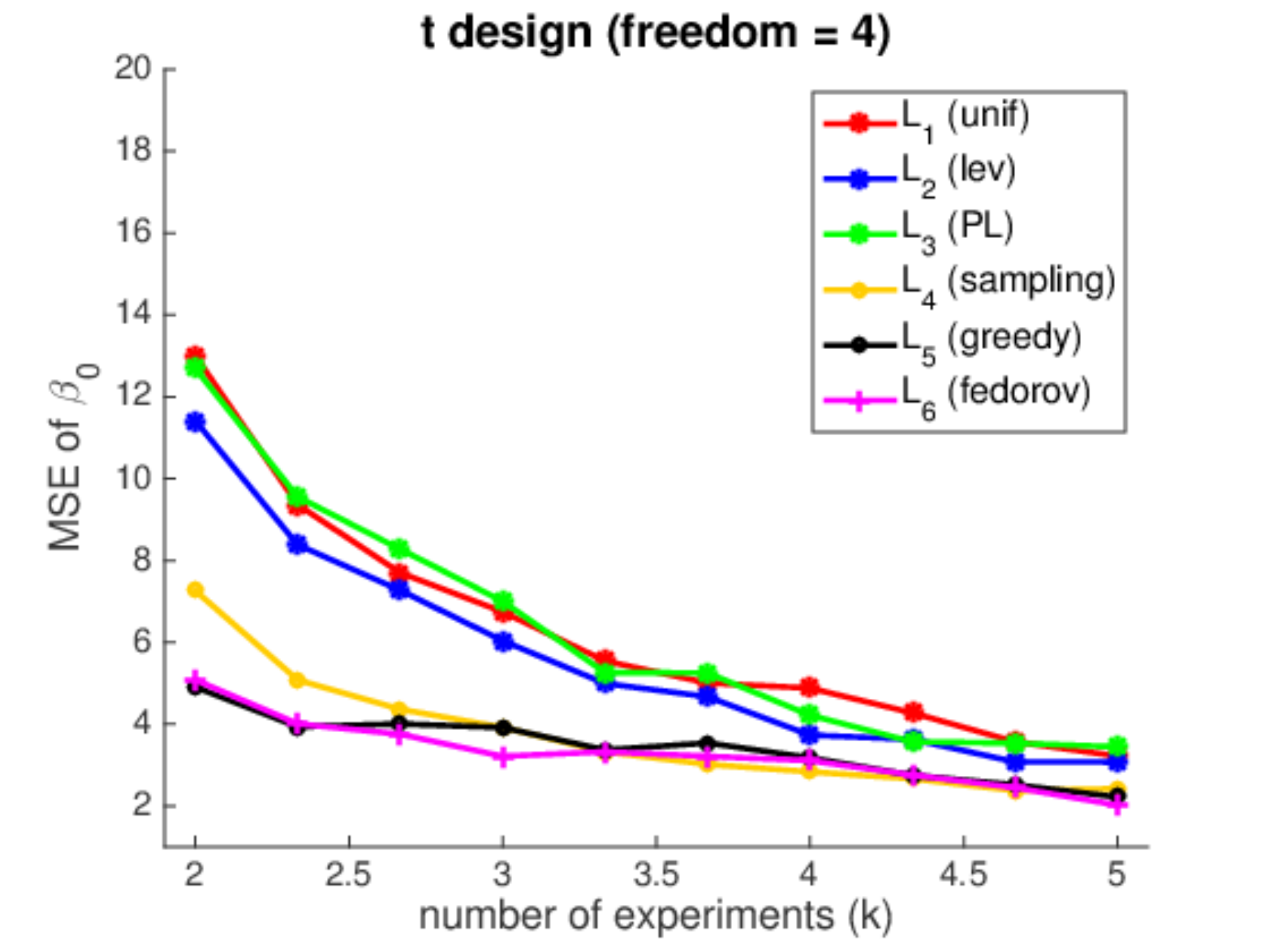}
\includegraphics[width=5cm]{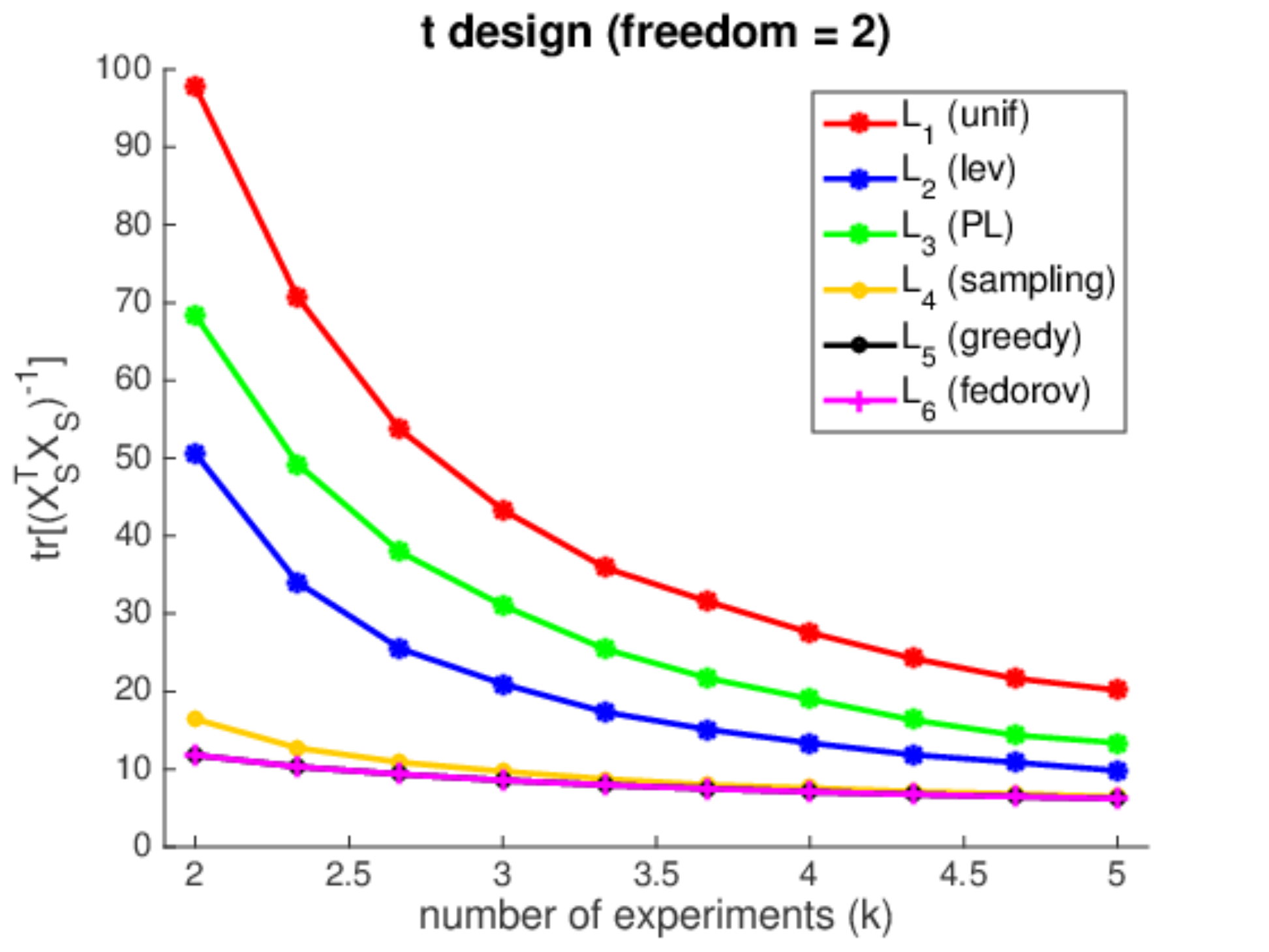}
\includegraphics[width=5cm]{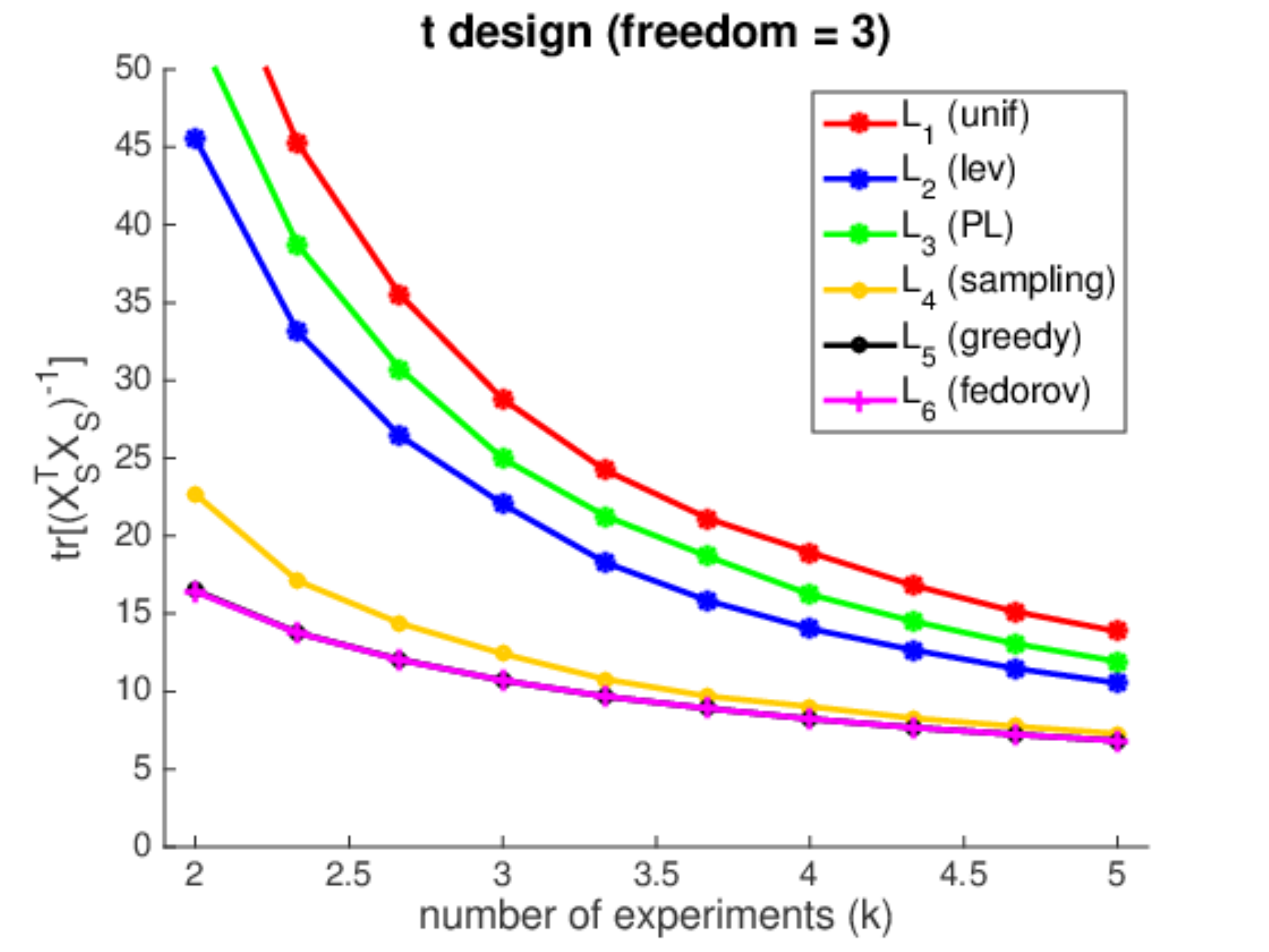}
\includegraphics[width=5cm]{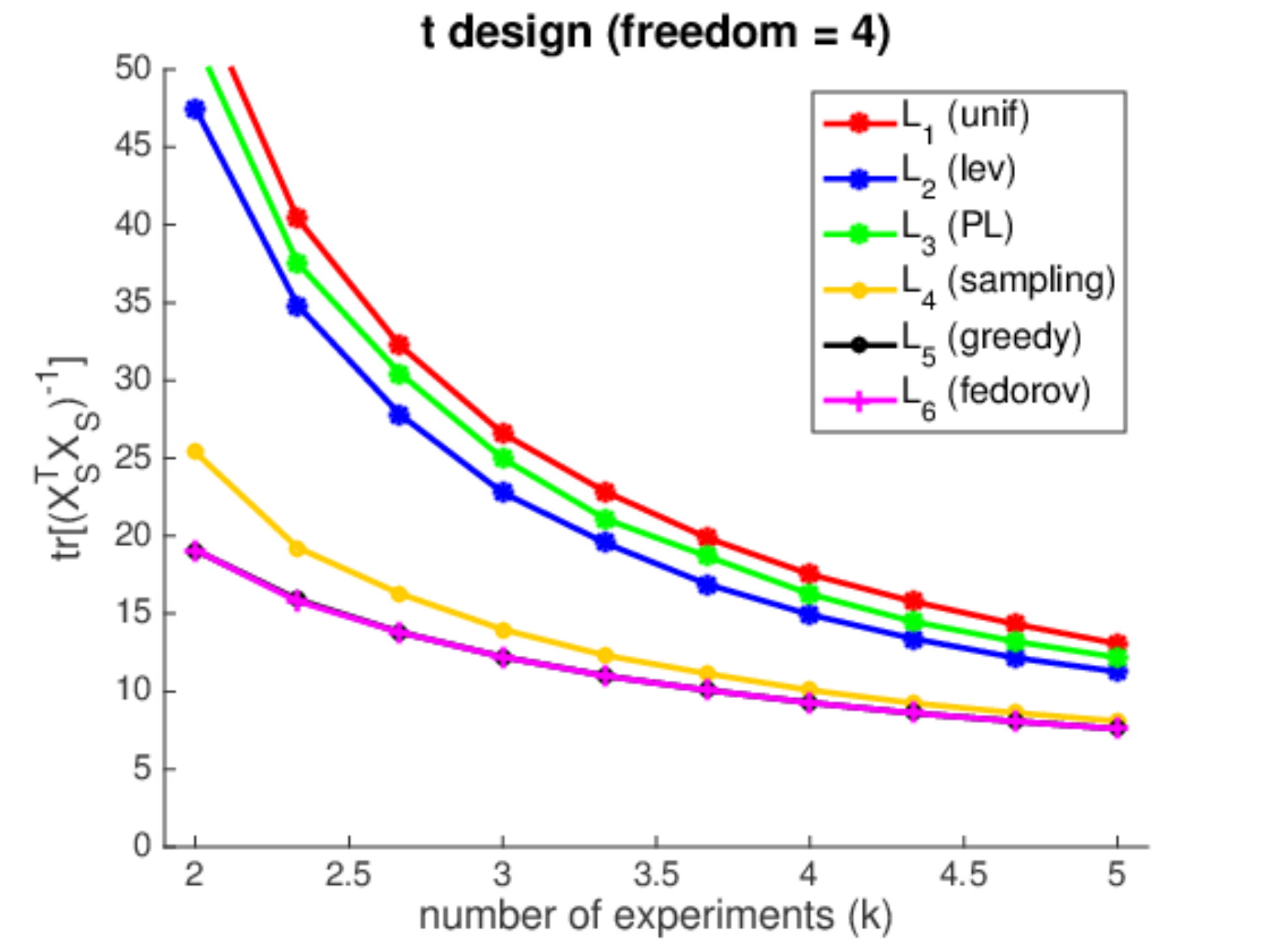}
\caption{\small The ratio of MSE $\|\hat\beta-\beta_0\|_2^2$ compared against the MSE of the OLS estimator on the full data pool $\|\hat\beta^{\mathrm{ols}}-\beta_0\|_2^2$, $\hat\beta^{\mathrm{ols}}=(X^\top X)^{-1}X^\top y$,
and objective values $F(\hat S;X)=\tr[(X_{\hat S}^\top X_{\hat S})^{-1}]$ on synthetic data sets. 
{Note that the $L_5$ (greedy) curve is under the $L_6$ (federov) curve when not visible.}
See Sec.~\ref{subsec:synthetic} for details of experimental settings.}
\label{fig:synthetic}
\end{figure}

\begin{table}[t]
\centering
\caption{Running time (seconds) / no. of iterations for $L_5^*$ and $L_6^{\dagger}$.}
\begin{tabular}{lccccc}
\hline
& \multicolumn{5}{c}{Isotropic design ($t$ distribution), $\mathrm{df}=1$.}\\
$k=$& $2p$&  $4p$&  $6p$&  $8p$&  $10p$\\
\hline
$L_5^*$& 1.8/31& 0.8/19& 1.7/26& 0.9/14& 0.3/9\\
$L_6^\dagger$&  140.1/90& 425.9/156& 767.5/216& 993.7/245& 1077/253 \\
\hline
& \multicolumn{5}{c}{Isotropic design ($t$ distribution), $\mathrm{df}=3$.}\\
$k=$& $2p$&  $4p$&  $6p$&  $8p$&  $10p$\\
\hline
$L_5^*$& 0.6/14& 0.4/8& 0.3/7& 0.2/5& 0.2/5 \\
$L_6^\dagger$& 158.3/104& 438.0/161& 802.5/223& 985.2/242& 1105/252  \\
\hline
& \multicolumn{5}{c}{Skewed design (multivariate Gaussian), $\alpha=3$.}\\
$k=$& $2p$&  $4p$&  $6p$&  $8p$&  $10p$\\
\hline
$L_5^*$& 0.8/16& 0.7/12& 0.5/9& 0.4/8& 0.4/8\\
$L_6^\dagger$& 182.9/120& 487.1/180& 753.4/212& 935.8/230& 1057/250\\
\hline
\end{tabular}
\label{tab:time}
\end{table}

We report selection performance (measured in terms of {$F(S;X) : = \tr[(X_S^\top X_S)^{-1}]$ which is 
the mean squared error of the ordinary least squares estimator based on $(X_S, y_S)$ in estimating the regression coefficients}) on synthetic data.
Only the without replacement setting is considered, and results for the with replacement setting are similar.
In all simulations, the experimental pool (number of given design points) $n$ is set to $1000$, number of variables $p$ is set to $50$,
number of selected design points $k$ ranges from $2p$ to $10p$.
For randomized methods, we run them for 20 independent trials under each setting and report the median.
Though Fedorov's exchange algorithm is randomized in nature (outcome depending on the initialization),
we only run it once for each setting because of its expensive computational requirement.
It is observed that in practice, the algorithm's outcome is not sensitive to initializations.

\subsection{Data generation}
The design pool $X$ is generated so that each row of $X$ is sampled from some underlying distribution.
We consider two distributional settings for generating $X$.
Similar settings were considered in \citep{levscore-regression} for subsampling purposes.

\begin{enumerate}

\item \emph{Distributions with skewed covariance}:
each row of $X$ is sampled i.i.d.~from a multivariate Gaussian distribution $\mathcal N_p(0,\Sigma_0)$
with $\Sigma_0=U\Lambda U^\top$, where $U$ is a random orthogonal matrix and $\Lambda=\diag(\lambda_1,\cdots,\lambda_p)$ controls the skewness or conditioning 
of $X$. A power-law decay of $\lambda$, $\lambda_j=j^{-\alpha}$, is imposed, with small $\alpha$ corresponding to ``flat'' distribution
and large $\alpha$ corresponding to ``skewed'' distribution.
Rows 1-2 in Figure \ref{fig:synthetic} correspond to this setting.

\item\emph{Distributions with heavy tails}:
We use $t$ distribution as the underlying distribution for generating each entry in $X$,
with degrees of freedom ranging from $2$ to $4$.
These distributions are heavy tailed and high-order moments of $X$ typically do not exist.
They test the robustness of experiment selection methods.
Rows 3-4 in Figure \ref{fig:synthetic} correspond to this setting.


\end{enumerate}

%

\subsection{Methods}
The methods that we compare are listed below:
\begin{itemize}
\item[-] $L_1$ (uniform sampling): each row of $X$ is sampled uniformly at random, without replacement.
\item[-] $L_2$ (leverage score sampling): each row of $X$, $x_i$, is sampled without replacement, with probability proportional to its leverage score
$x_i^\top(X^\top X)^{-1}x_i$.
This strategy is considered in \cite{levscore-regression} for subsampling in linear regression models.
\item[-] $L_3$ (predictive length sampling): each row of $X$, $x_i$, is sampled without replacement, with probability proportional to its $\ell_2$ norm $\|x_i\|_2$.
This strategy is derived in \cite{optimal-subsampling}.
\item[-] $L_4^*$ (sampling based selection with $\pi^*$): the sampling based method that is described in Sec.~\ref{subsec:sampling_deterministic},
based on $\pi^*$, the optimal solution of Eq.~(\ref{eq_c_opt}).
We consider the hard size constraint algorithm only.
\item[-] $L_5^*$ (greedy based selection with $\pi^*$): the greedy based method that is described in Sec.~\ref{subsec:greedy}.
\item[-] $L_6^\dagger$ (Fedorov's exchange algorithm): the Fedorov's exchange algorithm \citep{fedorov-exchange}
with the A-optimality objective. 
Note that this method does not have guarantees, nor does it provably terminate in polynomial number of exchanges.
Description of the algorithm is placed in Appendix \ref{appsec:fedorov}.
\end{itemize}

\subsection{Performance}

The ratio of the mean square error $\|\hat\beta-\beta_0\|_2^2$ compared to the OLS estimator on the full data set $\|\hat\beta^{\mathrm{ols}}-\beta_0\|_2^2$,
$\hat\beta^{\mathrm{ols}}=(X^\top X)^{-1}X^\top y$,
 and objective values $F(\hat S;X)=\tr((X_{\hat S}^\top X_{\hat S})^{-1})$ are reported in Figure \ref{fig:synthetic}.

Table \ref{tab:time} reports the running time and number of iterations of $L_5^*$ (greedy based selection) and $L_6^{\dagger}$ (Fedorov's exchange algorithm).
In general, the greedy method is 100 to 1000 times faster than the exchange algorithm, and also converges in much fewer iterations.

{
For both synthetic settings, our proposed methods ($L_4^*$ and $L_5^*$) significantly outperform existing approaches $(L_1,L_2,L_3)$,
and their performance is on par with the Fedorov's exchange algorithm, which is much more computationally expensive (cf. Table \ref{tab:time}).
 }

\section{Numerical results on real data}\label{sec:experiment}

\subsection{The material synthesis dataset}\label{subsec:material}

\begin{figure}[t]
\centering
\includegraphics[width=6cm]{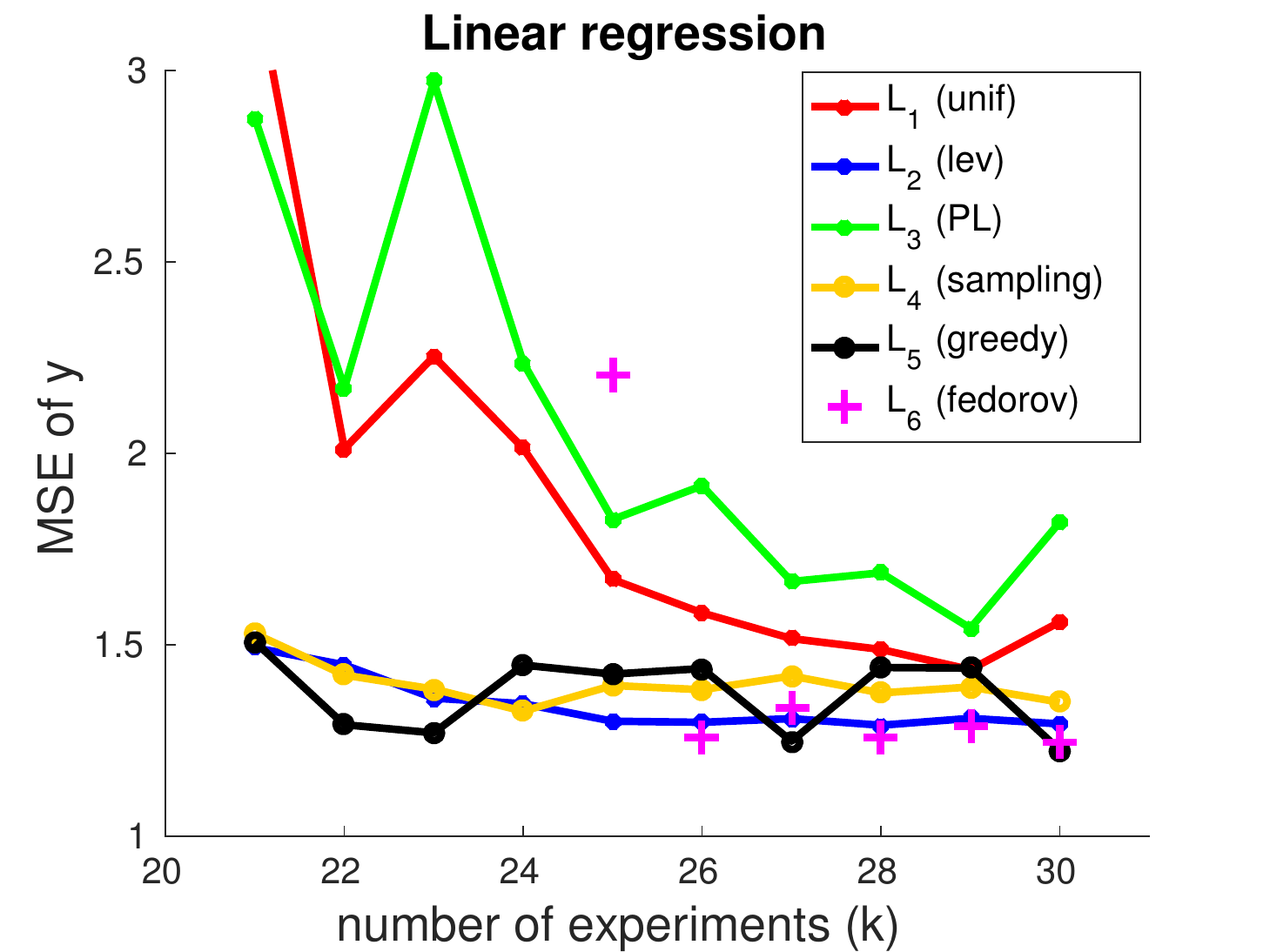}
\includegraphics[width=6cm]{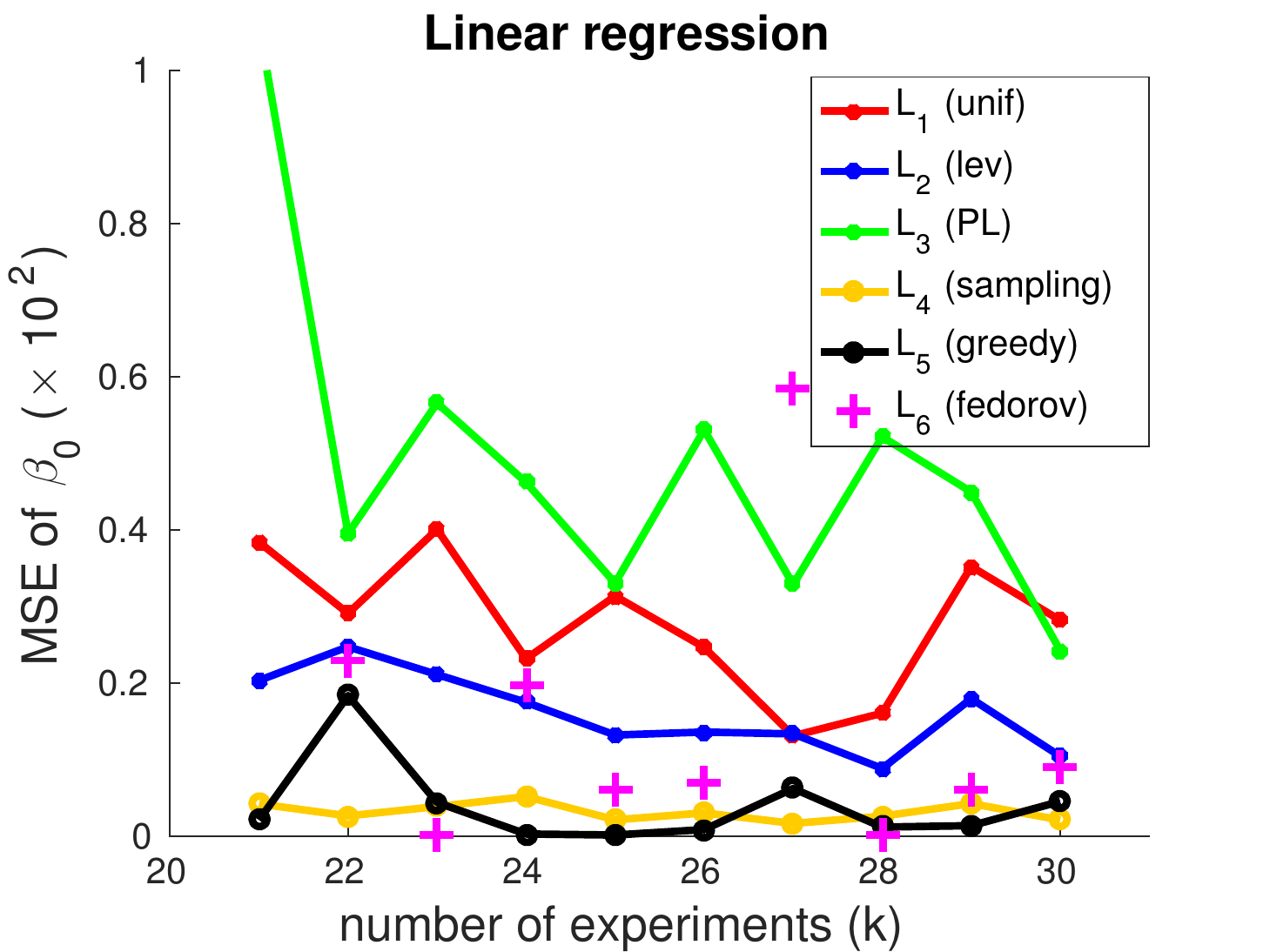}
\caption{\small Results on real material synthesis data ($n=133, p=9, 20<k\leq 30$).
Left: prediction error $\|y-X \hat\beta\|_2^2$, normalized by dividing by the prediction error $\|y-X \hat\beta^{\mathrm{ols}}\|_2^2$ of the OLS estimator on the full data set;
right: estimation error $\|\hat\beta-\hat\beta^{\mathrm{ols}}\|_2^2$.
All randomized algorithms ($L_1, L_2, L_3, L_4^*, L_6^\dagger$) are run for 50 times and the median performance is reported.}
\label{fig:material}
\end{figure}

\begin{table}[p]
\centering
\caption{\small Selected subsets of uniform sampling (\textsc{Uniform}), leverage score sampling (\textsc{Levscore}) and greedy method (\textsc{Greedy})
on the material synthesis dataset.}
\vskip 0.1in
\begin{tabular}{cccc||cccc||cccc}
\hline
\multicolumn{4}{c||}{\textsc{Uniform} ($L_1$)}& \multicolumn{4}{c||}{\textsc{Levscore} ($L_2$)}& \multicolumn{4}{c}{\textsc{Greedy} ($L_5^*$)}\\
$T$& $P$& $H$& $R$& $T$& $P$& $H$& $R$&$T$& $P$& $H$& $R$ \\
\hline
150& 30& 7.5& .07& 120& 30& 5& .07& 140& 15& 7.5& .80 \\
165& 30& 3& .73& 140& 15& 2.5& .80& 140& 15& 2.5& .80\\
140& 15& 7.6& .80& 170& 30& 7.5& .93& 140& 60& 2.5& .80\\
170& 30& 2.5& .07& 140& 60& 7.5& .80& 160& 15& 7.5& .80\\
160& 15& 7.5& .80& 150& 15& 2.5& .80& 160& 60& 7.5& .80\\
150& 15& 2.5& .80& 150& 30& 60& .67& 140& 60& 2.5& .80\\
160& 30& 2.5& .80& 160& 30& 30& .87&160& 60& 7.5& .80\\
150& 30& 6.1& .87& 170& 30& 2.5& .07& 170& 30& 7.5& .93\\
160& 30& 6.1& .87& 170& 30& 0& .93& 165& 30& 7.5& 0\\
150& 30& 4.1& .67& 140& 30& 5& .80& 150& 30& 5& 0\\
165& 30& 7.5& .13& 140& 30& 2.5& .87& 135& 30& 2.5& .5\\
140& 30& 15& .87& 140& 30& 6.2& .67& 150& 30& 4.1& .87\\
160& 60& 7.5& .80& 160& 30& 60& .87& 150& 30& 2.5& .87\\
150& 30& 5& .50& 140& 30& 60& .87& 135& 30& 3& .50\\
140& 30& 7.5& .80& 160& 30& 7.5& .80& 135& 30& 3& 0\\
150& 30& 5& .80& 160& 30& 2.5& .87& 140& 30& 30& .67\\
140& 30& 6.1& .67& 140& 30& 4.1& .67&140& 30& 30& .87\\
140& 30& 2.5& .67& 135& 30& 3& .50& 150& 30& 30& .87\\
150& 60& 5& .80& 150& 30& 30& .67&160& 30& 30& .87\\
160& 30& 5& .80& 165& 30& 5& .67&120& 30& 5& .07\\
135& 30& 3& .07& 160& 60& 7.5& .80& 160& 30& 60& .87\\
165& 30& 5& 0& 165& 30& 3& .67& 140& 30& 60& .87\\
120& 30& 7.5& .93& 135& 30& 3& .67& 120& 30& 7& .07\\
160& 30& 5& .80& 160& 15& 7.5& .80& 120& 30& 0& .93\\
165& 30& 0& .73& 150& 60& 7.5& .80& 170& 30& 0& .93\\
165& 30& 3& .67& 160& 30& 4.1& .67& 170& 30& 0& .07\\
150& 15& 7.5& .80& 165& 30& 0& .07& 150& 30& 0& .5\\
140& 30& 5& .80& 135& 30& 2.5& .73& 150& 30& 0& .6\\
150& 60& 2.5& .80& 150& 30& 15& .87& 165& 30& 0& .5\\
135& 30& 3& .67& 160& 60& 2.5& .80& 160& 30& 60& .07\\
\hline
\end{tabular}
\label{tab:material_subset}
\end{table}

We apply our proposed methods to an experimental design problem of low-temperature microwave-assisted crystallization of ceramic thin films \citep{reeja2012microwave,nakamura2017design}.
The microwave-assisted experiments were controlled by four experimental parameters: temperature $T$ ($120^\circ C$ to $170^\circ C$), hold time $H$ (0 to 60 minutes), ramp power $P$ ($0W$ to $60W$) and tri-ethyl-gallium (TEG) volume ratio $R$ (from 0 to 0.93).
A generalized quadratic regression model was employed in \citep{nakamura2017design} to estimate the coverage percentage of the crystallization $CP$:
\begin{equation}
CP = \beta_1 + \beta_2T  + \beta_3H + \beta_4 P + \beta_5 R + \beta_6 T^2 + \beta_7 H^2 + \beta_8 P^2 + \beta_9 R^2.
\end{equation}

As a proof of concept, we test our proposed algorithms and compare with baseline methods on a data set consisting of 133 experiments, with the coverage percentage fully collected and measured.
We consider selection of subsets of $k$ experiments, with $k$ ranging from 21 to 30, and report in Figure~\ref{fig:material} the mean-square error (MSE) of both the prediction error $\|y-X\hat\beta\|_2^2$, normalized by dividing by the prediction error $\|y-X \hat\beta_{\mathrm{ols}}\|_2^2$ of the OLS estimator on the full data set, 
on the full 133 experiments and the estimation error $\|\hat\beta-\hat\beta^{\mathrm{ols}}\|_2^2$. 
Figure \ref{fig:material} shows that our proposed methods consistently achieve the best performance and are more stable compared to uniform sampling and Fedorov's exchange methods, even though the 
linear model assumption may not hold. 

We also report the actual subset ($k=30$) design points selected by uniform sampling ($L_1$), leverage score sampling ($L_3$) and our greedy method ($L_5^*$) in Table \ref{tab:material_subset}.
Table \ref{tab:material_subset} shows that the greedy algorithm ($L_5^*$) picked very diverse experimental settings, including several settings of very lower temperature ($120^\circ C$) and very short hold time (0). In contrast, the experiments picked by both uniform sampling ($L_1$) and leverage score sampling ($L_3$) are less diverse.

\subsection{The CPU performance dataset}\label{subsec:cpu}

\begin{table}[t]
\centering
\caption{\small Comparison of the ``true'' model $\beta_0=[0.49;0.30;0.19;3.78]$ and the subsampled estimators $\hat\beta$ for different components of $\beta$. $k$ is the number of subsamples.
Random algorithms ($L_1$ through $L_4^*$) are repeated for 100 times and the median difference is reported.}
\vskip 0.1in
\begin{tabular}{cccccccccccc}
\hline 
& $\Delta\beta_1$& $\Delta\beta_2$& $\Delta\beta_3$& $\Delta\beta_4$& $\|\Delta\vct\beta\|_2$& & $\Delta\beta_1$& $\Delta\beta_2$& $\Delta\beta_3$& $\Delta\beta_4$& $\|\Delta\vct\beta\|_2$\\
\hline
$k=20$& \multicolumn{5}{}{}& $k=30$& \\
$L_1$& .043& .071& .074& .274& .295&  $L_1$& .038& .064& .070& .171& .199\\
$L_2$& .017& .026& .018& .260& .263& $L_2$& .011& .024& .014& .199& .201\\
$L_3$& .017& .043& .026& .243& .249& $L_3$& .015& .031& .018& .236& .239\\
$L_4^*$& .014& .039& .114& .215& .247& $L_4^*$& .014& .023& .030& .044& .060\\
$L_5^*$& .010& .024& .032& .209& .213& $L_5^*$& .022& .000& .050& .088& .104\\
$L_6^\dagger$& .010& .024& .032& .209& .213&$L_6^\dagger$& .022& .000& .050& .088& .104\\
\hline
$k=50$& \multicolumn{5}{}{}& $k=75$& \\
$L_1$& .025& .035& .037& .130& .142&$L_1$&  .016& .032& .024& .133& .139\\
$L_2$& .009& .027& .012& .130& .134& $L_2$& .006& .034& .010& .126& .131\\
$L_3$& .011& .035& .013& .166& .170& $L_3$& .009& .029& .009& .124& .128\\
$L_4^*$& .022& .009& .036& .097& .106& $L_4^*$& .012& .001& .015& .036& .041\\
$L_5^*$& .025& .003& .040& .009& .048&$L_5^*$&  .005& .013& .009& .016& .023\\
$L_6^\dagger$& .025& .003& .040& .009& .048& $L_6^\dagger$&.005& .012& .010& .011& .020\\
\hline
\end{tabular}
\label{tab_cpu}
\end{table}

CPU relative performance refers to the relative performance of a particular CPU model in terms of a base machine - the IBM 370/158 model.
Comprehensive benchmark tests are required to accurately measure the relative performance of a particular CPU,
which might be time-consuming or even impossible if the actual CPU has not been on the market yet.
\cite{cpu-relative-performance} considered a linear regression model to characterize the relationship between CPU relative performance
and several CPU capacity parameters such as main memory size, cache size, channel units and machine clock cycle time.
These parameters are fixed for any specific CPU model and could be known even before the manufacturing process.
The learned model can also be used to predict the relative performance of a new CPU model based on its model parameters, without running extensive benchmark tests.

Using domain knowledge, \cite{cpu-relative-performance} narrow down to three parameters of interest: \emph{average memory size} ($X_1$), \emph{cache size} ($X_2$)
and \emph{channel capacity} ($X_3$), all being explicitly computable functions from CPU model parameters.
An offset parameter is also involved in the linear regression model, making the number of variables $p=4$.
A total of $n=209$ CPU models are considered, with all of the model parameters and relative performance collected and no missing data.
Stepwise linear regression was applied to obtain the following linear model:
\begin{equation}
0.49X_1 + 0.30X_2 + 0.19X_3 + 3.78.
\label{eq:cpu-beta0}
\end{equation}

To use this data set as a benchmark for evaluating the experiment selection methods, we synthesize labels $Y$ using model Eq.~(\ref{eq:cpu-beta0}) with standard Gaussian noise 
and measure the difference between fit $\hat\beta$ and the true model $\beta_0=[0.49;0.30;0.19;3.78]$.
Table \ref{tab_cpu} shows that under various measurement budget ($k$) constraints, the proposed methods $L_4^*$ and $L_5^*$ consistently achieve
small estimation error, and is comparable to the Fedorov's exchange algorithm ($L_6^\dagger$).
As the data set is small ($209\times 4$), all algorithms are computationally efficient.

\begin{figure}[t]
\centering
\includegraphics[width=6cm]{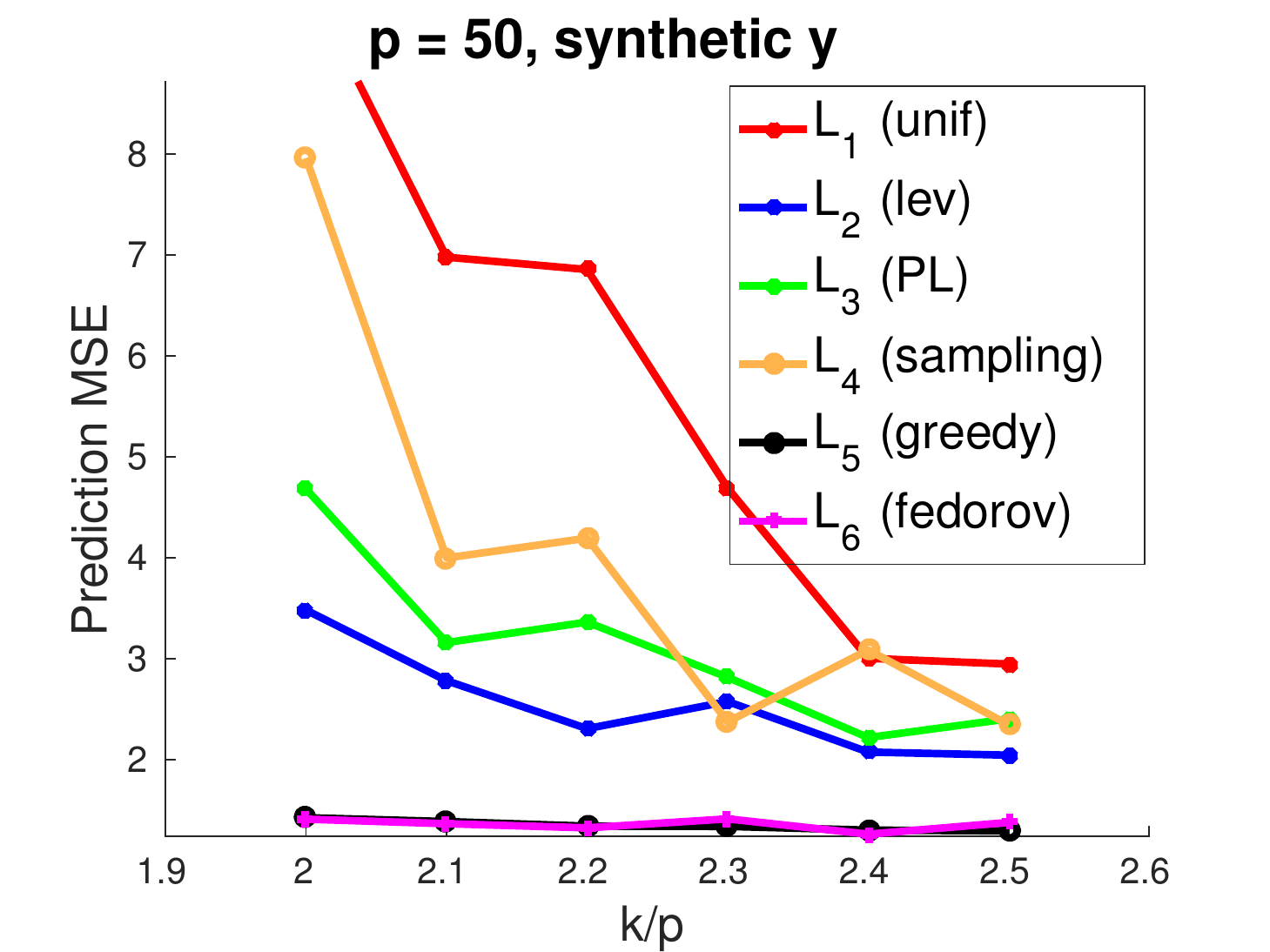}
\includegraphics[width=6cm]{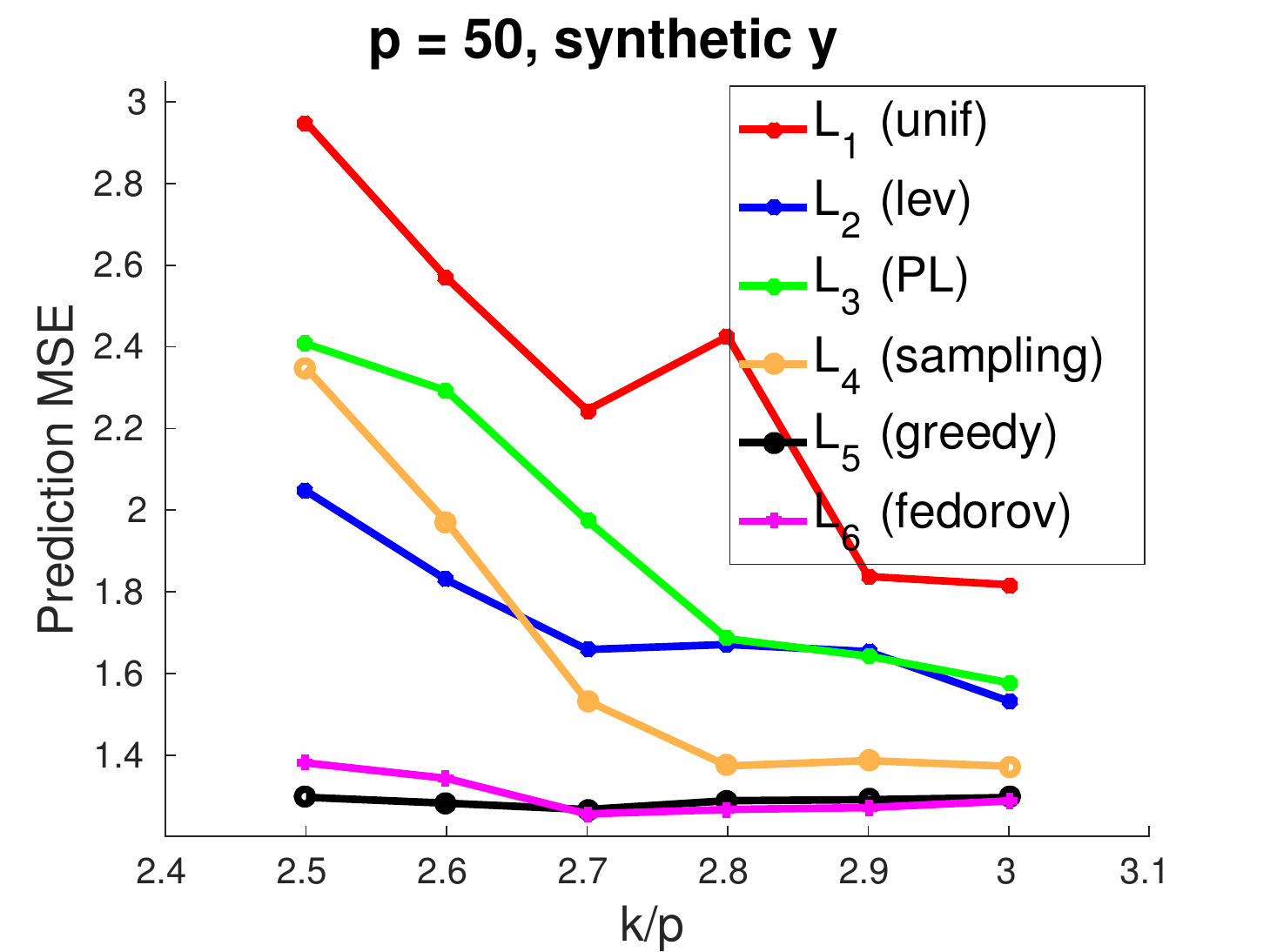}
\includegraphics[width=6cm]{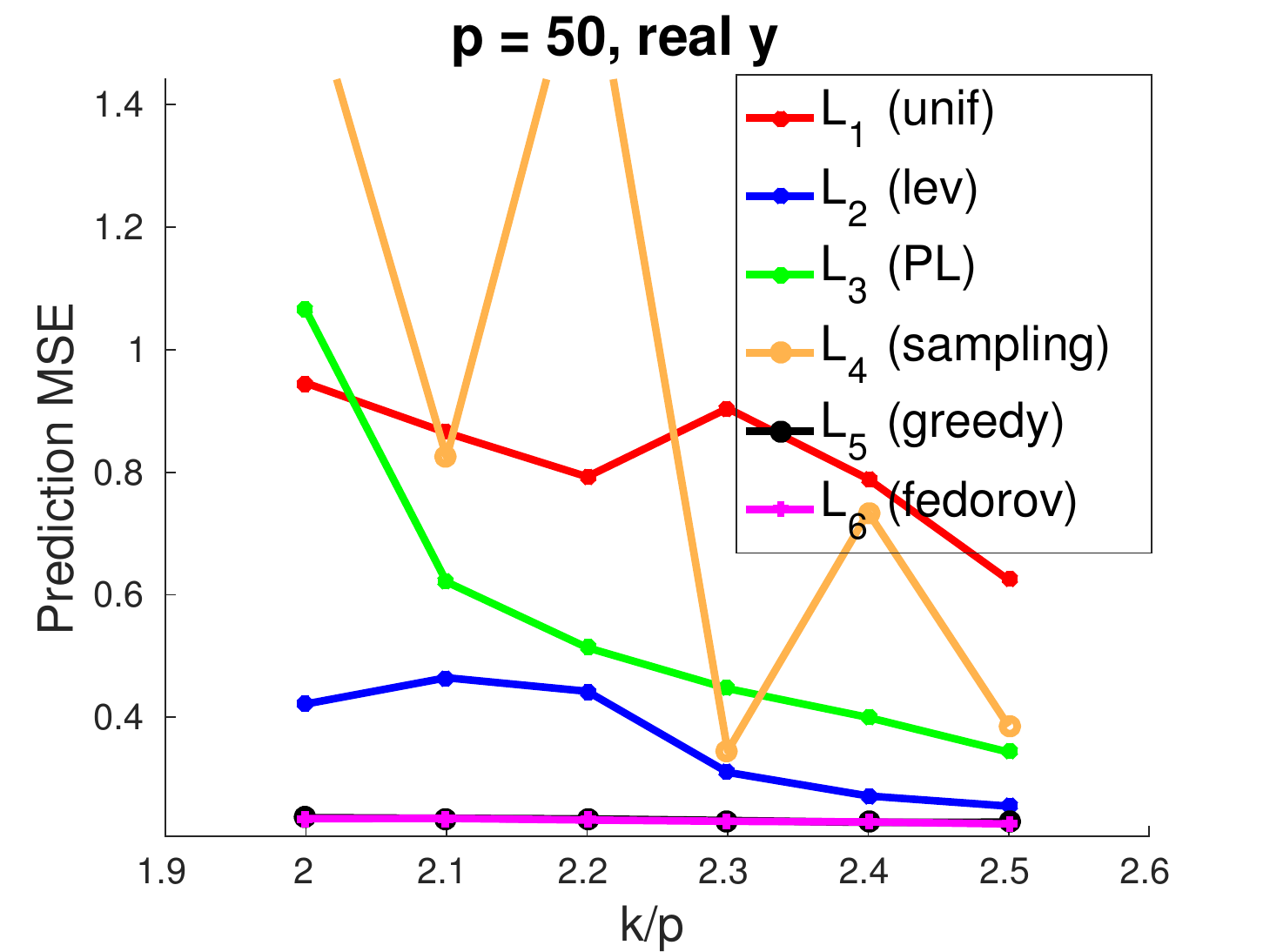}
\includegraphics[width=6cm]{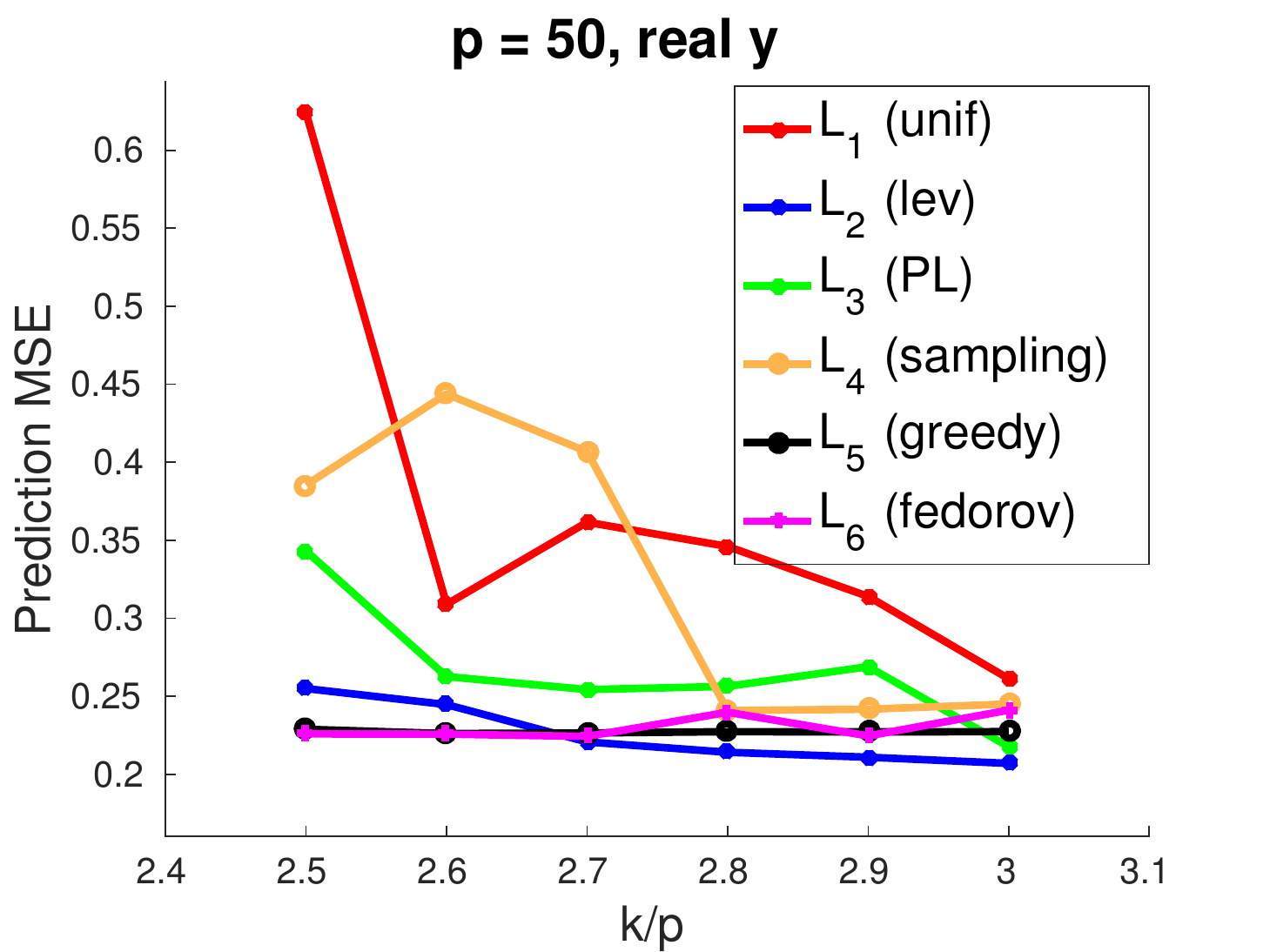}
\caption{\small Plots of the ratio of mean square prediction error $\frac{1}{n}\|V\hat\theta-y\|_2^2$ compared against the MSE of the full-sample OLS estimator $\frac{1}{n}\|V\hat\theta^{\mathrm{ols}}-y\|_2^2$ on the Minnesota wind dataset.
In the top panel response variables $y$ are synthesized as $y=V\hat\theta^{\ols}+\varepsilon$, where $\hat\theta^{\ols}$ is the full-sample OLS
estimator and $\varepsilon$ are i.i.d.~standard Normal random variables.
In the bottom panel the real wind speed response $y$ is used.
For randomized algorithms ($L_1$ to $L_4$) the experiments are repeated for 50 times and the median MSE ratio is reported.}
\label{fig:wind}
\end{figure}

\subsection{The Minnesota Wind Dataset}\label{subsec:wind}

The Minnesota wind dataset collects wind speed information across $n=2642$ locations in Minnesota, USA for a period of 24 months
(for the purpose of this experiment, we only use wind speed data for one month).
The 2642 locations are connected with 3304 bi-directional roads, which form an $n\times n$ sparse unweighted undirected graph $G$.
Let $L=\diag(d)-G$ be the $n\times n$ Laplacian of $G$,
where $d$ is a vector of node degrees
and let $V\in\mathbb R^{n\times p}$ be an orthonormal eigenbasis corresponding to the smallest $p$ eigenvalues of $L$.
As the wind speed signal $y\in\mathbb R^n$ is relatively smooth, it can be well-approximated as $y=\mat V\vct\theta+\vct\varepsilon$, where $\vct\theta\in\mathbb R^p$ corresponds to the coefficients of the graph Laplacian basis.
The speed signal $y$ has a fast decay on the Laplacian basis $V$;
in particular, OLS on the first $p=50$ basis accounts for over 99\% of the variation in $y$. (That is, $\|V\hat\theta^{\mathrm{ols}}-y\|_2^2 \leq 0.01\|y\|_2^2$, where
$\hat\theta^{\mathrm{ols}}=(V^\top V)^{-1}V^\top y$.)

In our experiments we compare the 6 algorithms ($L_1$ through $L_6$) when a very small portion of the full samples is selected.
The ratio of the mean-square prediction error $\frac{1}{n}\|V\hat\theta-y\|_2^2$ compared to the MLE of the full-sample OLS estimator $\frac{1}{n}\|V\hat\theta^{\mathrm{ols}}-y\|_2^2$ is reported.
Apart from the real speed data, we also report results under a ``semi-synthetic'' setting similar to Sec.~\ref{subsec:cpu},
where the full OLS estimate $\hat\theta^{\ols}$ is first computed and then $y$ is synthesized as $y=V\hat\theta^{\ols}+\varepsilon$ 
where $\varepsilon$ are i.i.d.~standard Normal random variables.

{
From Figure~\ref{fig:wind}, greedy methods ($L_5^*,L_6^\dagger$) achieve consistently the lowest MSE 
and are robust to subset size $k$.
However, the Fedorov's exchange algorithm is very slow and requires more than 10 times the running time than our proposed methods ($L_4^*,L_5^*$).
On the other hand, sampling based methods ($L_1$ through $L_4^*$) behave quite badly when subset size $k$ is small and close to the problem dimension $p$.
This is because when $k$ is close to $p$, even very small changes resulted from randomization could lead to highly singular designs and hence 
significantly increases the mean-square prediction error.
We also observe that for sufficiently large subset size (e.g., $k=3p$), the performance gap between all methods is smaller on real signal compared to the synthetic signals.
We conjecture that this is because the linear model $y=V\theta+\varepsilon$ only approximately holds in real data.
}

\section{Discussion}
We discuss potential improvements in the analysis presented in this paper.

\subsection{Sampling based method}
To fully understand the finite-sample behavior of the sampling method introduced in Secs.~\ref{subsec:sampling_expected} and \ref{subsec:sampling_deterministic},
it is instructive to relate it to the \emph{graph spectral sparsification} problem \citep{graph-sparsification} in theoretical computer science:
Given a directed weighted graph $G=(V,E,W)$, find a subset $\tilde E\subseteq E$ and new weights $\tilde W$
such that $\tilde G=(V,\tilde E,\tilde W)$ is a \emph{spectral sparsification} of $G$,
which means there exists $\epsilon\in(0,1)$ such that for any vector $z\in\mathbb R^{|V|}$:
$$
(1-\epsilon)\sum_{(u,v)\in E}{w_{uv}(z_u-z_v)^2} \leq \sum_{(u,v)\in\tilde E}{\tilde w_{uv}(z_u-z_v)^2} \leq (1+\epsilon)\sum_{u,v\in E}{w_{uv}(z_u-z_v)^2}.
$$
Define $B_G\in\mathbb R^{|E|\times |V|}$ to be the \emph{signed edge-vertex incidence matrix}, where each row of $B_G$ corresponds to an edge in $E$,
each column of $B_G$ corresponds to a vertex in $G$, and $[B_G]_{ij}=1$ if vertex $j$ is the head of edge $i$, $[B_G]_{ij}=-1$ if vertex $j$ is the tail of 
edge $i$, and $[B_G]_{ij}=0$ otherwise.
The spectral sparsification requirement can then be equivalently written as
$$
(1-\epsilon)z^\top (B_G^\top WB_G)z \leq z^\top (B_{\tilde G}^\top\tilde WB_{\tilde G}) z \leq (1+\epsilon)z^\top(B_G^\top WB_G)z.
$$

The similarity of graph sparsification and the experimental selection problem is clear:
$B_G\in\mathbb R^{n\times p}$ would be the known pool of design points and $W=\diag(\pi^*)$ is a diagonal matrix
with the optimal weights $\pi^*$ obtained by solving Eq.~(\ref{eq_c_opt}).
The objective is to seek a small subset of rows in $B_G$ (i.e., the sparsified edge set $\tilde E$) which is a spectral approximation of
the original $B_G^\top\diag(\pi^*)B_G$.
Approximation of the A-optimality criterion or any other eigen-related quantity immediately follows.
One difference is that in linear regression each row of $B_G$ is no longer a $\{\pm 1,0\}$ vector.
Also, the subsampled weight matrix $\tilde W$ needs to correspond to an unweighted graph for linear regression, i.e. diagonal entries of $\tilde W$ must be in $\{0,1\}$.
However, we consider this to be a minor difference as it does not interfere with the spectral properties of $B_G$.

The spectral sparsification problem where $\tilde W$ can be arbitrarily designed (i.e. not restricted to have $\{0,1\}$ diagonal entries) is completely solved \citep{graph-sparsification,batson2012twice},
where the size of the selected edge subset is allowed to be \emph{linear} to the number of vertices,
or in the terminology of our problem, $k\asymp p$.
Unfortunately, both methods require the power of arbitrary designing the weights in $\tilde W$, which is generally not available in experiment selection
problems (i.e., cannot set noise variance or signal strength arbitrarily for individual design points).
Recently, it was proved that when the original graph is unweighted ($W=I$), it is also possible to find unweighted linear-sized edge sparsifiers ($\tilde W\propto I$)
\citep{ks1,ks2,anderson2014efficient}.
This remarkable result leads to the solution of the long-standing Kardison-Singer problem.
However, the condition that the \emph{original} weights $W$ are uniform is not satisfied in the linear regression problem,
where the optimal solution $\pi^*$ may be far from uniform.
The experiment selection problem somehow falls in between, where an \emph{unweighted} sparsifier is desired for a \emph{weighted} graph.
This leads us to the following question:
\begin{question}
Given a \emph{weighted} graph $G=(V,E,W)$, under what conditions are there small edge subset $\tilde E\subseteq E$ with \emph{uniform} weights $\tilde W\propto I$
such that $\tilde G=(V,\tilde E,\tilde W)$ is a (one-sided) spectral approximation of $G$?
\end{question}

The answer to the above question, especially the smallest possible edge size $|\tilde E|$, would have immediate consequences on finite-sample conditions
of $k$ and $p$ in the experiment selection problem.

\subsection{Greedy method}
Corollary \ref{cor:s0} shows that the approximation quality of the greedy based method depends crucially upon $\|\pi^*\|_0$,
the support size of the optimal solution $\pi^*$.
In Lemma \ref{lem:support-bound} we formally established that $\|\pi^*\|_0\leq k+p(p+1)/2$ under mild conditions;
however, 
we conjecture the $p(p+1)/2$ term is loose and could be improved to $O(p)$ in general cases.

In Fig \ref{fig:check} we plot $\|\pi^*\|_0-k$ against number of variables $p$, where $p$ ranges from 10 to 100 and $k$ is set to $3p$.
The other simulation settings are kept unchanged.
We observe that in all settings $\|\pi^*\|_0-k$ scales linearly with $p$, suggesting that $\|\pi^*\|_0\leq k+O(p)$.
Furthermore, the slope of the scalings does not seem to depend on the conditioning of $X$, as shown in the right panel of Fig \ref{fig:check}
where the conditioning of $X$ is controlled by the spectral decay rate $\alpha$.
This is in contrast to the analysis of the sampling based method (Theorem \ref{thm:sampling_expected}), 
in which the finite sample bound depends crucially upon the conditioning of $X$ for the without replacement setting.

\begin{figure}[t]
\centering
\includegraphics[width=6cm]{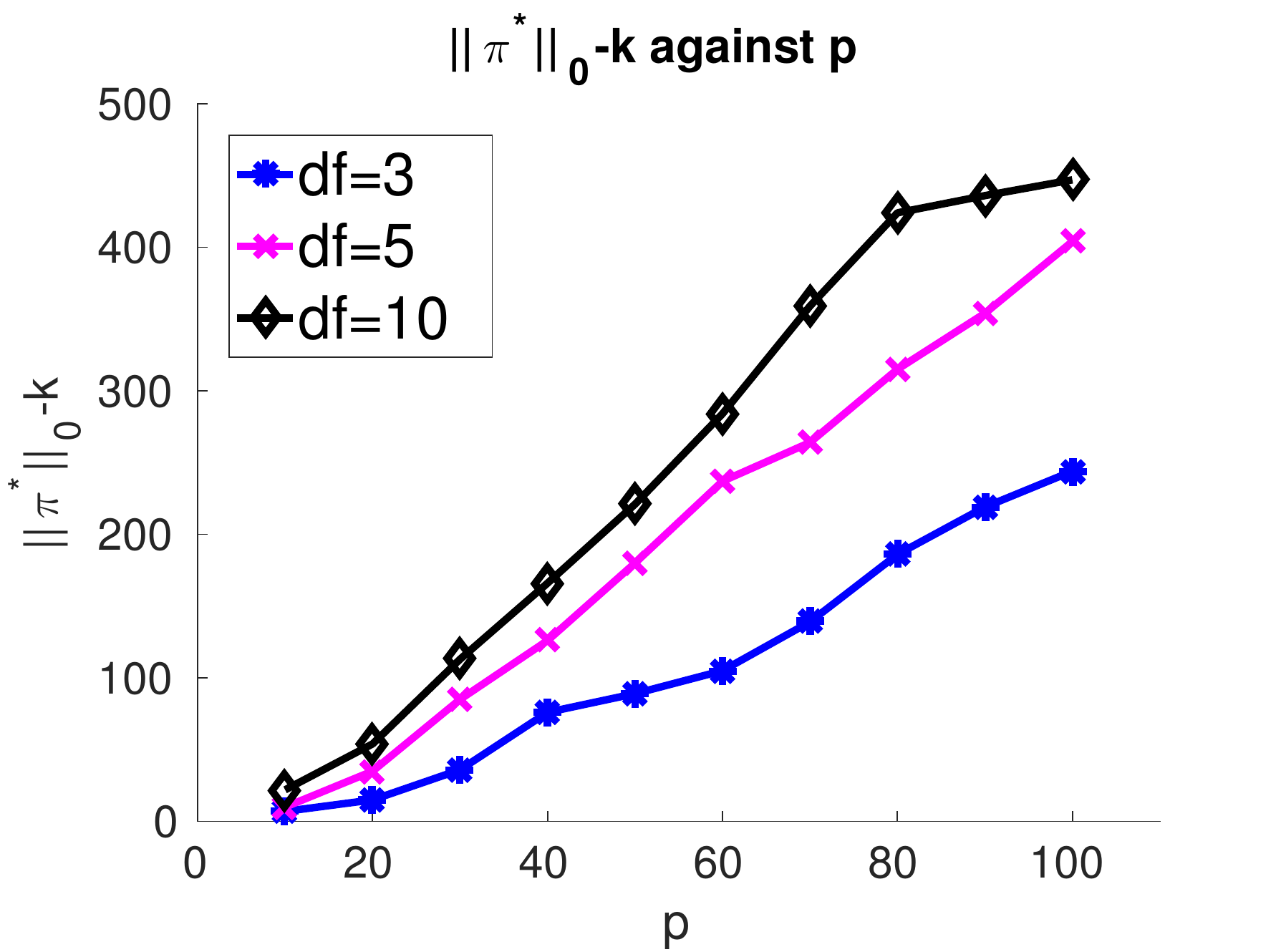}
\includegraphics[width=6cm]{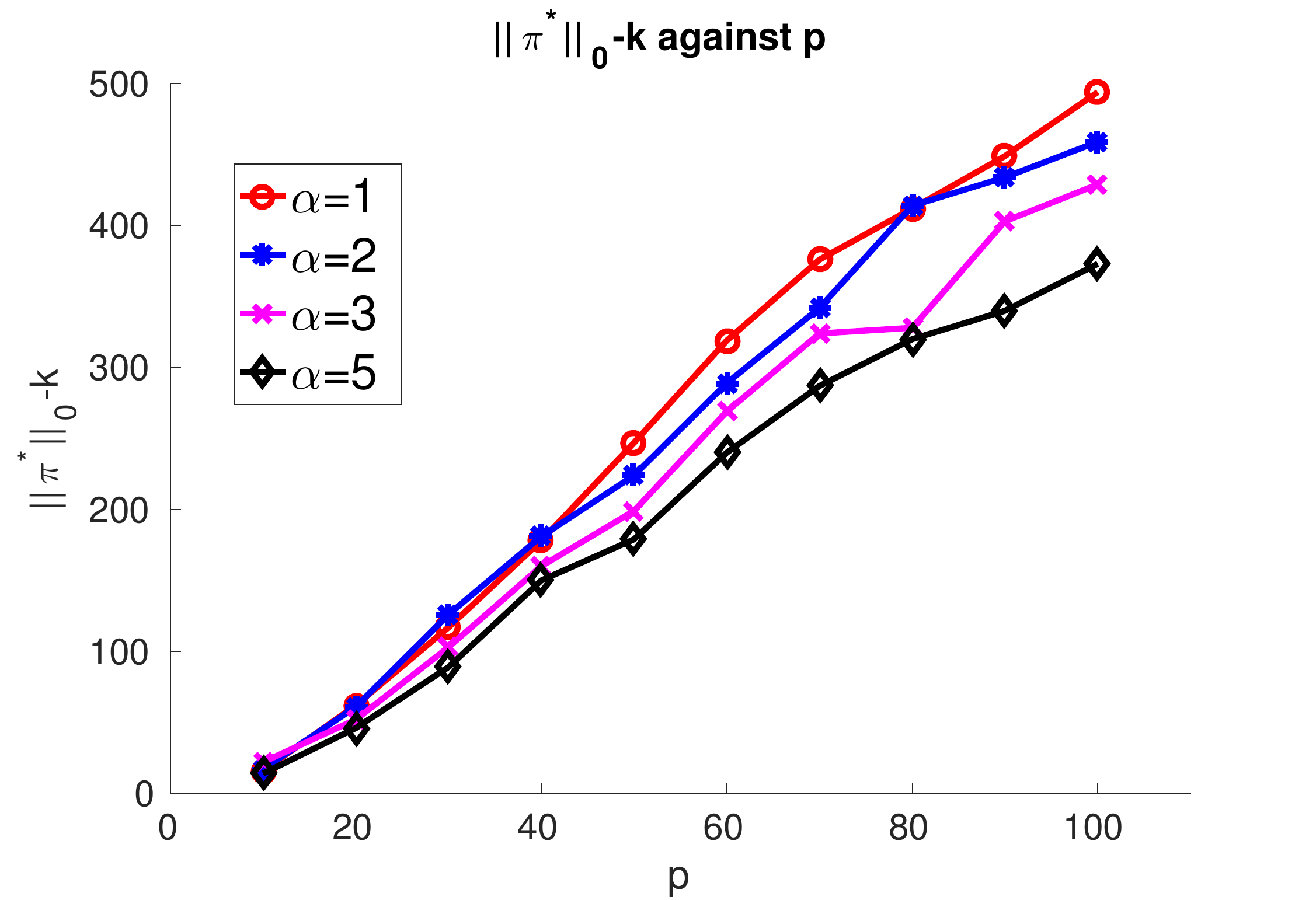}
\caption{\small Empirical verification of the rate of $\|\pi^*\|_0$. Left: isotropic design ($t$ distribution); Right: skewed design (transformed multivariate Gaussian) with spectral decay $\lambda_j\propto j^{-\alpha}$.}
\label{fig:check}
\end{figure}

{
\subsection{High-dimensional settings}

This paper focuses solely on the so-called \emph{large-sample, low-dimensional} regression setting, where the number of variables $p$ is assumed to be much smaller than both 
design pool size $n$ and subsampling size $k$.
It is an interesting open question to extend our results to the \emph{high-dimensional} setting, where $p$ is much larger than both $n$ and $k$,
and an assumption like sparsity of $\beta_0$ is made to make the problem feasible, meaning that very few components in $\beta_0$ are non-zero.
In particular, it is desirable to find a sub-sampling algorithm that attains the following minimax estimation rate over high-dimensional sparse models:
$$
\inf_{A\in\mathcal A_b(k)} \sup_{\|\beta_0\|_0\leq s} \mathbb E_{A,\beta_0}\left[\|\hat\beta-\beta_0\|_2^2\right].
$$

One major obstacle of designing subsampling algorithms for high-dimensional regression
is the difficulty of evaluating and optimizing the \emph{restricted eigenvalue} (or other similar criteria) of the subsampled covariance matrix.
Unlike the ordinary spectrum, the restricted eigenvalue of a matrix is NP-hard to compute \citep{fan2016regularity}
and heavily influences the statistical efficiency of a Lasso-type estimator \citep{bickel2009simultaneous}.
The question of optimizing such restricted eigenvalues could be even harder and remains largely open.
We also mention that \cite{constrained-adaptive-sensing} suggests a two-step approach where half the budget is used to identify the sparse locations from randomly sampled points and the remaining budget is used to generate a better estimate of regression coefficients using the same convex programming formulation proposed in our paper. 
Their paper provides some experimental support for this idea, however, no theoretical guarantees are established for the (sub)optimality of such a procedure. 
Analyzing such a two-step approach could be an interesting future direction. 
}

{

\subsection{Approximate linear models}

In cases when the linear model $y=X\beta_0+\varepsilon$ only approximately holds,  we describe here a method that takes into consideration
both bias and variance of OLS estimates on subsampled data in order to find good sub-samples.
Suppose $y=f_0(X)+\varepsilon$ for some unknown underlying function $f_0$ that might not be linear,
and let $\beta^*=(X^\top X)^{-1}X^\top f_0(X)$ be the optimal linear predictor on the full sample $X$.
Suppose $X_S=\Psi X$ is the sub-sampled data, where $\Psi\in\mathbb R^{|S|\times n}$ is the subsampling matrix
where each row of $\Psi$ is $e_i=(0,\cdots,0,1,0,\cdots,0)\in\mathbb R^n$, with $i\in[n]$ being the subsampled rows in $X$.
Let $\hat\beta=(X_S^\top X_S)^{-1}X_S^\top y_S$ be the OLS on the subsampled data.
The error of $\hat\beta$ can then be decomposed and upper bounded as
\begin{align}
\mathbb E\|\hat\beta-\beta^*\|_2^2
&= f_0(X)^\top\left[(X^\top\Psi^\top\Psi X)^{-1}X^\top\Psi^\top\Psi-(X^\top X)^{-1}X^{\top}\right]f_0(X) + \sigma^2\tr\left[(X^\top\Psi^\top\Psi X)^{-1}\right]\label{eq:bias}\\
&\leq \left\|(X^\top\Psi^\top\Psi X)^{-1}X^\top\Psi^\top\Psi-(X^\top X)^{-1}X^{\top}\right\|_{\mathrm{op}}\cdot \|f_0(X)\|_2^2 + \sigma^2\tr\left[(X^\top\Psi^\top\Psi X)^{-1}\right].\label{eq:relax-bias}
\end{align}
In the high noise setting, the first term can be ignored and the solution is close to the one considered in this paper. 
In the low-noise setting, the second term can be ignored and a relaxation similar to Eq.~(\ref{eq_c_opt}) can be derived; however a linear approximation may be undesirable in this case. 
In general, when $\|f_0(X)\|_2^2\approx\lambda$ is known or can be estimated, the following continuous optimization problem serves as an approximate objective of subsampled linear regression 
with approximate linear models:
\begin{eqnarray*}
\min_{\pi\in\mathbb R^n}& &
\lambda\left\|\left(\sum_{i=1}^n{\pi_ix_ix_i^\top}\right)^{-1}X^\top\diag(\pi) - (X^\top X)^{-1}X^\top\right\|_{\mathrm{op}} + \sigma^2\tr\left[\left(\sum_{i=1}^n{\pi_ix_ix_i^\top}\right)^{-1}\right]\\
s.t.& & \sum_{i=1}^n{\pi_i\leq k},\;\;\;0\leq \pi_i\leq 1.
\end{eqnarray*}

Unfortunately, the relaxation in Eq.~(\ref{eq:relax-bias}) may be loose and hence the optimization problem above fails to serve as a good objective for the 
optimal subsampling problem with approximate linear models.
Exact characterization of the bias term in Eq.~(\ref{eq:bias}) without strong assumptions on $f_0$ is an interesting open question.
}

\section{Proofs}\label{sec:proof}

\subsection{Proof of facts in Sec.~\ref{subsec:convex}}\label{subsec:proof-facts}


\begin{proof}[Proof of Fact \ref{fact:2}]
Let $A=X^\top\diag(\pi) X$ and $B=X^\top\diag(\pi')X$. By definition, 
$
B = A + \Delta
$
where $\Delta=X^\top\diag(\pi'-\pi)X$ is positive semi-definite.
Subsequently, $\sigma_{\ell}(B)\geq\sigma_{\ell}(A)$ for all $\ell=1,\cdots,p$.
We then have that
$$
f(\pi';X) = \tr(B^{-1}) = \sum_{\ell=1}^p{\sigma_{\ell}(B)^{-1}} \leq \sum_{\ell=1}^p{\sigma_\ell(A)^{-1}} = \tr(A^{-1}) = f(\pi;X).
$$
Note also that if $\pi'\neq \pi$ then $\Delta\neq 0$ and hence there exists at least one $\ell$ with $\sigma_{\ell}(B)<\sigma_{\ell}(A)$.
Therefore, the equality holds if and only if $\pi=\pi'$.
\end{proof}

\begin{proof}[Proof of Fact \ref{fact:3}]
Suppose $\|\pi^*\|_1 < k$.
Then there exists some coordinate $j\in\{1,\cdots,n\}$ such that $\pi_j^* < 1$.
Define $\pi'$ as $\pi'_i=\pi_i^*$ for $i\neq j$ and $\pi'_j = \min\{1,\pi_j^*+k-\|\pi^*\|_1\}$.
Then $\pi'$ is also a feasible solution.
On the other hand, By Fact \ref{fact:2} we have that $f(\pi';X) < f(\pi^*;X)$, contradicting the optimality of $\pi^*$.
Therefore, $\|\pi^*\|_1=k$.
\end{proof}

\subsection{Proof of Theorem \ref{thm:minimax}}

We only prove Theorem \ref{thm:minimax} for the with replacement setting ($b=1$). The proof for the without replacement setting
is almost identical.

Define $\widetilde{\mathcal A}_1(k)$ as the class of \emph{deterministic} algorithms that proceed as follows:
\begin{enumerate}
\item The algorithm deterministically outputs pairs $\{(w_i,\vct x_i)\}_{i=1}^M$, where $\vct x_i$ is one of the rows in $\mat X$
and $\{w_i\}_{i=1}^n$ satisfies $w_i\geq 0$, $\sum_{i=1}^M{w_i} \leq k$.
Here $M$ is an arbitrary finite integer.
\item The algorithm observes $\{y_i\}_{i=1}^M$ with $y_i=\sqrt{w_i}\vct x_i^\top\vct\beta + \varepsilon_i$.
Here $\vct\beta$ is a fixed but unknown regression model and $\varepsilon_i\overset{i.i.d.}\sim\nml(0,\sigma^2)$ is the noise.
\item The algorithm outputs $\hat{\vct\beta}$ as an estimation of $\vct\beta$, based on the observations $\{y_i\}_{i=1}^M$.
\end{enumerate}
Because all algorithms in $\widetilde{\mathcal A}_1(k)$ are deterministic and the design matrix still has full column rank,
 the optimal estimator of $\hat{\vct\beta}$ \emph{given} $\{(w_i,\vct x_i)\}_{i=1}^M$
is the OLS estimator. 
For a specific data set weighted through $\sqrt{w_1},\cdots,\sqrt{w_M}$, 
the minimax estimation error (which is achieved by OLS) is given by
$$
\inf_{\hat{\vct\beta}}\sup_{\vct\beta}\mathbb E\left[\|\hat{\vct\beta}-\vct\beta\|_2^2\right]
= \sup_{\vct\beta}\mathbb E\left[\|\hat{\vct\beta}^{\ols}-\vct\beta\|_2^2\right]
= \sigma^2\tr\left[(\widetilde{\mat X}^\top\widetilde{\mat X})^{-1}\right]
= \sigma^2\tr\left[\left(\sum_{i=1}^n{\tilde w_i\vct x_i\vct x_i^\top}\right)^{-1}\right],
$$
where $\tilde w_i$ is the aggregated weight of data point $\vct x_i$ in all the $M$ weighted pairs.
Subsequently,
$$
\inf_{\widetilde A\in\widetilde{\mathcal A}_1(k)}\sup_{\vct\beta}\mathbb E\left[\|\hat{\vct\beta}_{\widetilde A}-\vct\beta\|_2^2\right]
= \inf_{w_1+\cdots+w_M\leq k}\sigma^2\tr\left[\left(\sum_{i=1}^M{w_i\vct x_i\vct x_i^\top}\right)^{-1}\right]
= \sigma^2f_1^*(k;X).
$$
It remains to prove that
$$
\inf_{\widetilde A\in\widetilde{\mathcal A}_1(k)}\sup_{\vct\beta}\mathbb E\left[\|\hat{\vct\beta}_{\widetilde A}-\vct\beta\|_2^2\right]
\leq \inf_{A\in{\mathcal A}_1(k)}\sup_{\vct\beta}\mathbb E\left[\|\hat{\vct\beta}_A-\vct\beta\|_2^2\right].
$$
We prove this inequality by showing that for every (possibly random) algorithm $A\in\mathcal A_1(k)$, there exists $\widetilde A\in\widetilde{\mathcal A}_1(k)$ such that
$\sup_{\vct\beta}\mathbb E[\|\hat\beta_A-\vct\beta\|_2^2] \geq \sup_{\vct\beta}\mathbb E[\|\hat\beta_{\widetilde A}-\vct\beta\|_2^2]$.
To see this, we construct $\widetilde A$ based on $A$ as follows:
\begin{enumerate}
\item For every $k$-subset (duplicates allowed) 
{of all possible outputs of $A$ (which by definition are all subsets of $\mat X$)} 
and its corresponding weight vector $\vct w$, 
add $(w_i',\tilde{\vct x_i})$ to the design set of $\widetilde{ A}$,
where $w_i' =  w_i\Pr_A(\widetilde{\mat X})$.
\item The algorithm $\widetilde A$ observes all responses $\{y_i\}$ for $\{(w_i', \tilde{\vct x_i})\}$.
\item $\widetilde A$ outputs the \emph{expected} estimation of $A$;
that is, $\widetilde A(\mat X,\vct y) = \mathbb E_{\tilde{\mat X}}[\mathbb E_{\vct y}[\hat\beta_A|\widetilde{\mat X}]]$.
{Note that by definition of the estimator class ${\mathcal A}_1(k)$, all estimators $\hat{\vct\beta}_A$ \emph{conditioned on subsampled data points $\widetilde{\mat X}$}
are deterministic.
}
\end{enumerate}
We claim that $\widetilde A\in\widetilde{\mathcal A}_1(k)$ because
$$
\sum_i{w_i'} = \sum_{\widetilde{\mat X}}{\Pr(\widetilde{\mat X})\sum_{i=1}^k{w_i}} \leq k.
$$
Furthermore, by Jensen's inequality we have
$$
\mathbb E_{\widetilde{\mat X},\vct y}\left[\| \hat\beta_A-\vct\beta\|_2^2\right]
\geq \mathbb E_{\vct y}\left[\|\mathbb E_{\widetilde{\mat X}}[\hat\beta_A] - \vct\beta\|_2^2\right]
= \mathbb E_{\vct y}\left[\|\hat\beta_{\widetilde A}-\vct\beta\|_2^2\right].
$$
Taking supreme over $\vct\beta$ we complete the proof.

\subsection{Proof of Lemma \ref{lem:spectral_expected}}

\paragraph{With replacement setting}
Define $\Phi=\diag(\pi^*)$ and $\Pi=\Phi^{1/2}X\Sigma_*^{-1}X^\top\Phi^{1/2}\in\mathbb R^{n\times n}$.
The following proposition lists properties of $\Pi$:
\begin{proposition}[Properties of projection matrix]
The following properties for $\Pi$ hold:
\begin{enumerate}
\item $\mat\Pi$ is a projection matrix. That is, $\mat\Pi^2 = \mat\Pi$.
\item $\Range(\mat\Pi) = \Range(\mat\Phi^{1/2}\mat X)$.
\item The eigenvalues of $\mat\Pi$ are 1 with multiplicity $p$ and 0 with multiplicity $n-p$.
\item $\mat\Pi_{ii} = \|\mat\Pi_{i,\cdot}\|_2^2 = \pi_i^*\vct x_i^\top\Sigma_*^{-1}\vct x_i$.
\end{enumerate}
\label{prop_projection}
\end{proposition}
\begin{proof}
Proof of 1: By definition, $\mat\Sigma_*=\mat X^\top\mat\Phi\mat X$ and subsequently
\begin{eqnarray*}
\mat\Pi^2 &=& \mat\Phi^{1/2}\mat X\mat\Sigma_*^{-1}\mat X^\top\mat\Phi^{1/2}\mat\Phi^{1/2}\mat X\mat\Sigma_*^{-1}\mat X^\top\mat\Phi^{1/2}\\
&=& \mat\Phi^{1/2}\mat X(\mat X^\top\mat\Phi\mat X)^{-1}\mat X^\top\mat\Phi\mat X(\mat X^\top\mat\Phi\mat X)^{-1}\mat\Phi^{1/2}\\
&=& \mat\Phi^{1/2}\mat X(\mat X^\top\mat\Phi\mat X)^{-1}\mat X^\top\mat\Phi^{1/2} = \mat\Pi.
\end{eqnarray*}

Proof of 2: First note that $\Range(\mat\Pi)=\Range(\mat\Phi^{1/2}\mat X\mat\Sigma_*^{-1}\mat X^\top\mat\Phi^{1/2}) \subseteq \Range(\mat\Phi^{1/2}\mat X)$.
For the other direction, take arbirary $\vct u\in\Range(\mat\Phi^{1/2}\mat X)$ and express $\vct u$ as
$\vct u = \mat\Phi^{1/2}\mat X\vct v$ for some $\vct v\in\mathbb R^p$.
We then have
\begin{eqnarray*}
\mat\Pi\vct u &=& \mat\Phi^{1/2}\mat X\mat\Sigma_*^{-1}\mat X^\top\mat\Phi^{1/2}\vct u\\
&=& \mat\Phi^{1/2}\mat X(\mat X^\top\mat\Phi\mat X)^{-1}\mat X^\top\mat\Phi^{1/2}\mat\Phi^{1/2}\mat X\vct v\\
&=& \mat\Phi^{1/2}\mat X\vct v = \vct u
\end{eqnarray*}
and hence $\vct u\in\Range(\mat\Pi)$.

Proof of 3: Because $\mat\Sigma_*=\mat X^\top\mat\Phi\mat X$ is invertible, 
the $n\times p$ matrix $\mat\Phi^{1/2}\mat X$ must have full column rank and hence 
$\Kernel(\mat\Phi^{1/2}\mat X) = \{\vct 0\}$.
Consequently, $\dim(\Range(\mat\Pi)) = \dim(\Range(\mat\Phi^{1/2}\mat X)) = p - \dim(\Kernel(\mat\Phi^{1/2}\mat X)) = p$.
On the other hand, the eigenvalues of $\mat\Pi$ must be either 0 or 1 because $\mat\Pi$ is a projection matrix.
So the eigenvalues of $\mat\Pi$ are 1 with multiplicity $p$ and 0 with multiplicity $n-p$.

Proof of 4: By definition,
$$
\mat\Pi_{ii} = \sqrt{\pi_i^*}\vct x_i^\top\mat\Sigma_*^{-1}\vct x_i\sqrt{\pi_i^*} = \pi_i^*\vct x_i^\top\mat\Sigma_*^{-1}\vct x_i.
$$
In addition, $\mat\Pi$ is a symmetric projection matrix. Therefore,
$$
\mat\Pi_{ii} = [\mat\Pi^2]_{ii} = \mat\Pi_{i,\cdot}^\top\mat\Pi_{i,\cdot} = \|\mat\Pi_{i,\cdot}\|_2^2.
$$
\end{proof}

The following lemma shows that a spectral norm bound over deviation of the projection matrix implies spectral approximation of the underlying (weighted) covariance matrix.
\begin{lemma}[\cite{graph-sparsification}, Lemma 4]
Let $\mat\Pi = \mat\Phi^{1/2}\mat X\mat\Sigma_*^{-1}\mat X^\top\mat\Phi^{1/2}$ and $\mat W$ be an $n\times n$ non-negative diagonal matrix.
If $\|\mat\Pi\mat W\mat\Pi - \mat\Pi\|_2 \leq \epsilon$ for some $\epsilon\in(0,1/2)$ then
$$
(1-\epsilon)\vct u^\top\mat\Sigma_*\vct u \leq \vct u^\top\tilde{\mat\Sigma}_*\vct u \leq (1+\epsilon)\vct u^\top\mat\Sigma_*\vct u, \;\;\;\; \forall\vct u\in\mathbb R^p,
$$
where $\mat\Sigma_* = \mat X^\top\mat\Phi\mat X$ and $\tilde{\mat\Sigma}_* = \mat X^\top\mat W^{1/2}\mat\Phi\mat W^{1/2}\mat X$.
\label{lem_proj_spectral_approx}
\end{lemma}

We next proceed to find an appropriate diagonal matrix $W$ and validate Lemma \ref{lem_proj_spectral_approx}.
Define $w_j^*=\pi_j^*/(kp_j^{(2)})$ and $\hat\Sigma_{\hat W}=\sum_{t=1}^k{w_{i_t}^* x_{i_t}x_{i_t}^\top}$.
It is obvious that $\hat\Sigma_{\hat W}\preceq\hat\Sigma_{\hat S}$ because $w_i = \lceil w_i^*\rceil \geq w_i^*$.
It then suffices to lower bound the spectrum of $\hat\Sigma_{\hat W}$ by the spectrum of $\hat\Sigma_*$.
Define random diagonal matrix $W^{(1)}$ as ($\mathbb I[\cdot]$ is the indicator function)
$$
W^{(1)}_{jj} = \frac{\sum_{t=1}^k{w_{i_t}^*\mathbb I[i_t=j]}}{\pi_j^*}, \;\;\;\;\;j=1,\cdots,n.
$$
Then by definition, $\tilde\Sigma_* = \mat X^\top\mat W^{1/2}\mat\Phi\mat W^{1/2}\mat X = \sum_{t=1}^k{w_{i_t}^*x_{i_t}x_{i_t}^\top} =  \hat\Sigma_{\hat R}$.
The following lemma bounds the perturbation $\|\Pi W\Pi-\Pi\|_2$ for this particular choice of $W$.
\begin{lemma}
For any $\epsilon>0$, 
$$
\Pr\left[\|\Pi W^{(1)}\Pi-\Pi\|_2>\epsilon\right] \leq 2\exp\left\{-C\cdot\frac{k\epsilon^2}{p\log k}\right\},
$$
where $C>0$ is an absolute constant.
\label{lem:concentration_withrep}
\end{lemma}
\begin{proof}
Define $n$-dimensional random vector $v$ as
\footnote{For those $j$ with $\pi_j^*=0$, we have by definition that $p_j^{(2)}=0$.}
$$
\Pr\left[v=\sqrt{\frac{kw_j^*}{\pi_j^*}}\Pi_{j\cdot}\right] = p_j^{(2)}, \;\;\;\;\;\;
j=1,\cdots,n.
$$
Let $v_1,\cdots,v_k$ be i.i.d.~copies of $v$ and define $A_t=v_tv_t^\top$.
By definition, $\Pi W^{(1)}\Pi$ is equally distributed with $\frac{1}{k}\sum_{t=1}^k{A_t}$.
In addition, 
$$
\mathbb EA_t = \sum_{j=1}^n{\frac{kw_j^*}{\pi_j^*}p_j^{(2)}\Pi_{j\cdot}\Pi_{j\cdot}^\top} = \Pi^2=\Pi,
$$
which satisfies $\|\mathbb EA_t\|_2=1$, and
$$
\|A_t\|_2 = \|v_t\|_2^2 \leq \sum_{1\leq j\leq n}\frac{kw_j^*}{\pi_j^*}\|\Pi_{j\cdot}\|_2^2 \leq \sup_{1\leq j\leq n}\frac{p}{x_j^\top\Sigma_*^{-1}x_j}\cdot x_j^\top\Sigma_*^{-1}x_j = p.
$$
Applying Lemma \ref{lem:matrix_rudelson} we have that
$$
\Pr\left[\|\Pi W^{(1)}\Pi-\Pi\|_2>\epsilon\right] \leq 2\exp\left\{-C\cdot\frac{k\epsilon^2}{p\log k}\right\}.
$$
\end{proof}

With Lemma \ref{lem:concentration_withrep}, we know that $\|\Pi W\Pi-\Pi\|_2\leq\epsilon$ holds with probability at least 0.9
if $\frac{p\log k}{k} = O(\epsilon^2)$.
Eq.~(\ref{eq:spectral_lower_bound}) then holds with high probability by Lemma \ref{lem_proj_spectral_approx}.

\paragraph{Without replacement setting}
Define independently distributed random matrices $A_1,\cdots,A_n$ as
$$
A_j = (w_j-\pi_j^*)x_jx_j^\top,\;\;\;\;\;\;j=1,\cdots,n.
$$
Note that $w_j$ is a random Bernoulli variable with $\Pr[w_j=1]=kp_j^{(2)}=\pi_j^*$.
Therefore, $\mathbb EA_j=0$.
In addition, 
$$
\sup_{1\leq j\leq n}\|A_j\|_2 \leq \sup_{1\leq j\leq n}\|x_j\|_2^2 \leq \|X\|_{\infty}^2 \;\;\;\;a.s.
$$
and
$$
\left\|\sum_{j=1}^n{\mathbb EA_j^2}\right\|_2
= \left\|\sum_{j=1}^n{\pi_j^*(1-\pi_j^*)\|x_j\|_2^2x_jx_j^\top}\right\|_2
\leq \|X\|_{\infty}^2\|\Sigma_*\|_2.
$$
Noting that $\sum_{j=1}^n{A_j}=X_{\hat S}^\top X_{\hat S}-X^\top\diag(\pi^*) X = \hat\Sigma_{\hat S}-\Sigma_*$
and invoking Lemma \ref{lem:bernstein} with $t=\epsilon \lambda_{\min}(\Sigma_*)$ we have that
$$
\Pr\left[\|\hat\Sigma_{\hat S}-\Sigma_*\|_2 > \epsilon\lambda_{\min}(\Sigma_*)\right] \leq 2p\exp\left\{-\frac{\epsilon^2\lambda_{\min}(\Sigma_*)^2}{3\|\Sigma_*\|_2\|X\|_{\infty}^2 + 2\|X\|_\infty^2\cdot \epsilon\lambda_{\min}(\Sigma_*)}\right\}.
$$
Equating the right-hand side with $O(1)$ we have that
$$
\|\Sigma_*^{-1}\|_2\kappa(\Sigma_*)\|X\|_{\infty}^2\log p = O(\epsilon^2).
$$
Finally, by Weyl's theorem we have that
$$
\big|\lambda_\ell(\hat\Sigma_{\hat S})-\lambda_\ell(\Sigma_*)\big| \leq \|\hat\Sigma_{\hat S}-\Sigma_*\|_2 \leq \epsilon \lambda_{\min}(\Sigma_*) \leq \epsilon\lambda_{\ell}(\Sigma_*),
$$
and hence the proof of Lemma \ref{lem:spectral_expected}.

\subsection{Proof of Lemma \ref{lem:spectral_deterministic}}

\paragraph{With replacement setting}
Define $\tilde w_j^*=k/T\cdot w_j^* = \pi_j^*/(Tp_j^{(1)})$ 
and let $\hat\Sigma_T = \sum_{t=1}^T{\tilde w_{i_t}^*x_{i_t}x_{i_t}^\top}$.
Because $\hat\Sigma_T = \frac{k}{T}\hat\Sigma_{\hat\Sigma_W} \preceq \hat\Sigma_{\hat S}$, we have that
$\hat\Sigma_{\hat S}\succeq \frac{T}{k}\hat\Sigma_T$ and hence $z^\top\hat\Sigma_{\hat S}z \geq \frac{T}{k}z^\top\hat\Sigma_T z$ for all $z\in\mathbb R^p$.
Therefore, to lower bound the spectrum of $\hat\Sigma_{\hat S}$ it suffices to lower bound the spectrum of $\hat\Sigma_T$.

Define diagonal matrix $W^{(2)}$ as
$$
W^{(2)}_{jj} = \frac{\sum_{t=1}^T{\tilde w_{i_t}^*\mathbb I[i_t=j]}}{\pi_j^*}, \;\;\;\;\;\;j=1,\cdots,n.
$$
We have that $\tilde\Sigma_*=\hat\Sigma_T$ for this particular choice of $W$.
Following the same analysis in the proof of Lemma \ref{lem:concentration_withrep}, we have that
for every $t>0$
$$
\Pr\left[\|\Pi W^{(2)}\Pi - \Pi\|_2 > t\right] \leq 2\exp\left\{-C\cdot \frac{T\epsilon^2}{p\log T}\right\}.
$$
Set $t=O(1)$ and equate the right-hand side of the above inequality with $O(1)$.
We then have
$$
\Pr\left[\hat\Sigma_T\succeq \Omega(1)\cdot \hat\Sigma_*\right] = \Omega(1)\;\;\;\;\;\;\text{if}\;\;p\log T/T=O(1).
$$
Subsequently, under the condition that $p\log T/T=O(1)$, with probability at least 0.9 it holds that
$$
\hat\Sigma_{\hat S} \succeq \Omega(T/k)\cdot \Sigma_*,
$$
which completes the proof of Lemma \ref{lem:spectral_deterministic} for the without replacement setting.

\paragraph{Without replacement setting}
Define $\hat\Sigma_{\hat R}=X_{R_T}^\top X_{R_T} = \sum_{t=1}^T{\pi_{i_t}x_{i_t}x_{i_t}^\top}$.
Conditioned on $R_T$, the subset $\hat S=S_T$ is selected using the same procedure of the soft-constraint algorithm Figure \ref{alg:subsampling_expected}
on $X_{R_T}=\{x_i\}_{i\in R_T}$.
Subsequently, following analysis in the proof of Lemma \ref{lem:spectral_expected} we have
$$
\Pr\left[\|\hat\Sigma_{\hat S}-\hat\Sigma_T\|_2>t\big| R_T\right]\leq 2p\cdot \exp\left\{-\frac{t^2}{3\|\hat\Sigma_T\|_2\|X\|_{\infty}^2 + 2\|X\|_{\infty}^2t}\right\}.
$$
Setting $t=O(1)\cdot \lambda_{\min}(\Sigma_T)$ we have that, if $\|\hat\Sigma_T^{-1}\|_2\kappa(\hat\Sigma_T)\|X\|_{\infty}^2\log p=O(1)$, then
with probability at least 0.95 conditioned on $\hat\Sigma_T$ 
\begin{equation}
\Omega(1)\cdot \hat\Sigma_T \preceq \hat\Sigma_{\hat S} \preceq O(1)\cdot \hat\Sigma_{T}.
\label{eq:tn_chain1}
\end{equation}

It remains to establish spectral similarity between $\hat\Sigma_T$ and $\frac{T}{n}\Sigma_*$, a scaled version of $\Sigma_*$.
Define deterministic matrices $A_1,\cdots,A_n$ as
$$
A_j = \pi_j^*x_jx_j^\top-\frac{1}{n}\Sigma_*, \;\;\;\;\;\;j=1,\cdots,n.
$$
By definition, $\sum_{j=1}^n{A_j}=0$ and $\sum_{t=1}^T{A_{\sigma(t)}}=\hat\Sigma_T-\frac{T}{n}\Sigma_*$, where $\sigma$ is a random permutation from $[n]$ to $[n]$.
In addition, 
$$
\sup_{1\leq j\leq n}\|A_j\|_2 \leq \frac{1}{n}\|\Sigma_*\|_2 + \sup_{1\leq j\leq n}\|x_j\|_2^2 \leq 2\|X\|_{\infty}^2
$$
and
\begin{align*}
\frac{T}{n}\left\|\sum_{j=1}^n{A_j^2}\right\|_2 
&\leq \frac{2T}{n}\left(\left\|\sum_{j=1}^n{(\pi_j^*)^2\|x_i\|_2^2x_ix_i^\top}\right\|_2 + \frac{1}{n^2}\|\Sigma_*\|_2^2\right)\\
&\leq \frac{2T}{n}\left(\|X\|_{\infty}^2\left\|\sum_{j=1}^n{\pi_j^*x_ix_i^\top}\right\|_2 + \frac{1}{n^2}\|\Sigma_*\|_2^2\right)\\
&\leq \frac{2T}{n}\left(\|X\|_{\infty}^2\|\Sigma_*\|_2 + \frac{1}{n^2}\|\Sigma_*\|_2^2\right)\\
&\leq \frac{4T}{n}\|X\|_{\infty}^2\|\Sigma_*\|_2.
\end{align*}
Invoking Lemma \ref{lem:bernstein_worep}, we have that
$$
\Pr\left[\left\|\hat\Sigma_T-\frac{T}{n}\Sigma_*\right\|_2 > t\right] \leq p\exp\left\{-t^2\left[\frac{48T}{n}\|X\|_{\infty}^2\|\Sigma_*\|_2 +8\sqrt{2}\|X\|_{\infty}^2 t \right]^{-1}\right\}.
$$
Set $t=O(T/n)\cdot \lambda_{\min}(\Sigma_*)$.
We then have that, if $\|\Sigma_*^{-1}\kappa(\Sigma_*)\|X\|_{\infty}^2\log p=O(T/n)$ holds, then with probability at least 0.95 
\begin{equation}
\Omega(T/n)\cdot \Sigma_* \preceq \hat\Sigma_T \preceq O(T/n)\cdot \Sigma_*.
\label{eq:tn_chain2}
\end{equation}
Combining Eqs.~(\ref{eq:tn_chain1},\ref{eq:tn_chain2}) and noting that $\|\hat\Sigma_T^{-1}\|_2\leq O(\frac{n}{T})\|\Sigma_*^{-1}\|_2$, $\kappa(\Sigma_T)\leq O(1)\kappa(\Sigma_*)$, we complete the proof of Lemma \ref{lem:spectral_deterministic} under the without replacement setting.

\subsection{Proof of Lemma \ref{lem:support-bound}}\label{subsec:proof-support}

Let $ f(\vct\pi;\vct\lambda,\tilde{\vct\lambda},\mu)$ be the Lagrangian muliplier function of the without replacement formulation of Eq.~(\ref{eq_c_opt}):
$$
 f(\vct\pi;\vct\lambda,\tilde{\vct\lambda},\mu) 
= f(\vct\pi;X) - \sum_{i=1}^n{\lambda_i\pi_i} + \sum_{i=1}^n{\tilde\lambda_i\left(\pi_i-\frac{1}{k}\right)} + \mu\left(\sum_{i=1}^n{\pi_i} - 1\right).
$$
Here $\{\lambda_i\}_{i=1}^n\geq 0$, $\{\tilde\lambda_i\}_{i=1}^n\geq 0$ and $\mu\geq 0$ are Lagrangian multipliers for constraints $\pi_i\geq 0$,
$\pi_i\leq 1$ and $\sum_i{\pi_i}\leq k$, respectively.
By KKT condition, $\frac{\partial f}{\partial\pi_i} \big|_{\vct\pi^*} = \vct 0$ and hence
$$
-\frac{\partial f}{\partial\pi_i}\bigg|_{\vct\pi^*}
\;\;=\;\;
\vct x_i^\top\mat\Sigma_*^{-2}\vct x_i
\;\;=\;\;
\tilde\lambda_i-\lambda_i+\mu,\;\;\;\;\;\;i=1,\cdots,n,
$$
where $\mat\Sigma_* = \mat X^\top\diag(\vct\pi^*)\mat X$ is a $p\times p$ positive definite matrix.

Split the index set $[n]$ into three disjoint sets defined as $A=\{i\in[n]: \vct\pi_i^*=1\}$, $B=\{i\in[n]: 0<\vct\pi_i^*<1\}$ and 
$C=\{i\in[n]: \vct\pi_i^*=0\}$.
Note that $\|\vct\pi^*\|_0=|A|+|B|$ and $|A|\leq k$.
Therefore, to upper bound $\|\vct\pi^*\|_0$ it suffices to upper bound $|B|$.
By complementary slackness, for all $i\in B$ we have that $\tilde\lambda_i = \lambda_i = 0$;
that is,
\begin{equation}
\vct x_i^\top\mat\Sigma_*^{-2}\vct x_i 
= \langle\phi(\vct x_i), \psi(\mat\Sigma_*^{-2})\rangle
= \mu, \;\;\;\;\;\;
\forall i\in B,
\label{eq_linear_system}
\end{equation}
where $\phi:\mathbb R^p\to\mathbb R^{p(p+1)/2}$ is the mapping defined in Assumption \ref{asmp:general-position}
and $\psi(\cdot)$ takes the upper triangle of a symmetric matrix and vectorizes it into a $\frac{p(p-1)}{2}$-dimensional vector.
Assume by way of contradiction that $|B|>p(p+1)/2$ and let $\vct x_1,\cdots,\vct x_{p(p+1)/2+1}$ be arbitrary distinct $\frac{p(p+1)}{2}+1$ rows whose indices belong to $B$.
Eq.~(\ref{eq_linear_system}) can then be cast as a homogenous linear system with $\frac{p(p+1)}{2}+1$ variables and equations as follows:
$$
\left[\begin{array}{c}
\tilde\phi(\vct x_1)\\
\tilde\phi(\vct x_2)\\
\vdots\\
\tilde\phi(\vct x_{p(p+1)/2+1})\end{array}\right]
\left[\begin{array}{c}
\psi(\mat\Sigma_*^{-2})\\
-\mu\end{array}\right]
\;\; = \;\;
\mat 0.
$$
Under Assumption \ref{asmp:general-position}, $\tilde\Phi = [\tilde\phi(\vct x_1);\cdots;\tilde\phi(\vct x_{p(p+1)/2+1})]^\top$ is invertible and hence
both $\psi(\mat\Sigma_*^{-2})$ and $\mu$ must be zero.
This contradicts the fact that $\mat\Sigma_*^{-2}$ is positive definite.

\appendix

\section{Technical lemmas}

\begin{lemma}
Let $\mathbb B=\{x\in\mathbb R^p:\|x\|_2^2\leq B^2\}$ and $\vol(\mathbb B)=\int_{\mathbb B}1\ud x$ be the volume of $\mathbb B$. Then 
$\int_{\mathbb B}xx^\top\ud x = \frac{B^2}{p+2}\vol(\mathbb B)I_{p\times p}$.
\label{lem:integration}
\end{lemma}
\begin{proof}
Let $U$ be the uniform distribution in the $p$-dimensional ball of radius $B$.
By definition, $\int_{\mathbb B}xx^\top\ud x = \vol(\mathbb B)\mathbb E_{x\sim U}[xx^\top]$.
By symmetry, $\mathbb E_{x\sim U}[xx^\top] = c\cdot I_{p\times p}$ for some constant $c$ that depends on $p$ and $B$.
To determine the constant $c$, note that
$$
c = \mathbb E[x_1^2] = \frac{1}{p}\mathbb E\|x\|_2^2 = \frac{1}{p}\frac{\int_0^B{r^{p-1}\cdot r^2\ud r}}{\int_0^B{r^{p-1}\ud r}} = \frac{B^2}{p+2}.
$$
\end{proof}

\begin{lemma}[\cite{rudelson2007sampling}]
Let $x$ be a $p$-dimensional random vector such that $\|x\|_2\leq M$ almost surely and $\|\mathbb Exx^\top\|_2\leq 1$.
Let $x_1,\cdots,x_n$ be i.i.d.~copies of $x$.
Then for every $t\in(0,1)$
$$
\Pr\left[\left\|\frac{1}{n}\sum_{i=1}^n{x_ix_i^\top}-\mathbb Exx^\top\right\|_2 > t\right] \leq 2\exp\left\{-C\cdot \frac{nt^2}{M^2\log n}\right\},
$$
where $C>0$ is some universal constant.
\label{lem:matrix_rudelson}
\end{lemma}


\begin{lemma}[Corollary 5.2 of \citep{mackey2014matrix}, Matrix Bernstein]
Let $(Y_k)_{k\geq 1}$ be a sequence of random $d$-dimensional Hermitian matrices that satisfy
$$
\mathbb EY_k = 0 \;\;\;\;\;\text{and}\;\;\;\;\; \|Y_k\|_2\leq R \;\;\;a.s.
$$
Define $X=\sum_{k\geq 1}{Y_k}$. The for any $t>0$, 
$$
\Pr\left[\|X\|_2\geq t\right]\leq d\cdot\exp\left\{-\frac{t^2}{3\sigma^2+2Rt}\right\} \;\;\;\;\text{for}\;\;\sigma^2 = \left\|\sum_{k\geq 1}{\mathbb EY_k^2}\right\|_2.
$$
\label{lem:bernstein}
\end{lemma}

\begin{lemma}[Corollary 10.3 of \citep{mackey2014matrix}]
Let $A_1,\cdots,A_n$ be a sequence of deterministic $d$-dimensional Hermitian matrices that satisfy
$$
\sum_{k=1}^n{A_k}=0 \;\;\;\;\;\text{and}\;\;\;\;\; \sup_{1\leq k\leq n}\|A_k\|_2\leq R.
$$
Define random matrix $X=\sum_{j=1}^m{A_{\sigma(j)}}$ for $m\leq n$, where $\sigma$ is a random permutation from $[n]$ to $[n]$.
Then for all $t>0$, 
$$
\Pr\left[\|X\|_2\geq t\right] \leq d\exp\left\{-\frac{t^2}{12\sigma^2+4\sqrt{2}Rt}\right\} \;\;\;\;\text{for}\;\; \sigma^2 = \frac{m}{n}\left\|\sum_{k=1}^n{A_k^2}\right\|_2.
$$
\label{lem:bernstein_worep}
\end{lemma}

\section{Optimization methods}\label{appsec:optimization}

Two algorithms for optimizing Eq.~(\ref{eq_c_opt}) are described. The SDP formulation is of theoretical interest only
and the projected gradient descent algorithm is practical, which also enjoys theoretical convergence guarantees.

\paragraph{SDP formulation}
For $\pi\in\mathbb R^n$ define $A(\pi)=\sum_{i=1}^n{\pi_ix_ix_i^\top}$, which is a $p\times p$ positive semidefinite matrix.
By definition, $f(\pi;X)=\sum_{j=1}^p{e_j^\top A(\pi)^{-1}e_j}$, where $e_j$ is the $p$-dimensional vector with only $p$th coordinate being 1.
Subsequently, Eq.~(\ref{eq_c_opt}) is equivalent to the following SDP problem:
\begin{equation*}
\min_{\pi,t}\sum_{j=1}^p{t_j}\;\;\;\;
\text{subject to}\;\;\;\; 0\leq \pi_i\leq 1,\;\; \sum_{i=1}^n{\pi_i}\leq k, \;\;\diag(B_1,\cdots,B_p)\succeq 0,
\end{equation*}
where
$$
B_j = \left[\begin{array}{cc}
A(\pi)& e_j\\ e_j^\top& t_j\end{array}\right], \;\;\;\;\;\;j=1,\cdots,p.
$$

Global optimal solution of an SDP can be computed in polynomial time \citep{vandenberghe1996semidefinite}.
However, this formulation is not intended for practical computation because of the large number of variables in the SDP system.
First-order methods such as projected gradient descent is a more appropriate choice for practical computation.

\paragraph{Projected gradient descent}
For any convex set $S$ and point $x$ let $\mathcal P_S(x)=\argmin_{y\in S}\|x-y\|_2$ denote the $\ell_2$ projection of $x$ onto $S$.
The projected gradient descent algorithm is a general purpose method to solve convex constrained smooth convex optimization problems of the form
$$
\min_x f(x) \;\;\;\;\text{subject to}\;\;\;\; x\in S.
$$
The algorithm (with step size selected via backtracking line search) iterates until desired optimization accuracy is reached:

\begin{algorithm}[h]
\SetAlgorithmName{Figure}{}
\SetAlgoLined
\DontPrintSemicolon
\SetKwInOut{Input}{input}\SetKwInOut{Output}{output}
\Input{backtracking parameters $\alpha\in(0,1/2]$, $\beta\in(0,1)$.}
\Output{$\hat x$, approximate solution of the optimization problem.}
Initialization: $x_0$, $t=0$.\;
1. Compute gradient $g_t=\nabla f(x_t)$.\;
2. Find the smallest integer $s\geq 0$ such that $f(x')-f(x_t)\leq\alpha g_t^\top(x'-x_t)$,
where $x'=\mathcal P_S(x_t-\beta^s g_t)$.\;
3. Set $x_{t+1}=\mathcal P_S(x_t-\beta^sg_t)$, $t\gets t+1$ and repeat steps 1 and 2, until the desired accuracy is achieved.
Output $\hat x=x_t$.\;
\caption{The projected gradient descent algorithm.}
\label{alg:pgd}
\end{algorithm}

The gradient $\nabla_{\pi} f(\pi;X)$ in Eq.~(\ref{eq_c_opt}) is easy to compute:
$$
\frac{\partial f(\pi;X)}{\partial\pi_i} = -x_i^\top(X^\top\diag(\pi) X)^{-1}x_i, \;\;\;\;\;i=1,\cdots,n.
$$
Because $X^\top\diag(\pi)X$ is a shared term, computing $\nabla_{\pi} f(\pi;X)$ takes $O(np^2+p^3)$ operations.
The projection step onto the intersection of $\ell_1$ and $\ell_{\infty}$ balls is complicated and non-trivial, which we 
describe in details in Appendix \ref{suppsec:projection}.
In general, the projection step can be done in $O(n\log\|\pi\|_{\infty})$ time, where $\pi$ is the point to be projected.

The following proposition establishes convergence guarantee for the projected gradient descent algorithm.
Its proof is given in the appendix.
\begin{proposition}
Let $\pi^{(0)}$ be the ``flat'' initialization (i.e., $\pi^{(0)}_i=k/n$) and $\pi^{(t)}$ be the solution after $t$ projected gradient iterations. Then
$$
f(\pi^{(t)};X) - f(\pi^*;X) \leq \frac{\|X\|_2^4[\tr(\hat\Sigma^{-1})]^3\|\pi^{(0)}-\pi^*\|_2^2}{\beta t},
$$
where $\hat\Sigma=\frac{1}{n}X^\top X$.
\label{prop:convergence}
\end{proposition}
We also remark that the provided convergence speed is very conservative, especially in the cases when $k$ is large
where practical evidence suggests that the algorithm converges in very few iterations (cf. Sec. \ref{subsec:synthetic}).

\section{Fedorov's exchange algorithm}\label{appsec:fedorov}

The algorithm starts with a initial subset $S\subseteq[n]$, $|S|\leq k$, usually initialized with random indices.
A ``best'' pair of exchanging indices are computed as
$$
i^*, j^* = \argmin_{i\in S, j\notin S} \ell(S\backslash\{i\}\cup\{j\}),
$$
where $\ell(S)$ is the objective function to be minimized.
In our case it would be the A-optimality objective $\ell(S) = F(S;X) = \tr((X_S^\top X_S)^{-1})$.
The algorithm then ``exchanges'' $i^*$ and $j^*$ by setting $S'\gets S\backslash\{i^*\}\cup\{j^*\}$
and continues such exchanges until no exchange can lower the objective value.
Under without replacement settings, special care needs to be taken to ensure that $S$ consists of distinct indices.

Computing the objective function $F(S;X)=\tr((X_S^\top X_S)^{-1})$ requires inverting a $p\times p$ matrix, which could be computationally slow.
A more computationally efficient approach is to perform rank-1 update of the inverse of $X_S^\top X_S$ after each iteration, via the Sherman-Morrison formula:
$$
(A+uv^\top)^{-1} = A^{-1} - \frac{A^{-1}uv^\top A^{-1}}{1+u^\top A^{-1}v}.
$$
Each exchange would then take $O(nkp^2)$ operations.
The total number of exchanges, however, is unbounded and could be as large as $\binom{n}{k}$ in theory.
In practice we do observe that a large number of exchanges are required in order to find a local optimal solution.

\section{Projection onto the intersection of $\ell_1$ and $\ell_{\infty}$ norm balls}\label{suppsec:projection}

Similar to \citep{YuSuLi2012}, we only need to consider the case that $\pi$ lies in the first quadrant, i.e., $\pi_i\ge 0,i=1,2,...,n$. Then the projection $x$ also lies in the first quadrant. Furthermore, we assume the point to be project lies in the area which is out of both norm balls and projection purely onto either $\ell_1$ or $\ell_\infty$ ball is not the intersection of both. Otherwise, it is trivial to conduct the projection.
The projection problem is formulated as follows:
\begin{eqnarray*}
\min_{x} && \frac{1}{2} \|\pi-x\|^2_2\\\nonumber
\text{s.t.} && \| x\|_1\leq c_1,~~ \|x\|_\infty \le c_2 .\\\nonumber
\end{eqnarray*}
By introducing an auxiliary variable $d$, the problem above has the following equivalent form:
\begin{eqnarray}\label{l1infproj}
  \min_{x,d} && \frac{1}{2} \|\pi-x\|^2_2\\\nonumber
  \text{s.t.} && \| x\|_1\leq c_1,\\\nonumber 
    && x_{i}\leq d, ~i=1,2,...n,\\\nonumber
    && d\leq c_2.
\end{eqnarray}

\begin{algorithm}[h]
\SetAlgorithmName{Algorithm}{}
\SetAlgoLined
\DontPrintSemicolon
\SetKwInOut{Input}{input}\SetKwInOut{Output}{output}
\Input{$\pi, c_1, c_2$ and precision parameter $\delta$.}
\Output{$x$, the projection of $\pi$ onto $\{x:\|x\|_1\leq c_1\}\cap\{x:\|x\|_{\infty}\leq c_2\}$.}
Initialization: $\lambda_1=\|\pi\|_{\infty}$.\;
\While{$|h(\lambda_1)|>\delta$}{
  \eIf{$h(\lambda_1)>\delta$}{
    $l = \lambda_1$;\;
  }{
  $r = \lambda_1$;\;
}
$\lambda_1 = (l+r)/2$;\;
{\bfseries for each} $i$ {\bfseries do} $x_{i} = 、\min(\max(\pi_{i}-\lambda_1,0), c_2)$;\;
}
\caption{Projecting $\pi$ onto $\{x:\|x\|_1\leq c_1\}\cap\{x:\|x\|_{\infty}\leq c_2\}$.}
\label{alg:l1l1inf_projection}
\end{algorithm}

%
%



The Lagrangian of problem \eqref{l1infproj}  is
\begin{equation}\label{lagrangian_l1_l1inf}
\begin{split}
  \mathcal{L}(x,d,\lambda_1,\lambda_2)=&\frac{1}{2} ||\pi-x||^2_2 +
  \sum_{i}\omega_{i}(x_{i}-d)+\lambda_1(\|x\|_1-c_1)+\lambda_2(d-c_2)
\end{split}
\end{equation}
Let $x^*, d^*$ and $\lambda_1^*, \lambda_2^*$ respectively be the primal and dual solution of \eqref{l1infproj}, then its KKT condition is:
\begin{eqnarray}
 0 \in \partial \mathcal{L}_{x}(x^*,d^*,\lambda_1^*,\lambda_2^*), \label{subgradient_l1l1inf_x}\\
 0\in \partial \mathcal{L}_{d}(x^*,d^*,\lambda_1^*,\lambda_2^*), \label{subgradient_l1l1inf_d}\\
w_{i}^*(x_{i}^*-d^*)=0,~~ i = 1,2,...,n\\
\lambda_2^*(d^*-c_2)=0,\\
\lambda_1^*(\|x^*\|_1-c_1)=0,\\
x_{i}^*\leq d^*, ~~ i = 1,2,...,n\\
d^*\leq c_2,\\
\|x^*\|_1 \le c_1,\\
x_{i}^*, d^*,\lambda_1^*,\lambda_2^*,w_{i}^*\geq0.
\end{eqnarray}

The following lemmas illustrate the relation between the primal and dual solution of problem \eqref{l1infproj}:
\begin{lemma}\label{lemma_d}
Either one of the following holds:
1)  $d^*>0$ and $\sum_i(\max(\pi_{i}-\lambda_1^*,0)-x^*_{i})=\lambda^*_2$; 
2) $d^*=0$ and $\sum_i\max(\pi_{i}-\lambda_1^*,0)\leq\lambda_2^*$.
\end{lemma}
\begin{lemma}\label{lemma_l1l1in_solution}
For $i=1,2,...,n$, the optimal $x_i^*$ satisfies
\begin{equation}\label{eq_l1l1inf_lambda1}
  x_{i}^*=\min(\max(\pi_{i}-\lambda_1^*,0),d^*).
\end{equation}  
\end{lemma}
\begin{proof}
Lemma \ref{lemma_d} and Lemma \ref{lemma_l1l1in_solution} can be both viewed as the special cases of Lemma 6 and 7 in \cite{YuSuLi2012}.
\end{proof}

Now the KKT conditions can be reduced to finding $\lambda_1^*,\lambda_2^*, d^*$ that satisfy the following equations: 
\begin{eqnarray}
\label{l1l1inf_l1_constraint}
\sum_{i}\min(\max(\pi_{i}-\lambda_1^*,0),d^*)-c_1&= & 0,\\
\label{l1l1inf_l1inf_constraint}
d^*-c_2 &= & 0,\\
\label{l1l1inf_dibiggerthan0_consttaint}
\sum_{i}\max(\max(\pi_{i}-\lambda_1^*,0)-d^*, 0) &= &\lambda_2^* ~~\text{iff}~~ d^*>0,\\
\label{l1l1inf_diequi0_constraint}
\sum_{i}\max(\pi_{i}-\lambda_1^*,0) & \leq& \lambda_2^* ~~\text{iff}~~ d^*=0,\\
\label{l1l1inf_dimunu_constraint}
d^*\geq0, \lambda_2^*\geq0, \lambda_1^*\geq0.
\end{eqnarray}


Further simplification induces the following result:
\begin{lemma}
  Suppose $d^*$ and $\lambda_1^*$ are the primal and dual solution respectively, then
\begin{eqnarray}
  \label{eq_l1l1inf_di}
    d^*&=& c_2,\\
    \label{eq_l1linf_lambda1}
      x_{i}^*&=&\min(\max(\pi_{i}-\lambda_1^*,0),c_2).
\end{eqnarray}
\end{lemma}

\begin{proof}
Direct result of \eqref{eq_l1l1inf_lambda1} and \eqref{l1l1inf_l1inf_constraint}.
\end{proof}
\eqref{eq_l1linf_lambda1} shows  that the solution $x^*$ is determined once the optimal $\lambda_1^*$ is found. Given a $\lambda_1\in[0,\max_{i}(\pi_{i})]$, 
we only need to check whether \eqref{l1l1inf_l1_constraint} holds by looking into the function
\begin{equation}\label{h_lambda2}
  h(\lambda_1)=\sum_{i}\min(\max(\pi_{i}-\lambda_1,0),c_2)-c_1,
\end{equation}
and $\lambda_1^*$ is simply the zero point of $h(\cdot)$. The following theorem shows that $h(\cdot)$ is a strictly monotonically decreasing function, so a binary search is sufficient to find $\lambda_1^*$, and $x_i^*$ can be determined accordingly.

\begin{theorem}\label{thm_monotone}
1) $h(\lambda_1)$ is a continuous piecewise linear function in $[0, \max_i\{\pi_{i}\}]$;
2) $h(\lambda_1)$ is strictly monotonically decreasing and it has a unique root in $[0,\max_i\{\pi_{i}\}]$.
\end{theorem}
\begin{proof}
1) is obviously true. 
  It is easy to check that $h'(\lambda_1)<0$ in each piece, $h(\max_i\{\pi_i\}) < 0$ and $h(0)>0$, so 2) also holds.
\end{proof}

\textbf{Complexity Analysis:} Algorithm \ref{alg:l1l1inf_projection} is proposed based on the observation above to solve the projection problem, which is essentially a bisection to search $\lambda_1$. Given a precision $\delta$, the iteration complexity of bisection is $O(\log(\max_{i}(\pi_{i}) / \delta))$. Besides, 
the time complexity of evaluating $h(\cdot)$ is $O(n)$, so the total complexity of Algorithm \ref{alg:l1l1inf_projection} is $O(n \log(\max_{i}(\pi_{i}) / \delta))$.

\section{Convergence analysis of PGD}\label{suppsec:convergence}

We provide a convergence analysis of the projected gradient descent algorithm used in optimizing Eq.~(\ref{eq_c_opt}).
The analysis shows that the PGD algorithm approximately computes the global optimum of Eq.~(\ref{eq_c_opt}) in polynomial time.
In simulation studies, much fewer iterations are required for convergence than predicted by the theoretical results.

Because we're using exact projected gradient descent algorithms, the objective function shall decay monotonically
and hence convergence of such algorithms can be established by showing Lipschitz continuity of $\nabla f$ on a specific \emph{level set};
that is, for some Lipschitz constant $L>0$ the following holds for all $\pi,\pi'$ such that $f(\pi),f(\pi')\leq f(\pi^{(0)})$:
\begin{equation}
\|\nabla f(\vct\pi) - \nabla f(\vct\pi')\|_2 \leq L\|\vct\pi-\vct\pi'\|_2,
\label{eq_smooth}
\end{equation}
where $\vct\pi^{(0)} = (1/n,1/n,\cdots,1/n)$ is the initialization point.
Once Eq.~(\ref{eq_smooth}) holds, linear convergence (i.e., $ f(\vct\pi^{(t)})- f(\vct\pi^*) = O(1/t)$) can be established
via standard projected gradient analysis.

The main idea of establishing Eq.~(\ref{eq_smooth}) is to upper bound the spectral norm of the Hessian matrix $\mat H=\nabla^2 f(\vct\pi)$
\emph{uniformly} over all points $\vct\pi$ that satisfies $ f(\vct\pi)\leq f(\vct\pi^{(0)})$.
As a first step, we derive analytic forms of $\mat H$ in the following proposition:
\begin{proposition}
Let $\widetilde{\mat\Sigma} = \mat X^\top\diag(\vct\pi)\mat X$.
We then have that
$$
\mat H = 2(\mat X^\top\widetilde{\mat\Sigma}^{-2}\mat X)\circ (\mat X^\top\widetilde{\mat\Sigma}^{-1}\mat X),
$$
where $\circ$ denotes the element-wise Hadamard product between two matrices of same dimensions.
\end{proposition}
\begin{proof}
We first derive the gradient of $ f$.
Fix arbitrary $i\in[n]$.
The partial derivative can be computed as
$$
\frac{\partial f}{\partial\pi_i} = \frac{\partial \tr(\widetilde{\mat\Sigma}^{-1})}{\partial\pi_i}
= \left\langle \frac{\partial\tr(\widetilde{\mat\Sigma}^{-1})}{\partial\widetilde{\mat\Sigma}}, \frac{\partial\widetilde{\mat\Sigma}}{\partial\pi_i}\right\rangle
= \left\langle -\widetilde{\mat\Sigma}^{-2}, \vct x_i\vct x_i^\top\right\rangle
= -\vct x_i^\top\widetilde{\mat\Sigma}^{-2}\vct x_i.
$$
Here $\langle\mat A,\mat B\rangle = \tr(\mat B^\top\mat A)$ is the element-wise multiplication inner product between two matrices.
The second-order partial derivatives can then be computed as
\begin{multline*}
\frac{\partial^2 f}{\partial\pi_i\partial\pi_j}
= -\frac{\partial(\vct x_i^\top\widetilde{\mat\Sigma}^{-2}\vct x_i)}{\partial \pi_j}
= -\vct x_i^\top\frac{\partial\widetilde{\mat\Sigma}^{-2}}{\partial\pi_j}\vct x_i
= \vct x_i^\top\widetilde{\mat\Sigma}^{-2}\frac{\partial\widetilde{\mat\Sigma}^2}{\partial\pi_j}\widetilde{\mat\Sigma}^{-2}\vct x_i\\
= \vct x_i^\top\widetilde{\mat\Sigma}^{-2}(\vct x_j\vct x_j^\top\widetilde{\mat\Sigma} + \widetilde{\mat\Sigma}\vct x_j\vct x_j^\top)\widetilde{\mat\Sigma}^{-2}\vct x_i
= 2(\vct x_i^\top\widetilde{\mat\Sigma}^{-2}\vct x_j)\cdot(\vct x_i^\top\widetilde{\mat\Sigma}^{-1}\vct x_j).
\end{multline*}
Subsequently,
$$
\mat H = \nabla^2 f = \left[\frac{\partial^2 f}{\partial\pi_i\partial\pi_j}\right]_{i,j=1}^n = 2(\mat X^\top\widetilde{\mat\Sigma}^{-2}\mat X)\circ(\mat X^\top\widetilde{\mat\Sigma}^{-1}\mat X).
$$
\end{proof}

\begin{corollary}
Suppose $\vct\pi$ satisfies $ f(\vct\pi) \leq  f(\vct\pi^{(0)}) = \tr(\mat\Sigma_0^{-1})$, where $\mat\Sigma_0=\frac{1}{n}\mat X^\top\mat X$.
We then have 
$$
\|\mat H\|_2 = \|\nabla^2 f(\vct\pi)\|_2 \leq 2\|\mat X\|_2^4\cdot [\tr(\mat\Sigma_0^{-1})]^3.
$$
\label{cor_hessian}
\end{corollary}
\begin{proof}
Let $\rho(\mat A)=\max_i |\sigma_i(\mat A)|$ denote the \emph{spectral range} of matrix $\mat A$.
Clearly, $\|\mat A\|_2\leq\rho(\mat A)$
and $\|\mat A\|_2=\rho(\mat A)$ for positive semi-definite matrices.
In \citep{hom1991topics} it is established that $\rho(\mat A\circ\mat B)\leq\rho(\mat A)\rho(\mat B)$.
Subsequently,
$$
\|\mat H\|_2 \leq 2\|\mat X\|_2^4\|\widetilde{\mat\Sigma}^{-1}\|_2^3
\leq 2\|\mat X\|_2^4\cdot [\tr(\widetilde{\mat\Sigma}^{-1})]^3
\leq 2\|\mat X\|_2^4\cdot [\tr(\mat\Sigma_0^{-1})]^3,
$$
where the last inequality is due to the condition that $\tr(\widetilde{\mat\Sigma}^{-1})= f(\vct\pi)\leq f(\vct\pi^{(0)}) = \tr(\mat\Sigma_0^{-1})$.
\end{proof}

From Corollary \ref{cor_hessian}, we can prove the convergence of the optimization procedures outlined in Appendix.~\ref{appsec:optimization}
following standard analysis of projected gradient descent on objective functions with Lipschitz continuous gradient:
\begin{theorem}
Suppose $\vct\pi^{(t)}$ is the solution at the $t$th iteration of the projected gradient descent algorithm
and $\vct\pi^*$ is the optimal solution. We then have
$$
 f(\vct\pi^{(t)}) -  f(\vct\pi^*) \leq \frac{\|\mat X\|_2^4\cdot [\tr(\mat\Sigma_0^{-1})]^3\cdot \|\vct\pi^{(0)}-\vct\pi^*\|_2^2}{\beta t},
$$
where $\beta\in(0,1)$ is the backtracking parameter in backtracking line search.
\label{thm:convergence}
\end{theorem}


\acks{We thank Reeja Jayan and Chiqun Zhang for sharing with us data on the material synthesis experiments,
and Siheng Chen for a pre-processed version of the Minnesota wind speed data set.
This work is supported by NSF CCF-1563918, NSF CAREER IIS-1252412 and AFRL FA87501720212.}

\bibliography{sampleols}

\begin{thebibliography}{38}
\providecommand{\natexlab}[1]{#1}
\providecommand{\url}[1]{\texttt{#1}}
\expandafter\ifx\csname urlstyle\endcsname\relax
  \providecommand{\doi}[1]{doi: #1}\else
  \providecommand{\doi}{doi: \begingroup \urlstyle{rm}\Url}\fi

\bibitem[Ageev and Sviridenko(2004)]{ageev2004pipage}
Alexander~A Ageev and Maxim~I Sviridenko.
\newblock Pipage rounding: A new method of constructing algorithms with proven
  performance guarantee.
\newblock \emph{Journal of Combinatorial Optimization}, 8\penalty0
  (3):\penalty0 307--328, 2004.

\bibitem[Anderson et~al.(2014)Anderson, Gu, and
  Melgaard]{anderson2014efficient}
David~G Anderson, Ming Gu, and Christopher Melgaard.
\newblock An efficient algorithm for unweighted spectral graph sparsification.
\newblock \emph{arXiv preprint arXiv:1410.4273}, 2014.

\bibitem[Avron and Boutsidis(2013)]{faster-css}
Haim Avron and Christos Boutsidis.
\newblock Faster subset selection for matrices and applications.
\newblock \emph{SIAM Journal on Matrix Analysis and Applications}, 34\penalty0
  (4):\penalty0 1464--1499, 2013.

\bibitem[Batson et~al.(2012)Batson, Spielman, and Srivastava]{batson2012twice}
Joshua Batson, Daniel~A Spielman, and Nikhil Srivastava.
\newblock Twice-ramanujan sparsifiers.
\newblock \emph{SIAM Journal on Computing}, 41\penalty0 (6):\penalty0
  1704--1721, 2012.

\bibitem[Bickel et~al.(2009)Bickel, Ritov, and
  Tsybakov]{bickel2009simultaneous}
Peter~J Bickel, Ya'acov Ritov, and Alexandre~B Tsybakov.
\newblock Simultaneous analysis of lasso and dantzig selector.
\newblock \emph{The Annals of Statistics}, 37\penalty0 (4):\penalty0
  1705--1732, 2009.

\bibitem[Chaudhuri et~al.(2015)Chaudhuri, Kakade, Netrapalli, and
  Sanghavi]{active-mle}
Kamalika Chaudhuri, Sham Kakade, Praneeth Netrapalli, and Sujay Sanghavi.
\newblock Convergence rates of active learning for maximum likelihood
  estimation.
\newblock In \emph{Proceedings of Advances in Neural Information Processing
  Systems (NIPS)}, 2015.

\bibitem[Chen et~al.(2015)Chen, Varma, Singh, and
  Kova\u{c}evi\'{c}]{minnesota-wind-speed}
Siheng Chen, Rohan Varma, Aarti Singh, and Jelena Kova\u{c}evi\'{c}.
\newblock Signal representations on graphs: Tools and applications.
\newblock \emph{arXiv:1512.05406}, 2015.

\bibitem[Chernoff(1953)]{chernoff1953locally}
Herman Chernoff.
\newblock Locally optimal designs for estimating parameters.
\newblock \emph{The Annals of Statistics}, pages 586--602, 1953.

\bibitem[Davenport et~al.(2015)Davenport, Massimino, Needell, and
  Woolf]{constrained-adaptive-sensing}
Mark~A Davenport, Andrew~K Massimino, Deanna Needell, and Tina Woolf.
\newblock Constrained adaptive sensing.
\newblock \emph{IEEE Transactions on Signal Processing}, 64\penalty0
  (20):\penalty0 5437--5449, 2015.

\bibitem[Dhillon et~al.(2013)Dhillon, Lu, Foster, and Ungar]{subsample-ols}
Paramveer Dhillon, Yichao Lu, Dean Foster, and Lyle Ungar.
\newblock New sampling algorithms for fast least squares regression.
\newblock In \emph{Proceedings of Advances in Neural Information Processing
  Systems (NIPS)}, 2013.

\bibitem[Dobriban and Fan(2016)]{fan2016regularity}
Edgar Dobriban and Jianqing Fan.
\newblock Regularity properties for sparse regression.
\newblock \emph{Communications in Mathematics and Statistics}, 4\penalty0
  (1):\penalty0 1--19, 2016.

\bibitem[Drineas et~al.(2008)Drineas, Mahoney, and Muthukrishnan]{cur_relative}
Petros Drineas, Michael~W Mahoney, and S~Muthukrishnan.
\newblock Relative-error {CUR} matrix decompositions.
\newblock \emph{SIAM Journal on Matrix Analysis and Applications}, 30\penalty0
  (2):\penalty0 844--881, 2008.

\bibitem[Drineas et~al.(2011)Drineas, Mahoney, Muthukrishnan, and
  Sarl{\'o}s]{faster-ols}
Petros Drineas, Michael~W Mahoney, S~Muthukrishnan, and Tam{\'a}s Sarl{\'o}s.
\newblock Faster least squares approximation.
\newblock \emph{Numerische Mathematik}, 117\penalty0 (2):\penalty0 219--249,
  2011.

\bibitem[Ein-Dor and Feldmesser(1987)]{cpu-relative-performance}
Phillip Ein-Dor and Jacob Feldmesser.
\newblock Attributes of the performance of central processing units: a relative
  performance prediction model.
\newblock \emph{Communications of the ACM}, 30\penalty0 (4):\penalty0 308--317,
  1987.

\bibitem[Hazan and Karnin(2015)]{hard-margin-active}
Elad Hazan and Zohar Karnin.
\newblock Hard margin active linear regression.
\newblock In \emph{Proceedings of International Conference on Machine Learning
  (ICML)}, 2015.

\bibitem[Hom and Johnson(1991)]{hom1991topics}
RA~Hom and CR~Johnson.
\newblock \emph{Topics in Matrix Analysis}.
\newblock Cambridge UP, New York, 1991.

\bibitem[Horel et~al.(2014)Horel, Ioannidis, and
  Muthukrishnan]{horel2014budget}
Thibaut Horel, Stratis Ioannidis, and S~Muthukrishnan.
\newblock Budget feasible mechanisms for experimental design.
\newblock In \emph{Latin American Symposium on Theoretical Informatics
  (LATIN)}, pages 719--730. Springer, 2014.

\bibitem[Khuri et~al.(2006)Khuri, Mukherjee, Sinha, and Ghosh]{khuri2006design}
Andre Khuri, Bhramar Mukherjee, Bikas Sinha, and Malay Ghosh.
\newblock Design issues for generalized linear models: a review.
\newblock \emph{Statistical Science}, 21\penalty0 (3):\penalty0 376--399, 2006.

\bibitem[Ma et~al.(2015)Ma, Mahoney, and Yu]{levscore-regression}
Ping Ma, Michael~W Mahoney, and Bin Yu.
\newblock A statistical perspective on algorithmic leveraging.
\newblock \emph{Journal of Machine Learning Research}, 16\penalty0
  (1):\penalty0 861--911, 2015.

\bibitem[Mackey et~al.(2014)Mackey, Jordan, Chen, Farrell, and
  Tropp]{mackey2014matrix}
Lester Mackey, Michael~I Jordan, Richard~Y Chen, Brendan Farrell, and Joel~A
  Tropp.
\newblock Matrix concentration inequalities via the method of exchangeable
  pairs.
\newblock \emph{The Annals of Probability}, 42\penalty0 (3):\penalty0 906--945,
  2014.

\bibitem[Marcus et~al.(2015{\natexlab{a}})Marcus, Spielman, and
  Srivastava]{ks1}
Adam Marcus, Daniel Spielman, and Nikhil Srivastava.
\newblock Interlacing families {I}: bipartite ramanujan graphs of all degrees.
\newblock \emph{Annals of Mathematics}, 182:\penalty0 307--325,
  2015{\natexlab{a}}.

\bibitem[Marcus et~al.(2015{\natexlab{b}})Marcus, Spielman, and
  Srivastava]{ks2}
Adam Marcus, Daniel Spielman, and Nikhil Srivastava.
\newblock Interlacing families {II}: Mixed characteristic polynomials and the
  kadison-singer problem.
\newblock \emph{Annals of Mathematics}, 182:\penalty0 327--350,
  2015{\natexlab{b}}.

\bibitem[Miller and Nguyen(1994)]{fedorov-exchange}
Alan Miller and Nam-Ky Nguyen.
\newblock A fedorov exchange algorithm for d-optimal design.
\newblock \emph{Journal of the Royal Statistical Society, Series C (Applied
  Stiatistics)}, 43\penalty0 (4):\penalty0 669--677, 1994.

\bibitem[Nakamura et~al.(2017)Nakamura, Seepaul, Kadane, and
  Reeja-Jayan]{nakamura2017design}
Nathan Nakamura, Jason Seepaul, Joseph~B Kadane, and B~Reeja-Jayan.
\newblock Design for low-temperature microwave-assisted crystallization of
  ceramic thin films.
\newblock \emph{Applied Stochastic Models in Business and Industry}, 2017.

\bibitem[Pan and Yang(2010)]{pan2010survey}
Sinno~Jialin Pan and Qiang Yang.
\newblock A survey on transfer learning.
\newblock \emph{IEEE Transactions on Knowledge and Data Engineering},
  22\penalty0 (10):\penalty0 1345--1359, 2010.

\bibitem[Pataki(1998)]{pataki1998rank}
Gabor Pataki.
\newblock On the rank of extreme matrices in semidefinite programs and the
  multiplicity of optimal eigenvalues.
\newblock \emph{Mathematics of Operations Research}, 23\penalty0 (2):\penalty0
  339--358, 1998.

\bibitem[Pukelsheim(1993)]{optimal-design-book}
Friedrich Pukelsheim.
\newblock \emph{Optimal design of experiments}, volume~50.
\newblock SIAM, 1993.

\bibitem[Raskutti and Mahoney(2015)]{raskutti2015statistical}
Garvesh Raskutti and Michael Mahoney.
\newblock Statistical and algorithmic perspectives on randomized sketching for
  ordinary least-squares.
\newblock In \emph{Proceedings of International Conference on Machine Learning
  (ICML)}, 2015.

\bibitem[Reeja-Jayan et~al.(2012)Reeja-Jayan, Harrison, Yang, Wang, Yilmaz, and
  Manthiram]{reeja2012microwave}
B~Reeja-Jayan, Katharine~L Harrison, K~Yang, Chih-Liang Wang, AE~Yilmaz, and
  Arumugam Manthiram.
\newblock Microwave-assisted low-temperature growth of thin films in solution.
\newblock \emph{Scientific reports}, 2\penalty0 (1003):\penalty0 1--8, 2012.

\bibitem[Rudelson and Vershynin(2007)]{rudelson2007sampling}
Mark Rudelson and Roman Vershynin.
\newblock Sampling from large matrices: An approach through geometric
  functional analysis.
\newblock \emph{Journal of the ACM}, 54\penalty0 (4):\penalty0 21, 2007.

\bibitem[Sabato and Munos(2014)]{stratification}
Sivan Sabato and Remi Munos.
\newblock Active regression by stratification.
\newblock In \emph{Proceedings of Advances in Neural Information Processing
  Systems (NIPS)}, 2014.

\bibitem[Spielman and Srivastava(2011)]{graph-sparsification}
Daniel~A Spielman and Nikhil Srivastava.
\newblock Graph sparsification by effective resistances.
\newblock \emph{SIAM Journal on Computing}, 40\penalty0 (6):\penalty0
  1913--1926, 2011.

\bibitem[Tibshirani(2013)]{lasso-unique}
Ryan Tibshirani.
\newblock The lasso problem and uniqueness.
\newblock \emph{Electronic Journal of Statistics}, 7:\penalty0 1456--1490,
  2013.

\bibitem[Van~der Vaart(2000)]{van2000asymptotic}
Aad~W Van~der Vaart.
\newblock \emph{Asymptotic statistics}, volume~3.
\newblock Cambridge university press, 2000.

\bibitem[Vandenberghe and Boyd(1996)]{vandenberghe1996semidefinite}
Lieven Vandenberghe and Stephen Boyd.
\newblock Semidefinite programming.
\newblock \emph{SIAM Review}, 38\penalty0 (1):\penalty0 49--95, 1996.

\bibitem[Woodruff(2014)]{sketching-book}
David Woodruff.
\newblock Sketching as a tool for numerical linear algebra.
\newblock \emph{Foundations and Trends in Theoretical Computer Science},
  10\penalty0 (1--2):\penalty0 1--157, 2014.

\bibitem[Yu et~al.(2012)Yu, Su, and Li]{YuSuLi2012}
Adams~Wei Yu, Hao Su, and Fei{-}Fei Li.
\newblock Efficient euclidean projections onto the intersection of norm balls.
\newblock In \emph{Proceedings of International Conference on Machine Learning
  (ICML)}, 2012.

\bibitem[Zhu et~al.(2015)Zhu, Ma, Mahoney, and Yu]{optimal-subsampling}
Rong Zhu, Ping Ma, Michael Mahoney, and Bin Yu.
\newblock Optimal subsampling approaches for large sample linear regression.
\newblock \emph{arXiv:1509.05111}, 2015.

\end{thebibliography}


\end{document}